\documentclass{article}

\PassOptionsToPackage{numbers, compress}{natbib}

\usepackage[preprint]{neurips_2026}

\usepackage[utf8]{inputenc} 
\usepackage[T1]{fontenc}    
\usepackage{hyperref}       
\usepackage{url}            
\usepackage{booktabs}       
\usepackage{amsfonts}       
\usepackage{nicefrac}       
\usepackage{microtype}      
\usepackage{amsmath,amssymb,amsthm}
\usepackage{xcolor}
\usepackage{algorithm}
\usepackage{algpseudocode}
\usepackage{graphicx}
\usepackage{graphics}
\usepackage{float}
\usepackage[noabbrev,capitalise,nameinlink]{cleveref}

\newtheorem{assumption}{Assumption}

\newtheorem{lemma}{Lemma}
\newtheorem{example}{Example}

\newtheorem{proposition}{Proposition}
\newtheorem{theorem}{Theorem}

\newtheorem{definition}{Definition}

\DeclareMathOperator*{\argmin}{arg\,min}
\DeclareMathOperator*{\argmax}{arg\,max}

\allowdisplaybreaks

\title{Bellman Residual Minimization for Control: Geometry, Stationarity, and Convergence}

\author{
  Donghwan Lee and Hyukjun Yang\\
  School of Electrical Engineering\\
  Korea Advanced Institute of Science and Technology (KAIST)\\
  Daejeon, South Korea 34141 \\
  \texttt{donghwan@kaist.ac.kr}
}

\begin{document}
\setlength{\abovedisplayskip}{3pt}
\setlength{\belowdisplayskip}{3pt}

\maketitle

\begin{abstract}
Markov decision problems are most commonly solved via dynamic programming. Another approach is Bellman residual minimization, which directly minimizes the squared Bellman residual objective function. However, compared to dynamic programming, this approach has received relatively less attention, mainly because it is often less efficient in practice and can be more difficult to extend to model-free settings such as reinforcement learning. Nonetheless, Bellman residual minimization has several advantages that make it worth investigating, such as more stable convergence with function approximation for value functions. While Bellman residual methods for policy evaluation have been widely studied, methods for policy optimization (control tasks) have been scarcely explored. In this paper, we establish foundational results for the control Bellman residual minimization for policy optimization.
\end{abstract}

\section{Introduction}

A standard approach to solving Markov decision problems~\citep{puterman2014markov} is dynamic programming~\citep{bertsekas1996neuro,bertsekas2012dynamic1,bertsekas2012dynamic2}. Another approach is Bellman residual minimization~\citep{baird1995residual,geist2017bellman,scherrer2010should}, which directly minimizes the squared Bellman residual objective. However, compared to dynamic programming, this approach has received relatively less attention, mainly because it is often less efficient in practice and can be more difficult to extend to model-free settings such as reinforcement learning~\citep{sutton1998reinforcement}. Nevertheless, Bellman residual minimization has several advantages that make it worth investigating~\citep{geist2017bellman,scherrer2010should}. For example, when function approximation is used, Bellman-residual-based methods can still reliably find a meaningful solution even if properties such as the contraction of the Bellman operator~\citep{bertsekas1996neuro,bertsekas2012dynamic1,bertsekas2012dynamic2} are lost. In contrast, dynamic-programming-based methods may lose the contraction property under function approximation, in which case convergence is hard to guarantee without additional conditions~\citep{lim2024regularized}. To date, Bellman residual methods for policy evaluation have been studied in several works~\citep{geist2017bellman,scherrer2010should}, whereas methods for policy design/optimization (control tasks) have been scarcely explored. In this paper, we establish foundational results for control Bellman residual (CBR) minimization for policy optimization.

In particular, we first consider the following control Bellman residual (CBR) objective:
\[f(\theta ) = \frac{1}{2}\left\| {T Q_\theta - Q_\theta } \right\|_2^2, \]
where $T$ denotes the standard control Bellman operator~\citep{bertsekas1996neuro}, $Q_\theta$ is a linear function approximator of the Q-function, and $\theta$ is its parameter vector. In policy evaluation, the Bellman residual objective is a strongly convex quadratic function of the parameter vector, which makes the analysis relatively straightforward. In contrast, for the CBR objective, the presence of the max operator in the Bellman operator renders the objective nonconvex and nonsmooth. As a result, the analysis becomes more challenging, and the optimization problem is also more difficult to solve.
In this paper, we study various mathematical properties of the CBR objective. For instance, we show that although the objective is nonconvex and nonsmooth, it is locally Lipschitz and admits a piecewise quadratic structure. Building on these properties, we derive an explicit characterization of the Clarke subdifferential~\citep{clarke1990optimization,bagirov2014introduction}, and use it to investigate the applicability and convergence of generalized gradient descent methods~\citep{burke2005robust}.
Next, we consider the soft control Bellman residual (SCBR) defined using the softmax and soft Bellman operator~\citep{haarnoja2017reinforcement,fox2016taming,haarnoja2018soft} as follows:
\[f(\theta ) = \frac{1}{2}\left\| {F_\lambda Q_\theta - Q_\theta } \right\|_2^2, \]
where $F_\lambda$ denotes the soft control Bellman operator~\citep{haarnoja2017reinforcement,fox2016taming,haarnoja2018soft} and $\lambda>0$ is the temperature parameter. Compared to the CBR case, the SCBR objective is differentiable, and thus standard gradient descent~\citep{nesterov2018lectures} can be applied. As in the CBR case, we will establish several properties of SCBR, and based on these results, we prove the convergence of gradient descent when applied to SCBR.

Based on these findings, this paper investigates the use of gradient-based algorithms to minimize the CBR/SCBR objectives. We show that, with function approximation, these methods can find approximate solutions more effectively than traditional projected value iteration.

\section{Related work}
\citet{baird1995residual} first proposed Bellman residual gradient methods, which perform gradient descent on the mean squared Bellman residual to guarantee convergence (albeit potentially slowly), and supported the approach with both analysis and simulations.
\citet{schoknecht2002optimality} showed that several classic policy-evaluation algorithms with linear function approximation (e.g., TD(0)~\citep{sutton1988learning}, LSTD~\citep{boyan1999least}, and residual-gradient methods~\citep{baird1995residual}) each converge to the minimizer of a particular quadratic objective. The resulting solutions can be interpreted as different projections of the true value function under the corresponding projection operators.
\citet{sutton2008convergent,sutton2009fast} introduced gradient temporal-difference learning for off-policy policy evaluation with linear function approximation. The key idea is to obtain stable, convergent TD-style updates by optimizing a squared projected Bellman residual objective via an equivalent saddle-point formulation and primal–dual stochastic approximation.
\citet{scherrer2010should} studies policy evaluation with linear function approximation, compares the projected TD(0) fixed point (satisfying the projected Bellman equation) with mean-squared Bellman residual minimization, and gives small MDP examples where either can perform better. Their key contribution is a unified oblique-projection framework that views both solutions as oblique projected Bellman equations, which leads to simplified characterizations and tight, interpretable error bounds.
\citet{geist2017bellman} compares Bellman-residual minimization with direct value maximization for policy search, derives proxy/performance bounds, and empirically shows that the Bellman residual can be a poor proxy unless the sampling distribution matches the optimal occupancy measure. Although their focus is policy optimization, the Bellman-residual objective they study differs from the one considered in this paper.
\citet{fujimoto2022should} show that the Bellman error can be minimized without improving value accuracy, due to error cancellation and non-uniqueness in finite or incomplete data, making it a poor proxy for value error.
\citet{sun2015online} studies squared Bellman-residual minimization for policy evaluation as an online learning problem and shows that (under an online-stability condition) no-regret updates yield small long-term prediction error. They also generalize residual-gradient methods as a special case and propose a broader family of online residual algorithms with theoretical guarantees and empirical validation.
\citet{maillard2010finite} studies policy evaluation via Bellman residual minimization with linear function approximation (using a generative model and double-sampling) and provides finite-sample generalization bounds showing the true Bellman residual of the empirical minimizer is controlled by the empirical residual plus an $O(1/\sqrt{n})$ term.
\citet{saleh2019deterministic} proposes deterministic Bellman residual minimization for policy optimization in deterministic environments, where the double-sampling issue~\citep{baird1995residual} in squared Bellman-residual minimization disappears because repeated next-state samples coincide. They provide theoretical and empirical results, and discuss limitations when stochasticity or distribution mismatch is present. Although there has been substantial work on Bellman residual minimization, most of it focuses on policy evaluation. While~\citet{saleh2019deterministic} considered Bellman residual minimization involving the max operator, a deeper theoretical investigation has not been carried out.

\section{Preliminaries}

\subsection{Markov decision problem}
We consider infinite-horizon discounted Markov decision problems~\citep{puterman2014markov}, where the agent sequentially takes actions to maximize cumulative discounted rewards. In a Markov decision process with the state space ${\cal S}:=\{ 1,2,\ldots ,|{\cal S}|\}$ and action space ${\cal A}:= \{1,2,\ldots,|{\cal A}|\}$, where $|{\cal S}|$ and $|{\cal A}|$ denote cardinalities of the respective sets, the decision maker selects an action $a \in {\cal A}$ at the current state $s\in {\cal S}$, then the state
transitions to the next state $s'\in {\cal S}$ with probability $P(s'|s,a)$, and the transition incurs a
reward $r(s,a,s') \in {\mathbb R}$, where $P(s'|s,a)$ is the state transition probability from the current state
$s\in {\cal S}$ to the next state $s' \in {\cal S}$ under action $a \in {\cal A}$, and $r:{\cal S}\times {\cal A}\times {\cal S} \to {\mathbb R}$ is the transition reward function. We use $R(s,a):={\mathbb E}[r(s,a,s')|s,a]$ for the expected one-step reward; when the reward is deterministic we simply write $r(s_k,a_k ,s_{k + 1}) =:r_{k+1}$, where $k \in \{ 0,1,\ldots \}$ is the time step. A deterministic policy, $\pi :{\cal S} \to {\cal A}$, maps a state $s \in {\cal S}$ to an action $\pi(s)\in {\cal A}$. The objective of the Markov decision problem is to find an optimal policy, $\pi^*$, such that the cumulative discounted reward over an infinite horizon is maximized, i.e., $\pi^*:= \argmax_{\pi\in \Theta} {\mathbb E}\left[\left.\sum_{k=0}^\infty {\gamma^k r_{k+1}}\right|\pi\right]$, where $\gamma \in [0,1)$ is the discount factor, $\Theta$ is the set of all deterministic policies, $(s_0,a_0,s_1,a_1,\ldots)$ is a state-action trajectory generated by the Markov chain under policy $\pi$, and ${\mathbb E}[\cdot|\pi]$ is an expectation conditioned on the policy $\pi$. Moreover, the Q-function under policy $\pi$ is defined as $Q^{\pi}(s,a)={\mathbb E}\left[ \left. \sum_{k=0}^\infty {\gamma^k r_{k+1}} \right|s_0=s,a_0=a,\pi \right], (s,a)\in {\cal S} \times {\cal A}$, and the optimal Q-function is defined as $Q^*(s,a)=Q^{\pi^*}(s,a)$ for all $(s,a)\in {\cal S} \times {\cal A}$. Once $Q^*$ is known, an optimal policy can be retrieved via the greedy policy $\pi^*(s)=\argmax_{a\in {\cal A}}Q^*(s,a)$. Throughout, we assume that the Markov decision process is ergodic so that the stationary state distribution exists.

\subsection{Definitions and notation}
Throughout the paper, we will use the following notations for compact matrix-vector representations:
\begin{align*}
P:=& \begin{bmatrix}
   P_1\\
   \vdots\\
   P_{|{\cal S}|}\\
\end{bmatrix},\; R:= \begin{bmatrix}
   R(1,\cdot) \\
   \vdots \\
   R(|{\cal S}|,\cdot) \\
\end{bmatrix},
\; Q:= \begin{bmatrix}
   Q(1,\cdot)\\
  \vdots\\
   Q(|{\cal S}|,\cdot)\\
\end{bmatrix},
\end{align*}
where $P_s\in {\mathbb R}^{|{\cal A}| \times |{\cal S}|}$ is the matrix whose $a$-th row is $P(\cdot|s,a)^\top$, $Q(s,\cdot)\in {\mathbb R}^{|{\cal A}|}$, and $R(s,\cdot)\in {\mathbb R}^{|{\cal A}|}$.
Note that $P\in{\mathbb R}^{|{\cal S}\times {\cal A}| \times |{\cal S}|  }$, $R \in {\mathbb R}^{|{\cal S}\times {\cal A}|}$, and $Q\in {\mathbb R}^{|{\cal S}\times {\cal A}|}$. In this notation, the Q-function is encoded as a single vector $Q \in {\mathbb R}^{|{\cal S}\times {\cal A}|}$ using the \emph{state-major} ordering, i.e., the block corresponding to a fixed state $s$ contains all actions $a\in {\cal A}$. In particular, the single value $Q(s,a)$ can be written as
\begin{align*}
Q(s,a) = (e_s  \otimes e_a )^\top Q,
\end{align*}
where $e_s \in {\mathbb R}^{|{\cal S}|}$ and $e_a \in {\mathbb R}^{|{\cal A}|}$ are the $s$-th basis vector (all components are $0$ except for the $s$-th component which is $1$) and the $a$-th basis vector, respectively. This convention is used throughout, including the SCBR Hessian formulas in Lemmas~\ref{thm:SCBR-tabular:hessian-1} and~\ref{thm:SCBR-LFA:hessian}.

For any stochastic policy, $\pi:{\cal S}\to \Delta_{|{\cal A}|}$, where $\Delta_{|{\cal A}|}$ is the set of all probability distributions over ${\cal A}$, we define the corresponding action transition matrix~\citep{wang2007dual} as
\begin{align}
\Pi^\pi:=\begin{bmatrix}
   e_1^\top \otimes \pi(1)^\top\\
   e_2^\top \otimes \pi(2)^\top\\
    \vdots\\
   e_{|{\cal S}|}^\top \otimes \pi(|{\cal S}|)^\top \\
\end{bmatrix}\in {\mathbb R}^{|{\cal S}| \times |{\cal S}\times {\cal A}|},\label{eq:switching-matrix}
\end{align}
where $e_s \in {\mathbb R}^{|{\cal S}|}$ is the basis vector defined above.
Then, it is well known that
$
P\Pi^\pi \in {\mathbb R}^{|{\cal S}\times {\cal A}| \times |{\cal S}\times {\cal A}|}
$
is the transition probability matrix of the state-action pair under policy $\pi$~\citep{wang2007dual}.
If we consider a deterministic policy, $\pi:{\cal S}\to {\cal A}$, the stochastic policy can be replaced by the corresponding one-hot encoding vector
$
\vec{\pi}(s):=e_{\pi(s)}\in \Delta_{|{\cal A}|},
$
where $e_a \in {\mathbb R}^{|{\cal A}|}$, and the corresponding action transition matrix is identical to~\eqref{eq:switching-matrix} with $\pi$ replaced by $\vec{\pi}$.
Given preselected state-action feature functions $\phi_1,\ldots,\phi_m:{\mathcal S}\times {\mathcal A}\to {\mathbb R}$, the matrix, $\Phi \in {\mathbb R}^{|{\mathcal S}\times {\mathcal A}| \times m}$, called the feature matrix, is defined as a matrix whose $(s,a)$-th row vector is $\phi(s,a):=\begin{bmatrix} \phi_1(s,a) &\cdots & \phi_m(s,a) \end{bmatrix}$. Throughout the paper, we assume that $\Phi \in {\mathbb R}^{|{\mathcal S}\times {\mathcal A}| \times m}$ has full column rank. All proofs are provided in the Appendix.

\section{Control Bellman residual (CBR)}
In this paper, we first consider the control Bellman residual (CBR) objective~\citep{baird1995residual} with linear function approximation (LFA), ${Q_\theta } = \Phi \theta$, and the corresponding CBR objective function
\[f(\theta ) = \frac{1}{2}\left\| {T(\Phi \theta ) - \Phi \theta } \right\|_2^2, \]
where $\theta \in {\mathbb R}^m$ is the parameter vector and $T$ is the control Bellman operator~\citep{bertsekas1996neuro} defined as
\begin{align*}
&(TQ_\theta)(s,a): = R(s,a) + \gamma \sum_{s' \in {\cal S}} {P(s'|s,a){{\max }_{a' \in {\cal A}}}Q_\theta(s',a')}.
\end{align*}
Finding a solution that minimizes the above objective function constitutes the most direct and intuitive approach to solving the control Bellman equation (CBE), $Q_\theta = TQ_\theta$. However, since the above $f$ is typically nonconvex and nondifferentiable (due to the max operator), finding its minimizer is generally a challenging problem. Therefore, in this paper we examine the properties of $f$ in greater depth. First, we can see that $f$ is a piecewise quadratic function~(\cref{thm:CBR-LFA:piecewise-quadratic}).
\begin{proposition}
$f$ is piecewise quadratic (so piecewise smooth) and continuous with the finite polyhedral cover $S_\pi$ of $\mathbb{R}^m$ defined as
\begin{align*}
{S_\pi }: =& \{ {\theta  \in {\mathbb R}^m:\pi (s) \in  \argmax_{a \in {\cal A}} Q_\theta(s,a)},\forall s\in {\cal S} \}
\end{align*}
for each $\pi \in \Theta$, where $\Theta$ is the set of all deterministic policies, and $\argmax$ is interpreted as a set-valued map, i.e., $\pi(s)$ is one of the maximizers, and ties are allowed.
\end{proposition}
In other words, $S_\pi$ is the set of all $\theta \in {\mathbb R}^m$ such that $\pi (s) \in  {\argmax _{a \in {\cal A}}}{Q_\theta }(s,a)$.
In general, the interiors of the sets $S_\pi$ need not be disjoint under linear function approximation because persistent ties can occur.
Within each region, $f$ reduces to a single quadratic function, and the cover exhibits an interesting geometric property (\cref{thm:CBR-LFA:partition-set}) summarized below.
\begin{proposition}
For each $\pi \in \Theta$, the set ${S_\pi }$ is an intersection of half-spaces and is a convex cone.
\end{proposition}
Another geometric property that helps characterize optimality is that $f$ can be upper- and lower-bounded by two quadratic functions, respectively (\cref{thm:CBR-LFA:quadratic-bound}).
\begin{proposition}\label{thm:CBR-LFA:quadratic-bound0}
$f$ is bounded by strongly convex quadratic functions as ${q_1}(\theta ) \le f(\theta ) \le {q_2}(\theta )$, where
${q_1}(\theta ): = \frac{{{{(1 - \gamma )}^2}}}{{2|{\cal S} \times {\cal A}|}}\left\| {{Q_\theta } - Q^*} \right\|_2^2$ and $ {q_2}(\theta ): = \frac{{{{(1 + \gamma )}^2}|{\cal S} \times {\cal A}|}}{2}\left\| {{Q_\theta } - {Q^*}} \right\|_2^2$.
\end{proposition}
In this case, the two quadratic functions $q_1$ and $q_2$ share the same minimizer, $ {\argmin _{\theta  \in {\mathbb R}^m}}{\left\| {{Q_\theta } - {Q^*}} \right\|_2}$, but their minimum values are generally different.
These properties help us understand the geometric structure of $f$ and will be useful later for deriving several results related to optimality.
Next, we may attempt to solve the above optimization problem by using the (Clarke) subdifferential~\citep{clarke1990optimization,bagirov2014introduction}, which is a standard tool for optimizing nonconvex and nondifferentiable functions.

\subsection{Subdifferential}
To minimize $f$, conventional methods such as gradient descent~\citep{nesterov2018lectures} cannot be directly applied, because $f$ is nondifferentiable.
Instead, we can leverage the (Clarke) subdifferential~\citep{clarke1990optimization,bagirov2014introduction}, which extends the notion of gradients to nonconvex and nondifferentiable functions.
To this end, we first establish the subdifferential of the CBR objective $f$ (\cref{thm:subdifferential-2}) as follows. In the LFA setting, not every greedy policy at $Q_\theta$ needs to be reachable from differentiability regions in the parameter space. Therefore, we use the active set
\[
\Lambda_\Phi(\theta):=\left\{\pi\in\Theta:\theta\in \overline{\operatorname{int}(S_\pi)}\right\},
\]
where the interior and closure are taken in $\mathbb R^m$.
\begin{theorem}
The subdifferential of $f$ is given by
\begin{align*}
{\partial}f(\theta ) =& \{ {\Phi ^\top}{(\gamma P{\Pi ^\beta } - I)^\top}(T{Q_\theta } - {Q_\theta }): \beta  \in {\rm conv}(\Lambda_\Phi(\theta))\} ,
\end{align*}
where ${\rm conv}$ denotes the convex hull.
\end{theorem}
We can observe that the subdifferential of $f$ admits a simple and intuitive form.
We can also derive several useful properties of $f$.
For instance, $f$ is locally Lipschitz, and its subdifferential is nonempty, convex, and compact (\cref{thm:CBR-LFA:properties}). Using the subdifferential derived above, one can see that the stationary point $\bar \theta$, i.e. $0 \in \partial f(\bar \theta )$, satisfies
\begin{align}
{\Phi ^\top}{(\gamma P{\Pi ^{\bar \beta} } - I)^\top}(T{Q_{\bar \theta }} - {Q_{\bar \theta }}) = 0 \label{eq:5}
\end{align}
for some $\bar \beta  \in {\rm conv}(\Lambda_\Phi(\bar\theta))$. Specifically, $\bar \beta$ satisfies
\begin{align}
\bar \beta  = \mathop {\argmin}\limits_{\beta  \in {\rm conv}(\Lambda_\Phi(\bar\theta)) } {\left\| {{\Phi ^\top}{{(\gamma P{\Pi ^\beta } - I)}^\top}(T Q_{\bar\theta} - Q_{\bar\theta})} \right\|_2}.\label{eq:beta-bar}
\end{align}
We can easily see that~\eqref{eq:5} can be equivalently written as
\begin{align}
{\cal R}((\gamma P{\Pi ^{\bar \beta} } - I)\Phi ) \bot (T{Q_{\bar \theta }} - {Q_{\bar \theta }}),\label{eq:6}
\end{align}
where $\cal R$ denotes the range space. Equivalently,
\[
TQ_{\bar\theta}-Q_{\bar\theta}\in {\cal N}\left(\left((\gamma P{\Pi ^{\bar \beta} } - I)\Phi\right)^\top\right)
={\cal R}((\gamma P{\Pi ^{\bar \beta} } - I)\Phi)^\perp .
\]
Thus the Bellman error lies in the orthogonal complement of the range of $(\gamma P{\Pi ^{\bar \beta} } - I)\Phi$, not in the null space of that matrix itself.
Following the ideas in~\citet{scherrer2010should},~\eqref{eq:6} can be viewed as an oblique-projection condition, provided that the relevant oblique projector is well defined. The Bellman error $T{Q_{\bar \theta }} - {Q_{\bar \theta }}$ then lies in the orthogonal complement of ${\cal R}((\gamma P{\Pi ^{\bar \beta} } - I)\Phi)$.
\begin{assumption}[Oblique-projection regularity]\label{ass:oblique-regularity}
Whenever an oblique projection $\Gamma_{\Phi|\Psi}$ is used, we assume that $\Psi^\top\Phi$ is nonsingular. Whenever an oblique projector $\Gamma_{\Phi|\Omega}$ is used in an error bound, we likewise assume that $\Omega^\top\Phi$ is nonsingular. In particular, for CBR stationary points we assume $
\Phi^\top(\gamma P\Pi^{\bar\beta}-I)^\top\Phi$ is nonsingular, where $\bar\beta\in {\rm conv}(\Lambda_\Phi(\bar\theta))$ satisfies~\eqref{eq:5}.
\end{assumption}
Under Assumption~\ref{ass:oblique-regularity}, ${\Psi_{\bar\beta}^\top}\Phi={\Phi ^\top}(\gamma P{\Pi ^{\bar \beta} } - I)^\top\Phi $ is invertible, where $\Psi_{\bar \beta} := (\gamma P{\Pi ^{\bar \beta} } - I)\Phi$, and the oblique projector onto ${\cal R}(\Phi)$ along ${\cal N}(\Psi_{\bar\beta}^\top)$ is given by
\begin{align*}
{\Gamma _{\Phi |\Psi_{\bar \beta} }}(x) = \Phi {({\Psi_{\bar \beta} ^\top}\Phi )^{ - 1}}{\Psi_{\bar \beta}^\top}x,
\end{align*}
where ${\cal N}$ denotes the null space and this operator is called the oblique projection~\citep{scherrer2010should}.
Then,~\eqref{eq:6} can be written as the following compact form:
\begin{align}
{Q_{\bar \theta }} = {\Gamma _{\Phi |\Psi_{\bar \beta} }}T{Q_{\bar \theta }},\label{eq:OP-CBE}
\end{align}
which is called the oblique projected control Bellman equation (OP-CBE).
This oblique projection generalizes the standard orthogonal projection: when $\Psi_{\bar \beta}$ is replaced by $\Phi$, it reduces to the usual orthogonal projection, $\Gamma _{\Phi |\Phi}$, onto the range space of $\Phi$. The notion and viewpoint of oblique projection were first proposed in~\citet{scherrer2010should} and have been studied only for policy evaluation. In this paper, we extend this framework to policy-optimization scenarios. The following theorem formalizes this fact.
\begin{theorem}
Under Assumption~\ref{ass:oblique-regularity}, the stationary point $\bar \theta$ with $0 \in \partial f(\bar \theta)$ satisfies the OP-CBE, ${Q_{\bar \theta }} = \Gamma _{\Phi |\Psi_{\bar \beta}} T Q_{\bar \theta}$, where $\Psi_{\bar\beta} = (\gamma P{\Pi ^{\bar\beta} } - I)\Phi$ and ${\bar\beta} \in {\rm conv}(\Lambda_\Phi(\bar\theta))$ is given in~\eqref{eq:beta-bar}.
\end{theorem}
Since a stationary point always exists, the OP-CBE admits a solution whenever Assumption~\ref{ass:oblique-regularity} holds at that stationary point (\cref{thm:CBR-LFA:stationary1}), although the solution is not unique in general.
This conclusion does not require the composite operator $\Gamma _{\Phi |\Psi_\beta}T$ to be a contraction.

\subsection{Generalized gradient descent method}
To minimize $f$, which is nonconvex and nondifferentiable, we can apply the following generalized gradient descent method~\citep{burke2005robust}:
\begin{align}
\theta_{k+1}=\theta_k-\alpha_k g_k,\label{eq:subgradient1}
\end{align}
where $g_k \in \argmin_{g\in\partial f(\theta_k)} \|g\|_2$ and $\alpha_k>0$ is a step size generated by backtracking search with the Armijo rule~\citep{Boyd2004}.
A key issue in nonsmooth nonconvex optimization is that an arbitrary choice
$g_k\in\partial f(\theta_k)$ does not necessarily yield a descent direction. A classical remedy is to select the minimum-norm subgradient~\citep{burke2005robust}, $g_k \in \argmin_{g\in\partial f(\theta_k)} \|g\|_2={\rm Proj}_{\partial f(\theta_k)}(0)$, where ${\rm Proj}$ denotes the projection. Note that by construction, $\|g_k\|_2={\rm dist}\bigl(0,\partial f(\theta_k)\bigr)$, where ${\rm dist}$ denotes the Euclidean distance between a point and a convex set. We can then prove, using standard convergence results~\citep{burke2005robust}, that when the step sizes are selected based on the Armijo rule~\citep{Boyd2004}, every limit point is a stationary point (\cref{thm:CBR-LFA:gradient-descent-convergence}).
\begin{theorem}
Let $(\theta_k)_{k\geq 0}$ be generated by $\theta_{k+1}=\theta_k-\alpha_k g_k$, where $g_k \in \argmin_{g\in\partial f(\theta_k)} \|g\|_2$ and $\alpha_k>0$ is a step size generated by backtracking search with the Armijo rule~\citep{Boyd2004}.
Then, the sequence $(\theta_k)_{k\geq 0}$ admits at least one
limit point $\bar \theta$. Every limit point $\bar \theta$ satisfies $0\in \partial f(\bar \theta)$, i.e., $\bar \theta$ is a stationary point.
\end{theorem}
Note that a stationary point is a solution that satisfies the OP-CBE in~\eqref{eq:OP-CBE}.
With LFA, approximate dynamic programming methods, such as the projected value iteration~\citep{bertsekas2011temporal}, are generally not guaranteed to converge and may even be ill-defined. In contrast, the generalized gradient descent algorithm can find a solution that satisfies the OP-CBE. This is a major advantage of CBR minimization.

\subsection{Tabular case}
In the tabular case, $\Phi = I, \theta = Q$, and $m = |{\cal S}\times {\cal A}|$, the corresponding CBR objective is reduced to $f(Q): = \frac{1}{2}\left\| TQ - Q \right\|_2^2 $ and its subdifferential is given by
\[{\partial}f(Q) = \{ {(\gamma P{\Pi ^\beta } - I)^\top}(TQ - Q):\beta  \in {\rm conv}\{ {\Lambda _Q}\} \}, \]
where $\Lambda (Q): = \left\{ {\pi  \in \Theta :\pi (s) \in {\argmax_{a \in {\cal A}}}Q(s,a)} \right\}$ is the set of all possible greedy policies for $Q$
Based on the above result, we can show that, in the tabular case, there exists a unique stationary solution $Q$ satisfying $0 \in \partial f(Q)$, and this solution coincides with the optimal $Q^*$ (\cref{thm:CBR:stationary-1}).
\begin{theorem}
The stationary point $\bar Q$ with $0 \in \partial f(\bar Q)$ in the tabular case is unique and is given by $\bar Q =Q^*$.
\end{theorem}
This function $f$ also satisfies the properties in~\cref{thm:CBR-LFA:quadratic-bound0} for the LFA case. In this case, the two quadratic functions $q_1$ and $q_2$ not only share the same minimizer $Q^*$ but also attain the same minimum value. This implies that $f$ likewise has a unique minimizer, which is the same one. From the above result, we can observe that in the tabular case, although the CBR objective is nonconvex, its stationary point is unique and coincides with the solution, $Q^*$, of the optimal control Bellman equation, $Q=TQ$, which is a favorable structure. Therefore, obtaining the stationary point amounts to solving the optimal control Bellman equation.
To minimize $f$, we can consider the generalized gradient descent in~\eqref{eq:subgradient1}, where the algorithm reduces to $Q_{k+1}=Q_k-\alpha_k g_k$. In this case, every limit point is $Q^*$ because there exists a unique stationary point $Q^*$.
Therefore, we obtain a stronger guarantee than the previous limit-point-based weak convergence: the entire sequence is guaranteed to converge to $Q^*$.
\begin{theorem}\label{thm:main-text:tabular:convergence}
Let $(Q_k)_{k\geq 0}$ be generated by $Q_{k+1}=Q_k-\alpha_k g_k$, where $g_k \in \argmin_{g\in\partial f(Q_k)} \|g\|_2$ and $\alpha_k>0$ is a step size generated by backtracking search with the Armijo rule~\citep{Boyd2004}. Then, the sequence $(Q_k)_{k\geq 0}$ converges to $Q^*$.
\end{theorem}
As suggested by the convergence results in~\cref{thm:main-text:tabular:convergence}, generalized gradient descent typically exhibits slower convergence and less favorable practical behavior. This is slower than value iteration, $Q_{k+1} = TQ_k$, which leverages the contraction property of the Bellman operator and enjoys exponential (geometric) convergence.
Nevertheless, this approach remains appealing and worthy of further investigation.

\subsection{Error bounds}
As studied in~\citet{geist2017bellman}, in the policy evaluation setting, the objective $f$ is a strongly convex quadratic function and therefore admits a unique stationary point. In the control setting, there may exist multiple stationary points, and it is generally difficult to derive a relationship between the policy associated with each stationary point and the optimal policy. However, by leveraging the following general bound (\cref{thm:CBR-LFA:bound0}), we can infer, to some extent, the relationship between the solution obtained from a stationary point and the optimal policy as well as the optimal Q-function:
\begin{proposition}\label{thm:CBR-LFA:bound5}
For any $\theta \in {\mathbb R}^m$, we have
\begin{align}
{\left\| {{Q_\theta } - {Q^*}} \right\|_\infty } \le& \frac{{\sqrt 2 }}{{1 - \gamma }}\sqrt {f(\theta )} ,\label{eq:CBR:error-bound1}\\
{\left\| {{Q^{{\pi _\theta }}} - {Q^*}} \right\|_\infty } \le& \frac{{2\sqrt 2 \gamma }}{{1 - \gamma}}\sqrt {f(\theta )},\label{eq:CBR:error-bound2}
\end{align}
where ${\pi _\theta }(s) :=  \argmax _{a \in {\cal A}} Q_\theta (s,a)$.
\end{proposition}
That is, let us suppose we select a minimizer, $\theta$, of $f$ among the stationary points. Then, this amounts to finding a parameter $\theta$ that at least minimizes the upper bounds on the right-hand sides of~\eqref{eq:CBR:error-bound1} and~\eqref{eq:CBR:error-bound2}. In particular, the second bound~\eqref{eq:CBR:error-bound2} provides an upper bound on the distance between the Q-function, $Q^{{\pi _\theta }}$, induced by the greedy policy corresponding to our computed $Q_\theta$ and $Q^*$. By minimizing $f$, we can find the best $\theta$ in terms of the upper bounds given above. Moreover, the following result shows that if $Q^*$ is sufficiently close to the function representation space ${\cal R}(\Phi)$, then $Q_{\theta ^*}$ is close to $Q^*$, where $\theta ^*$ is a minimizer of $f$ (\cref{thm:CBR-LFA:bound4}).
\begin{proposition}\label{thm:CBR-LFA:bound6}
Suppose that ${\theta ^*} =  {\argmin _{\theta  \in {\mathbb R}^m}}f(\theta )$. Then, we have
\begin{align}
{\left\| {{Q_{{\theta ^*}}} - {Q^*}} \right\|_\infty } \le \frac{{(1 + \gamma )\sqrt {|{\cal S} \times {\cal A}|} }}{{1 - \gamma }}{\left\| {{\Gamma _{\Phi |\Phi }}{Q^*} - {Q^*}} \right\|_2}.\label{eq:10}
\end{align}
\end{proposition}
In~\eqref{eq:10}, the term $\| \Gamma _{\Phi |\Phi }Q^* - Q^* \|_2$ represents the approximation error of $Q^*$, i.e., the distance between $Q^*$ and its best possible function approximation ${Q_\theta } = {\Gamma _{\Phi |\Phi }}{Q^*}$, where $\Gamma _{\Phi |\Phi }$ is the orthogonal projection onto ${\cal R}(\Phi)$. Hereafter, we refer to this as the approximation error.
Therefore,~\eqref{eq:10} implies that if the approximation error is arbitrarily small, then the error between $Q^*$ and $Q_{\theta ^*}$ can also be made arbitrarily small.
Several other related error bounds are summarized below (\cref{thm:CBR-LFA:quadratic-bound2}).
\begin{proposition}
Let us define the minimizers
\begin{align*}
\theta_1^*: = & \argmin _{\theta  \in {\mathbb R}^m} {\left\| Q_\theta - Q^* \right\|_2},\quad \theta _2^*: =  {\argmin _{\theta  \in {\mathbb R}^m}}f(\theta ).
\end{align*}
Let $e_\Phi:=\left\|\Gamma_{\Phi|\Phi}Q^*-Q^*\right\|_2$ and $n:=|{\cal S}\times{\cal A}|$. Then, we have
\begin{align*}
0 \le f(\theta _1^*) - f(\theta _2^*)
&\le \left\{\frac{(1+\gamma)^2 n}{2}-\frac{(1-\gamma)^2}{2n}\right\}e_\Phi^2,\\
\left\|\Gamma_{\Phi|\Phi}Q^* - Q_{\theta _2^*}\right\|_2
&\le \left\{1+\frac{(1+\gamma)n}{1-\gamma}\right\}e_\Phi.
\end{align*}
\end{proposition}

The local convexity result provided in Appendix~\cref{thm:CBR-LFA:bound2} shows that if $Q^*$ is sufficiently close to ${\cal R}(\Phi)$, then $f$ is strongly convex in a neighborhood of $Q^*$, there exists a unique local minimizer in that neighborhood, and this local minimizer provides a good approximation to the optimal $Q^*$.

In the next section, we focus on the soft control Bellman equation~\citep{haarnoja2017reinforcement,fox2016taming}. Because the soft control Bellman residual yields a differentiable objective, one can apply the more standard gradient descent methods~\citep{nesterov2018lectures} for differentiable functions. Accordingly, the next section presents a theoretical analysis of solutions to the soft Bellman residual, and we then discuss in greater detail gradient descent algorithms for solving it.

\section{Soft control Bellman residual (SCBR)}
In this section, we consider the so-called soft control Bellman residual (SCBR) objective function
\[f(\theta ) = \frac{1}{2}\left\| {{F_\lambda }(\Phi \theta ) - \Phi \theta } \right\|_2^2,\]
where $F_\lambda$ is the soft control Bellman operator~\citep{haarnoja2017reinforcement,fox2016taming,haarnoja2018soft} defined by
\begin{align*}
&({F_\lambda }Q_\theta )(s,a): = R(s,a) + \gamma \sum\limits_{s' \in {\cal S}} {P(s'|s,a)\lambda \ln \left( {\sum_{u \in {\cal A}} {\exp \left\{ {\frac{{Q_\theta(s',u)}}{\lambda }} \right\}} } \right)}.
\end{align*}
Here, $\lambda >0$ is the temperature parameter. In the SCBR objective above, the soft Bellman equation replaces the nonsmooth max operator by its log-sum-exp approximation. Thus, it produces a differentiable Bellman operator (actually, $C^\infty$). This construction provides the theoretical basis for soft Q-learning and the soft actor-critic framework~\citep{haarnoja2017reinforcement,fox2016taming,haarnoja2018soft}. Introducing this soft Bellman operator resolves the differentiability issue of the CBR in the previous section. However, the SCBR remains nonconvex, and hence, finding a global minimum is still not straightforward.
Moreover, analogous quadratic upper and lower bounds hold for the SCBR objective after replacing the hard optimal solution $Q^*$ by the soft fixed point $Q^*_\lambda$ of $F_\lambda$.

\subsection{Gradient}
To further understand the properties of the SCBR, we first compute its gradient as follows (\cref{thm:subdifferential-4}).
\begin{theorem}
The gradient of $f$ is given by
\[{\nabla _\theta }f(\theta ) = {\Phi ^\top}{(\gamma P{\Pi ^{{\pi _\theta }}} - I)^\top}({F_\lambda }(\Phi \theta ) - \Phi \theta )\]
where ${\pi _\theta }(j|i): = \frac{{\exp ({Q_\theta }(i,j)/\lambda )}}{{\sum_{u \in {\cal A}} {\exp ({Q_\theta }(i,u)/\lambda )} }}$ is the Boltzmann policy of $Q_\theta$.
\end{theorem}
This expression is similar to the result for CBR. One difference is that CBR uses the subdifferential, whereas here we can use the standard gradient for SCBR. From the gradient given above, we can see that the stationary point $\bar \theta$, i.e., ${\nabla _\theta }f(\bar \theta ) = 0$, satisfies ${\Phi ^\top}{(\gamma P{\Pi ^{{\pi _{\bar \theta }}}} - I)^\top}({F_\lambda }(\Phi \bar \theta ) - \Phi \bar \theta ) = 0$, which also implies ${\cal R}((\gamma P{\Pi ^{{\pi _{\bar \theta }}}} - I)\Phi ) \bot ({F_\lambda }(\Phi \bar \theta ) - \Phi \bar \theta )$, where ${\cal R}$ is the range space. This result is analogous to that for the CBR case, and is also closely related to the result in~\citet{scherrer2010should} for policy evaluation. As before, we can prove that any stationary point satisfies an oblique projected soft control Bellman equation (OP-SCBE).
\begin{theorem}
Under Assumption~\ref{ass:oblique-regularity}, the stationary point $\bar \theta$ with ${\nabla _\theta }f(\bar \theta ) = 0$ satisfies the OP-SCBE, ${Q_{\bar \theta }} = {\Gamma _{\Phi |\Psi_{\bar \theta} }}{F_\lambda }{Q_{\bar \theta }}$, where $\Psi_\theta  = (\gamma P{\Pi ^{{\pi _{\theta }}}} - I)\Phi$ for any $\theta \in {\mathbb R}^m$.
\end{theorem}
Since a stationary point always exists (\cref{thm:SCBR-LFA:stationary1}), the above result shows that the OP-SCBE has a solution whenever Assumption~\ref{ass:oblique-regularity} holds at that stationary point, although it is not unique in general. Moreover, similarly to the CBR objective, we can show that $f$ is locally smooth, i.e., the gradient is locally Lipschitz continuous (\cref{thm:SCBR-LFA:smooth1}). The above result plays an important role in establishing the convergence of the gradient descent method.

\subsection{Gradient descent method}
In the SCBR case, since the objective function is $C^\infty$, we can consider a standard gradient descent algorithm~\citep{nesterov2018lectures} to efficiently find a solution that minimizes the SCBR objective.
Using the local smoothness of $f$, we can guarantee asymptotic convergence to a stationary point (\cref{thm:SCBR-LFA:convergence}).
\begin{theorem}
Let us consider the gradient descent iterates, ${\theta _{k + 1}} = {\theta _k} - \alpha {\nabla _\theta }f({\theta _k})$, for $k=0,1,\ldots$ with any initial point $\theta_0 \in {\mathbb R}^m$. Fix any $c>f(\theta_0)$ and set $\mathcal L_c:=\{\theta:f(\theta)\le c\}$. For the chosen step size $\alpha$, define the compact convex set
\[
D_\alpha:=\operatorname{conv}\left(\mathcal L_c\cup\{\theta-\alpha\nabla_\theta f(\theta):\theta\in\mathcal L_c\}\right),
\]
and let $L_\alpha$ be a Lipschitz constant of $\nabla_\theta f$ on $D_\alpha$. If $0<\alpha<2/L_\alpha$, then the iterates satisfy $\lim_{k \to \infty } {\left\| {{\nabla _\theta }f({\theta _k})} \right\|_2} = 0$ and
\[{\min _{0 \le i \le N}}\left\| {{\nabla _\theta }f({\theta _i})} \right\|_2^2 \le \frac{{f({\theta _0}) - {{\min }_{\theta  \in {\mathbb R}^m}}f(\theta )}}{{(N + 1)\alpha \left( {1 - \frac{{\alpha {L_\alpha}}}{2}} \right)}}.\]
\end{theorem}
The above convergence result is stronger than the one for the nondifferentiable CBR case. For CBR, we could only establish convergence of limit points, whereas for SCBR, differentiability allows us to guarantee asymptotic stationarity of the iterates.
With LFA, approximate soft value iteration methods, such as projected value iteration, are generally not guaranteed to converge and may even be ill-defined. In contrast, the gradient descent algorithm above can find a solution that satisfies the OP-SCBE. This is a major advantage of the SCBR minimization.

\subsection{Tabular case}
In the tabular case, $\Phi = I,\theta = Q$, and $m=|{\cal S}\times {\cal A}|$, the SCBR objective function is reduced to $f(Q): = \frac{1}{2}\left\| {{F_\lambda }Q - Q} \right\|_2^2$. In this case, the gradient of $f$ is given by
\[{\nabla_Q}f(Q) = (\gamma P{\Pi ^{{\pi _Q}}} - I)^\top({F_\lambda }Q - Q),\]
where ${\pi _Q}(j|i): = \frac{{\exp (Q(i,j)/\lambda )}}{{\sum_{u \in {\cal A}} {\exp (Q(i,u)/\lambda )} }}$ is the Boltzmann policy of $Q$.
As in the previous section for the CBR, the SCBR has a unique stationary point that coincides with the solution of the soft control Bellman equation (SCBE), $Q_\lambda ^*$ satisfying $Q_\lambda ^* = F_\lambda Q_\lambda ^*$.
\begin{theorem}
The stationary point $\bar Q$ with ${\nabla _Q}f(\bar Q) = 0$ is the unique solution $\bar Q =Q_\lambda ^*$ to the soft control Bellman equation $Q_\lambda ^*={F_\lambda }Q_\lambda ^*$.
\end{theorem}
This is a useful fact: in the tabular case, although the SCBR is nonconvex, if we use gradient descent to locate a stationary point, that point will be the unique solution of the SCBE. Because the tabular setting is a special case of the LFA, the SCBR objective function is locally smooth. Moreover, the tabular setting admits several stronger properties than the linear function approximation setting. In particular, we can prove that $f$ is locally strongly convex (\cref{thm:SCBR:strong-convex-1}) in a neighborhood of the unique solution, $Q_\lambda ^*$, to the soft Bellman equation, $Q_\lambda ^*=F_\lambda Q_\lambda ^*$.
Furthermore, we can prove that the SCBR satisfies the Polyak–{\L}ojasiewicz (PL) condition~\citep{karimi2016linear} on any compact set (\cref{thm:SCBR:PL-condition-1}).
\begin{lemma}
$f$ is Polyak–{\L}ojasiewicz in any compact set $C$, i.e., it satisfies the following inequality:
\[f(Q) - f(Q_\lambda ^*) \le L(C)\left\| {{\nabla _Q}f(Q)} \right\|_2^2,\quad \forall Q \in C\]
for a constant $L(C)>0$.
\end{lemma}
This PL property is instrumental for proving that gradient descent converges at an exponential (i.e., linear) rate to the unique solution within that set~\citep{karimi2016linear}.
Building on the foregoing properties, we can establish that the gradient descent algorithm converges exponentially fast to the unique solution of the SCBE (\cref{thm:SCBR-tabular:convergence}).
\begin{theorem}\label{thm:SCBR-tabular:convergence}
Let us consider the gradient descent iterates, $Q_{k + 1} = Q_k - \alpha {\nabla _Q}f(Q_k)$, for $k=0,1,\ldots$ with any initial point $Q_0 \in {\mathbb R}^{|{\cal S}\times {\cal A}|}$. Fix $c>f(Q_0)$ and set ${\cal L}_c:=\{Q:f(Q)\le c\}$. For the chosen step size $\alpha$, define
\[
D_\alpha:=\operatorname{conv}\left({\cal L}_c\cup\{Q-\alpha\nabla_Q f(Q):Q\in{\cal L}_c\}\right),
\]
and let $L_\alpha>0$ be a Lipschitz constant of $\nabla_Q f$ on $D_\alpha$. Suppose $l_c>0$ satisfies
$f(Q)-f(Q_\lambda^*)\le l_c\|\nabla_Q f(Q)\|_2^2$ for all $Q\in {\cal L}_c$.
If $0<\alpha < \frac{2}{L_\alpha}$ and ${l_c} > \alpha\left(1-\frac{{L_\alpha}}{2}{\alpha}\right)$, the iterates satisfy
\begin{align*}
f(Q_k) - f(Q_\lambda ^*)\le& {\left( {1 - \alpha \left( {1 - \frac{{\alpha {L_\alpha}}}{2}} \right)\frac{1}{{{l_c}}}} \right)^k}\times \left[ {f({Q_0}) - f(Q_\lambda ^*)} \right],
\end{align*}
where $l_c$ is any such Polyak--{\L}ojasiewicz constant on ${\cal L}_c$.
\end{theorem}
Compared to the CBR setting, the above convergence result for SCBR is stronger. In particular, it guarantees exponential convergence to the unique solution even with a constant step size, which is substantially more powerful than the CBR case, where one typically establishes only asymptotic convergence to $Q^*$ under a backtracking line search.

\subsection{Error bounds}
The CBR error bounds that compare $Q_\theta$ with the Bellman fixed point carry over to the SCBR setting after replacing $Q^*$ by the soft fixed point $Q^*_\lambda$ and $T$ by $F_\lambda$. For example, $\|Q_\theta-Q^*_\lambda\|_\infty\le \frac{\sqrt 2}{1-\gamma}\sqrt{f(\theta)}$. If $\theta^*\in\argmin_{\theta\in\mathbb R^m} f(\theta)$, then $\|Q_{\theta^*}-Q^*_\lambda\|_\infty\le \frac{(1+\gamma)\sqrt{|\mathcal S\times\mathcal A|}}{1-\gamma}\|\Gamma_{\Phi|\Phi}Q^*_\lambda-Q^*_\lambda\|_2$.
The corresponding local convexity result for SCBR is provided in Appendix~\cref{thm:SCBR-LFA:bound1}.

As illustrated in~\cref{ex:SCBR-FrozenLake} of the Appendix, this gradient-based SCBR viewpoint can also be advantageous in practice: even when projected value iteration becomes unstable because the projected Bellman operator need not be contractive, SCBR minimization can still converge to an approximate solution and yield a usable policy.


\section{Conclusion}
This paper develops a convergence and structural analysis of control Bellman residual (CBR) minimization for policy optimization. Despite the nonconvex/nonsmooth max operator, we show CBR is locally Lipschitz and piecewise quadratic, derive its Clarke subdifferential, and prove descent methods converge to Clarke stationary points. We also analyze a differentiable soft variant (SCBR) and relate its stationarity to a projected/oblique Bellman equation perspective.

\bibliographystyle{plainnat}
\bibliography{reference}

\newpage

\appendix

\renewcommand{\thesection}{A.\arabic{section}}
\setcounter{section}{0}
\renewcommand{\thesubsection}{\thesection.\arabic{subsection}}

\renewcommand{\thetheorem}{A.\arabic{theorem}}
\setcounter{theorem}{0}

\renewcommand{\thelemma}{A.\arabic{lemma}}
\setcounter{lemma}{0}
\renewcommand{\thedefinition}{A.\arabic{definition}}
\setcounter{definition}{0}

\renewcommand{\theproposition}{A.\arabic{proposition}}
\setcounter{proposition}{0}
\renewcommand{\theassumption}{A.\arabic{assumption}}
\setcounter{assumption}{0}

\section{Definitions, lemmas, and notation}

\subsection{Oblique projection}
Here, we briefly describe the concept of oblique projection~\citep{scherrer2010should} before proceeding. First, the oblique projection is defined as follows.
\begin{definition}[Oblique projection]
Let $\Psi \in {\mathbb R}^{n\times m}$ be a constant matrix. For any convex set $C \subseteq {\mathbb R}^n$, the oblique projection is the (possibly set-valued) map defined by
\[{\Gamma_\Psi }(y) \in  {\argmin _{x \in C}}\frac{1}{2}\left\| {{\Psi ^\top}(x - y)} \right\|_2^2.\]
\end{definition}
When $C$ is the column space of the feature matrix $\Phi\in {\mathbb R}^{n\times m}$, then the optimization problem defining the oblique projection onto the range space of $\Phi$, ${\cal R}(\Phi)$, can be written as
\[{\Gamma _{\Phi |\Psi }}(y) \in  {\argmin _{x \in {\cal R}(\Phi)}}\frac{1}{2}\left\| {{\Psi ^\top}(x - y)} \right\|_2^2.\]
The geometric meaning of this oblique projection is illustrated in~\cref{fig:1}.
\begin{figure}[ht!]
\centering\includegraphics[width=0.6\textwidth, keepaspectratio]{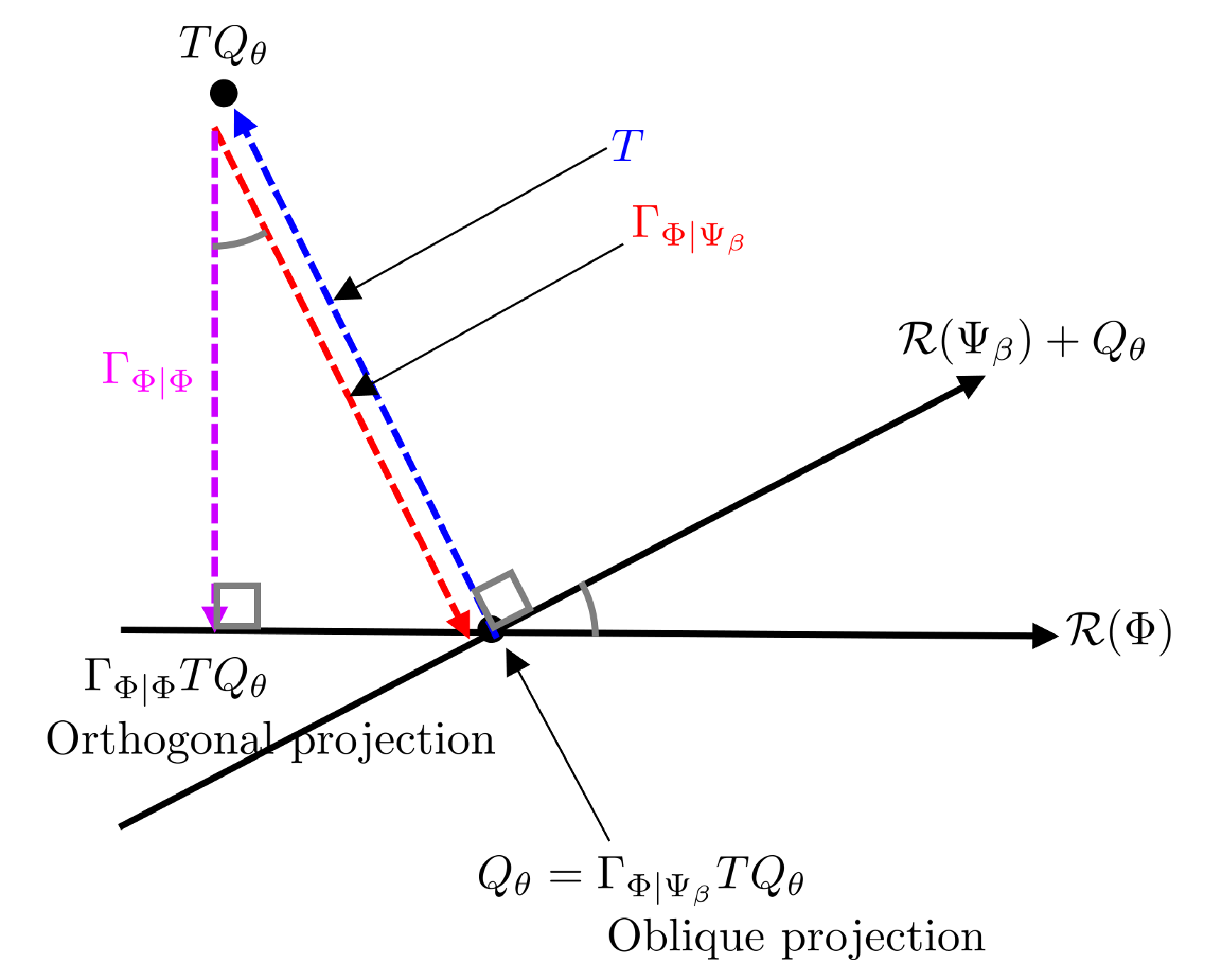}
\caption{Illustration of the oblique projection}\label{fig:1}
\end{figure}
Although the oblique projection above is formulated as an optimization problem, under certain conditions it can be expressed simply as a matrix product.
\begin{lemma}
Let $\Psi, \Phi\in {\mathbb R}^{n\times m}$ be some constant matrices.
If ${\Psi ^\top}\Phi \in {\mathbb R}^{m \times m}$ is invertible, then
\[{\Gamma _{\Phi |\Psi }}(y) = \Phi {({\Psi ^\top}\Phi )^{ - 1}}{\Psi ^\top}y.\]
\end{lemma}
\begin{proof}
First of all, note that ${\Gamma _{\Phi |\Psi }}(y) $ can be written as
\[{\Gamma _{\Phi |\Psi }}(y) =  {\argmin _{\theta  \in {\mathbb R}^m}}\frac{1}{2}\left\| {{\Psi ^\top}(\Phi \theta  - y)} \right\|_2^2.\]
Since the objective function of the above optimization is a strongly convex quadratic function, by the first-order optimality condition, the unique global optimizer satisfies
\[{\nabla _\theta }\frac{1}{2}\left\| {{\Psi ^\top}(\Phi \theta  - y)} \right\|_2^2 = {\Phi ^\top}\Psi {\Psi ^\top}\Phi \theta  - {\Phi ^\top}\Psi {\Psi ^\top}y = 0,\]
Equivalently, with $B:=\Psi^\top\Phi$, the normal equation is $B^\top B\theta=B^\top\Psi^\top y$, and since $B$ is invertible, $\theta=B^{-1}\Psi^\top y$. Hence $\Phi\theta=\Phi(\Psi^\top\Phi)^{-1}\Psi^\top y$.
\end{proof}

\subsection{Markov decision problem}

Here, we briefly summarize several standard properties related to Markov decision processes and Markov decision problems. These results will be used later to prove various statements in this paper.
We first introduce a lemma to bound the norm of a matrix inverse:
\begin{lemma}[Page 351 in~\citet{horn2012matrix}]\label{thm:matrix_inversion}
If the square matrix $M\in {\mathbb R}^{n\times n}$ satisfies $||M||<1$ for some matrix norm $||\cdot||$, then $I-M$ is nonsingular, and
    \begin{align*}
        \left\|(I-M)^{-1} \right\| \leq \frac{1}{1-\left\|M\right\|}.
    \end{align*}
\end{lemma}
The above lemma can be used to prove the following result, which is frequently used in MDPs.
\begin{lemma}\label{thm:matrix_inversion2}
For any policy $\pi \in \Delta_{|{\cal A}|}$, where $\Delta_{|{\cal A}|}$ is the set of all possible probability distributions over ${\cal A}$, the matrix $\gamma P{\Pi ^\pi } - I$ is always nonsingular.
\end{lemma}
\begin{proof}
Since ${\left\| {\gamma P{\Pi ^\pi }} \right\|_\infty } = \gamma {\left\| {P{\Pi ^\pi }} \right\|_\infty } = \gamma  < 1$, the proof is completed by directly applying~\cref{thm:matrix_inversion}.
\end{proof}
The following two lemmas are also results that are used very frequently in our theoretical analysis throughout this paper.
\begin{lemma}\label{thm:fundamental1}
For any $Q\in {\mathbb R}^{|{\cal S}\times {\cal A}|}$, we have ${\left\| {Q - {Q^\pi }} \right\|_\infty } \le \frac{1}{{1 - \gamma }}{\left\| {{T^\pi }Q - Q} \right\|_\infty } \le \frac{1}{{1 - \gamma }}{\left\| {{T^\pi }Q - Q} \right\|_2}$.
\end{lemma}
\begin{proof}
Using the contraction property of $T^\pi$, one gets
\begin{align*}
{\left\| {Q - Q^\pi} \right\|_\infty } =& {\left\| {T^\pi Q - T^\pi {Q^\pi} + Q - T^\pi Q} \right\|_\infty }\\
\le& {\left\| {T^\pi Q - T^\pi {Q^\pi}} \right\|_\infty } + {\left\| {Q - T^\pi Q} \right\|_\infty }\\
\le& \gamma {\left\| {Q - {Q^\pi}} \right\|_\infty } + {\left\| {Q - T^\pi Q} \right\|_\infty }.
\end{align*}
Rearranging terms leads to $(1 - \gamma ){\left\| {Q - {Q^\pi }} \right\|_\infty } \le {\left\| {{T^\pi }Q - Q} \right\|_\infty } \le {\left\| {{T^\pi }Q - Q} \right\|_2}$. This completes the proof.
\end{proof}
\begin{lemma}\label{thm:fundamental2}
For any $Q\in {\mathbb R}^{|{\cal S}\times {\cal A}|}$, we have ${\left\| {T^\pi Q - Q} \right\|_\infty } \le (1 + \gamma ){\left\| {Q - {Q^\pi}} \right\|_\infty }$.
\end{lemma}
\begin{proof}
Using the contraction property of $T$, one gets
\begin{align*}
{\left\| {T^\pi Q - Q} \right\|_\infty } =& {\left\| {T^\pi Q - T^\pi {Q^\pi} - Q + {Q^\pi}} \right\|_\infty }\\
\le& {\left\| {T^\pi Q - T^\pi {Q^\pi}} \right\|_\infty } + {\left\| {Q - {Q^\pi}} \right\|_\infty }\\
\le& (1 + \gamma ){\left\| {Q - {Q^\pi}} \right\|_\infty }.
\end{align*}
This completes the proof.
\end{proof}

Similar to the two lemmas above, the following two lemmas, applied to the optimal Bellman operator $T$, also play an important role in this paper.
\begin{lemma}\label{thm:fundamental4}
For any $Q\in {\mathbb R}^{|{\cal S}\times {\cal A}|}$, we have
\[{\left\| {Q - {Q^*}} \right\|_\infty } \le \frac{1}{{1 - \gamma }}{\left\| {TQ - Q} \right\|_\infty } \le \frac{1}{{1 - \gamma }}{\left\| {TQ - Q} \right\|_2}.\]
\end{lemma}
\begin{proof}
Using the contraction property of $T^\pi$, one gets
\begin{align*}
{\left\| {Q - {Q^*}} \right\|_\infty } =& {\left\| {TQ - T{Q^*} + Q - TQ} \right\|_\infty }\\
\le& {\left\| {TQ - T{Q^*}} \right\|_\infty } + {\left\| {Q - TQ} \right\|_\infty }\\
\le& \gamma {\left\| {Q - {Q^*}} \right\|_\infty } + {\left\| {Q - TQ} \right\|_\infty }.
\end{align*}
Rearranging terms leads to $(1 - \gamma ){\left\| {Q - {Q^*}} \right\|_\infty } \le {\left\| {TQ - Q} \right\|_\infty } \le {\left\| {TQ - Q} \right\|_2}$. This completes the proof.
\end{proof}

\begin{lemma}\label{thm:fundamental5}
For any $Q\in {\mathbb R}^{|{\cal S}\times {\cal A}|}$, we have ${\left\| {TQ - Q} \right\|_\infty } \le (1 + \gamma ){\left\| {Q - {Q^*}} \right\|_\infty }$.
\end{lemma}
\begin{proof}
Using the contraction property of $T$, one gets
\begin{align*}
{\left\| {TQ - Q} \right\|_\infty } =& {\left\| {TQ - T{Q^*} - Q + {Q^*}} \right\|_\infty }\\
\le& {\left\| {TQ - T{Q^*}} \right\|_\infty } + {\left\| {Q - {Q^*}} \right\|_\infty }\\
\le& (1 + \gamma ){\left\| {Q - {Q^*}} \right\|_\infty }.
\end{align*}
This completes the proof.
\end{proof}

We rely on results from~\citet{geist2017bellman}. To keep the presentation self-contained, we summarize the main statements and reproduce their proofs.
\begin{lemma}[Results from~\citet{geist2017bellman}]\label{thm:fundamental3}
For a given policy $\pi$, let us consider the Bellman residual objective
\[f(\theta ) = \frac{1}{2}\left\| {{T^\pi }\Phi \theta  - \Phi \theta } \right\|_2^2\]
for policy evaluation, and let $\Psi : = (\gamma P{\Pi ^\pi } - I)\Phi$. Its stationary point satisfying ${\nabla _\theta }f(\bar \theta ) = {\Phi ^\top}{(\gamma P{\Pi ^\pi } - I)^\top}({T^\pi }{Q_{\bar \theta }} - {Q_{\bar \theta }}) = 0$ is unique. If $\Psi^\top\Phi$ is nonsingular, then this stationary point also satisfies the oblique projected Bellman equation
\[{Q_{\bar \theta }} = {\Gamma _{\Phi |\Psi }}{T^\pi }{Q_{\bar \theta }}.\]
Moreover, we have
\[{\left\| {{Q^\pi } - {Q_{\bar \theta }}} \right\|_2} \le {\left\| {I - {\Gamma _{\Phi |\Omega }}} \right\|_2}{\left\| {{Q^\pi } - {\Gamma _{\Phi |\Phi }}{Q^\pi }} \right\|_2}\]
where $\Omega : = {(\gamma P{\Pi ^\pi } - I)^\top}(\gamma P{\Pi ^\pi } - I)\Phi $.
\end{lemma}
\begin{proof}
The uniqueness of the stationary point comes from the strong convexity of $f$.
The stationary point equation is given as
\[{\Phi ^\top}{(\gamma P{\Pi ^\pi } - I)^\top}(R + \gamma P{\Pi ^\pi }\Phi \bar \theta  - \Phi \bar \theta ) = 0.\]
Equivalently, $\Psi^\top(\Phi\bar\theta-T^\pi Q_{\bar\theta})=0$. If $\Psi^\top\Phi$ is nonsingular, this gives
\[
Q_{\bar\theta}=\Phi(\Psi^\top\Phi)^{-1}\Psi^\top T^\pi Q_{\bar\theta}=\Gamma_{\Phi|\Psi}T^\pi Q_{\bar\theta}.
\]
After algebraic manipulations, it can be seen that the above equation is equivalent to
\[\bar \theta  =  - {\left\{ {{\Phi ^\top}{{(\gamma P{\Pi ^\pi } - I)}^\top}(\gamma P{\Pi ^\pi } - I)\Phi } \right\}^{ - 1}}{\Phi ^\top}{(\gamma P{\Pi ^\pi } - I)^\top}R,\]
which can be further expressed as
\[\bar \theta  =  - {\left\{ {{\Phi ^\top}{{(\gamma P{\Pi ^\pi } - I)}^\top}(\gamma P{\Pi ^\pi } - I)\Phi } \right\}^{ - 1}}{\Phi ^\top}{(\gamma P{\Pi ^\pi } - I)^\top}(\gamma P{\Pi ^\pi } - I){(\gamma P{\Pi ^\pi } - I)^{ - 1}}R\]

Equivalently, using ${Q^\pi } = {(I - \gamma P{\Pi ^\pi })^{ - 1}}R$, we have
\[\Phi \bar \theta  = \Phi {\left\{ {{\Phi ^\top}{{(\gamma P{\Pi ^\pi } - I)}^\top}(\gamma P{\Pi ^\pi } - I)\Phi } \right\}^{ - 1}}{\Phi ^\top}{(\gamma P{\Pi ^\pi } - I)^\top}(\gamma P{\Pi ^\pi } - I){Q^\pi } = {\Gamma _{\Phi |\Omega }}{Q^\pi }\]

Using this equation, we can derive
\begin{align*}
{Q^\pi } - {Q_{\bar \theta }} =& {Q^\pi } - {\Gamma _{\Phi |\Omega }}{Q^\pi }\\
=& (I - {\Gamma _{\Phi |\Omega }}){Q^\pi }\\
=& (I - {\Gamma _{\Phi |\Omega }})(I - {\Gamma _{\Phi |\Phi }}){Q^\pi }\\
=& (I - {\Gamma _{\Phi |\Omega }})({Q^\pi } - {\Gamma _{\Phi |\Phi }}{Q^\pi }),
\end{align*}
where the third line is due to $(I - {\Gamma _{\Phi |\Omega }})(I - {\Gamma _{\Phi |\Phi }}) = I - {\Gamma _{\Phi |\Omega }}$. Finally, we obtain the desired conclusion: ${\left\| {{Q^\pi } - {Q_{\bar \theta }}} \right\|_2} \le {\left\| {I - {\Gamma _{\Phi |\Omega }}} \right\|_2}{\left\| {{Q^\pi } - {\Gamma _{\Phi |\Phi }}{Q^\pi }} \right\|_2}$.
\end{proof}

\subsection{Clarke subdifferential}
To analyze and minimize the control Bellman residual (CBR) objective function in this paper, we need to understand the notion of subdifferentials, which serve as a counterpart to derivatives for nonconvex and nondifferentiable functions, as well as several of their key properties. Accordingly, in this section we briefly review some well-established basic concepts related to such subdifferentials.
Throughout, let $f:\mathbb{R}^n\to\mathbb{R}$ be locally Lipschitz and consider the optimization problem
\[{x^*} =  {\argmin _{x \in {\mathbb R}^n}}f(x),\]
where $f$ is potentially nonconvex and nonsmooth.
First, we define the generalized directional derivative.
\begin{definition}[\citet{clarke1990optimization,bagirov2014introduction}]
Let $f:{\mathbb R}^n\to {\mathbb R}$ be a locally Lipschitz continuous function at $x$. The generalized directional derivative of $f$ at $x\in {\mathbb R}^n$ in the direction $d \in {\mathbb R}^n$ is defined by
\begin{align*}
f^{\circ}(x;d) :=\limsup_{\substack{y\to x\\ t\downarrow 0}}
\frac{f(y+td)-f(y)}{t}.
\end{align*}
\end{definition}

Then, based on the above definition, the Clarke subdifferential (Clarke generalized gradient) is given as follows.
\begin{definition}[\citet{clarke1990optimization,bagirov2014introduction}]
Let $f:{\mathbb R}^n\to {\mathbb R}$ be a locally Lipschitz continuous function at $x$. The subdifferential of $f$ at $x\in {\mathbb R}^n$ is the set $\partial f(x)$ of vectors such that
\begin{align*}
\partial f(x) :=\Bigl\{v\in\mathbb{R}^n:\ \langle v,d\rangle \le f^{\circ}(x;d)\ \ \forall d\in\mathbb{R}^n\Bigr\}.
\end{align*}
\end{definition}

It is standard that $\partial f(x)$ is a nonempty, convex, and compact set~\citep[Theorem~3.3, page~64]{bagirov2014introduction}.
\begin{lemma}[\citet{bagirov2014introduction}, Theorem~3.3]\label{thm:locally-lipschitz-compact-subdiff}
Let $f:{\mathbb R}^n\to {\mathbb R}$ be a locally Lipschitz continuous function at $x$ with a Lipschitz constant $L>0$. Then, the subdifferential $\partial f(x)$ is a nonempty, convex, and compact set such that
\[\partial f(x) \subseteq B(0;L),\]
where $B(z;r)$ is a ball centered at $z$ with radius $r>0$.
\end{lemma}
Moreover, the generalized directional derivative above can be expressed equivalently as follows.
\begin{lemma}[\citet{bagirov2014introduction}, Theorem~3.4]
Let $f:{\mathbb R}^n\to {\mathbb R}$ be a locally Lipschitz continuous function at $x$. Then, for every $d \in {\mathbb R}^n$, one has
\begin{align*}
f^{\circ}(x;d)=\max_{v\in\partial f(x)}\langle v,d\rangle.
\end{align*}
\end{lemma}

Next, we summarize several properties related to continuity.
\begin{lemma}[\citet{bagirov2014introduction}, Theorem~3.2]
Let $f:{\mathbb R}^n\to {\mathbb R}$ be a locally Lipschitz continuous function at $x$. Then, the function $(x,d) \mapsto {f^ \circ }(x;d)$ is upper semicontinuous.
\end{lemma}
\begin{lemma}[\citet{bagirov2014introduction}, Theorem~3.5]
Let $f:{\mathbb R}^n\to {\mathbb R}$ be a locally Lipschitz continuous function at $x$. Then, the set-valued function $x \mapsto \partial f(x)$ is upper semicontinuous.
\end{lemma}

The following result is essential when calculating subgradients in practice.
Namely, the subdifferential can be constructed as a convex hull of all possible limits of gradients at point $x_k$ converging to $x$.
\begin{lemma}[\citet{bagirov2014introduction}, Theorem~3.9]\label{thm:subdifferential-6}
Let $f:{\mathbb R}^n\to {\mathbb R}$ be a locally Lipschitz continuous function at $x\in {\mathbb R}^n$.
Then,
\[{\partial}f(x): = {\rm conv}\left( {\left\{ {g \in {\mathbb R}^n:\exists ({x_i})_{i = 1}^\infty \,,\lim_{i \to \infty } {x_i} = x, \lim_{i \to \infty } {{\left. {{\nabla _x}f(x)} \right|}_{x = x_i}} = g} \right\}} \right).\]
\end{lemma}
Next, we define the notions of a generalized stationary point and a descent direction.
\begin{definition}[\citet{bagirov2014introduction}, Definition~4.3]
A point $x$ satisfying $0\in\partial f(x)$ is called a stationary point of $f$.
\end{definition}
\begin{definition}[\citet{bagirov2014introduction}, Definition~4.5]
Direction $d\in {\mathbb R}^n$ is called a descent direction for $f:{\mathbb R}^n \to {\mathbb R}$ at $x\in {\mathbb R}^n$, if there exists $\varepsilon>0$ such that for all $t\in (0,\varepsilon]$
\[f(x + td) < f(x).\]
\end{definition}

Using the subdifferential, we can obtain the following generalized first-order necessary conditions for a point to be an extremum.
\begin{lemma}[\citet{bagirov2014introduction}, Theorem~3.17]\label{thm:optimality1}
Let $f:{\mathbb R}^n\to {\mathbb R}$ be a locally Lipschitz continuous function at $x$. If it attains its extremum at $x$, then
\begin{align*}
0\in\partial f(x).
\end{align*}
\end{lemma}
\begin{lemma}[\citet{bagirov2014introduction}, Theorem~4.1]
Let $f:{\mathbb R}^n\to {\mathbb R}$ be a locally Lipschitz continuous function at $x^*$. If $f$ attains its local minimum at $x^*$, then
\[0 \in \partial f({x^*}),\quad {f^\circ }({x^*};d) \ge 0,\quad \forall d \in {\mathbb R}^n.\]
\end{lemma}

\begin{lemma}[\citet{bagirov2014introduction}, Theorem~4.5]\label{thm:descent-direction1}
Let $f:{\mathbb R}^n\to {\mathbb R}$ be a locally Lipschitz continuous function at $x\in {\mathbb R}^n$.
The direction $d \in {\mathbb R}^n$ is a descent direction for $f$ at $x$ if
\[{g^\top}d < 0,\quad \forall g \in \partial f(x)\]
or equivalently, $ {f^ \circ }(x;d) < 0$.
\end{lemma}
\begin{lemma}[\citet{bagirov2014introduction}, Corollary~4.1]
Let $f:{\mathbb R}^n\to {\mathbb R}$ be a locally Lipschitz continuous function at $x \in {\mathbb R}^n$.
Then either $0\in \partial f(x)$ or there exists a descent direction $d\in {\mathbb R}^n$ for $f$ at $x \in {\mathbb R}^n$.
\end{lemma}
\begin{lemma}[\citet{burke2005robust}]\label{thm:descent-direction2}
Let $f:{\mathbb R}^n\to {\mathbb R}$ be a locally Lipschitz continuous function at $x\in {\mathbb R}^n$.
Fix $x\in\mathbb{R}^n$ and let $g={\rm Proj}_{\partial f(x)}(0)$ be the Euclidean projection of $0$ onto $\partial f(x)$ or the minimum-norm subgradient
\[g = {{\rm Proj}_{\partial f(x)}}(0) =  {\argmin _{v \in \partial f(x)}}{\left\| v \right\|_2}.\]
Then  $f^{\circ}(x;-g)\le -\|g\|_2^2$. In particular, if $g\neq 0$, then $-g$ is a strict descent direction.
\end{lemma}
\begin{proof}
Since $\partial f(x)$ is closed and convex, the projection optimality condition~\citep[Proposition~2.2.1, page~88]{bertsekas2003convex} gives
\begin{align*}
\langle v-g,\,0-g\rangle \le 0\qquad \forall v\in \partial f(x),
\end{align*}
equivalently, it can be rewritten as follows:
\begin{align*}
\langle v,g\rangle \ge \|g\|^2\qquad \forall v\in \partial f(x).
\end{align*}
Let $d=-g$. Then for all $v\in \partial f(x)$,
\begin{align*}
\langle v,d\rangle = -\langle v,g\rangle \le -\|g\|^2.
\end{align*}
Using the support-function representation of the Clarke directional derivative,
\begin{align*}
f^{\circ}(x;d)=\max_{v\in \partial f(x)}\langle v,d\rangle,
\end{align*}
we obtain $f^{\circ}(x;-g)\le -\|g\|^2$. This completes the proof.
\end{proof}

In particular, when $f$ is a piecewise quadratic function, we can obtain several special-case properties that are useful in practice.
Before proceeding, we formally define the class of piecewise quadratic functions considered in this paper.
\begin{definition}\label{def:piecewise-quadratic}
A function $f:\mathbb{R}^n\to\mathbb{R}$ is called (finite) piecewise quadratic if there exist a finite index set $\mathcal I=\{1,\dots,N\}$, a polyhedral partition $\{U_i\}_{i\in\mathcal I}$ of $\mathbb{R}^n$ (i.e., $U_i$ are polyhedra with pairwise disjoint interiors and $\bigcup_i U_i=\mathbb{R}^n$), and quadratic polynomials
\[
f_i(x)=\tfrac12 x^\top H_i x + b_i^\top x + c_i,
\qquad H_i=H_i^\top,
\]
such that
\[
f(x)=f_i(x)\qquad \forall x\in U_i.
\]
In particular, each $f_i$ is ${\cal C}^\infty$. Therefore, $f$ is piecewise ${\cal C}^\infty$ with finitely many pieces.
\end{definition}

It can be shown that any piecewise quadratic function $f$ defined as above is always locally Lipschitz.
\begin{lemma}\label{thm:piecewise-quadratic-locally-Lipschitz}
Every finite piecewise quadratic $f$ is locally Lipschitz on $\mathbb{R}^n$.
\end{lemma}
\begin{proof}
Let us consider a piecewise quadratic function defined in~\cref{def:piecewise-quadratic}.
Let us fix any $x_0\in\mathbb{R}^n$ and radius $r>0$, and consider the closed ball
\[
B(x_0,r)=\{x \in {\mathbb R}^n: \|x-x_0\|_2\le r\}.
\]
Because each $f_i$ is quadratic, its gradient is affine: $\nabla f_i(x)=H_i x + b_i$. Hence on the compact set $B(x_0,r)$,
\[{M_i}: = \max_{x \in B({x_0},r)} {\left\| {\nabla f_i(x)} \right\|_2} < \infty. \]
Let $M:=\max_{i\in\mathcal I} M_i < \infty$. Now, let us take any $x,y\in B(x_0,r)$ and consider the line segment
\[
\ell(t):=y+t(x-y),\qquad t\in[0,1].
\]
Since the sets $U_i$ form a polyhedral partition, the segment intersects only finitely many region boundaries.
Therefore, there exist breakpoints
\[
0=t_0<t_1<\cdots<t_m=1
\]
and indices $i_1,\dots,i_m\in\mathcal I$ such that $\ell(t)\in U_{i_j}$ for all $t\in[t_{j-1},t_j]$.
Let us define $z_j:=\ell(t_j)$ so that $z_0=y$ and $z_m=x$.
On each subsegment $[z_{j-1},z_j]\subset B(x_0,r)$, we have $f=f_{i_j}$, and by the mean value theorem,
\[
|f(z_j)-f(z_{j-1})|
=|f_{i_j}(z_j)-f_{i_j}(z_{j-1})|
\le \sup_{\xi\in[z_{j-1},z_j]}\|\nabla f_{i_j}(\xi)\|_2\,\|z_j-z_{j-1}\|_2
\le M\,\|z_j-z_{j-1}\|_2.
\]
Summing over $j=1,\dots,m$ yields
\[
|f(x)-f(y)|
\le \sum_{j=1}^m |f(z_j)-f(z_{j-1})|
\le M\sum_{j=1}^m \|z_j-z_{j-1}\|_2.
\]
But all $z_j$ lie on the same segment from $y$ to $x$, so the subsegment lengths add up to the total length:
\[
\sum_{j=1}^m \|z_j-z_{j-1}\|_2=\|x-y\|_2.
\]
Therefore, $|f(x)-f(y)|\le M\|x-y\|_2$ for all $x,y\in B(x_0,r)$, which implies that $f$ is Lipschitz on $B(x_0,r)$; since $x_0$ and $r$ were arbitrary, $f$ is locally Lipschitz on $\mathbb{R}^n$. This concludes the proof.
\end{proof}

\section{Generalized gradient descent method}

Let $f:\mathbb{R}^n\to\mathbb{R}$ be a locally Lipschitz function (not necessarily convex or smooth). We study a generalized gradient (subgradient) descent method~\citep{burke2005robust} based on the Clarke subdifferential and summarize a standard convergence skeleton showing that every limit point is Clarke-stationary.
A key issue in nonsmooth nonconvex optimization is that an \emph{arbitrary} choice
$g\in\partial f(x)$ does not necessarily yield a descent direction.
A classical remedy~\citep{burke2005robust} is to select the minimum-norm Clarke subgradient $g_k \in \argmin_{v\in\partial f(x_k)} \|v\|_2=\mathrm{Proj}_{\partial f(x_k)}(0)$, and update
\begin{align*}
x_{k+1}=x_k-\alpha_k g_k,
\end{align*}
with some step size $\alpha_k>0$. Note that by construction, we can easily see
\begin{align*}
\|g_k\|_2=\mathrm{dist}\bigl(0,\partial f(x_k)\bigr).
\end{align*}
The step sizes are selected based on the Armijo rule~\citep{Boyd2004}.
In particular, we first fix parameters $\bar \alpha>0$, $\beta\in(0,1)$, and $\sigma\in(0,1)$.
If $g_k=0$, stop. Otherwise, choose the step size $\alpha_k=\beta^{m_k}\bar \alpha$ with the smallest
integer $m_k\ge 0$ such that the Armijo condition with Clarke derivative holds:
\begin{equation}\label{eq:Armijo}
f(x_k - \alpha_k g_k)\ \le\ f(x_k)\ +\ \sigma\, \alpha_k\, f^{\circ}(x_k;d_k).
\end{equation}
Then, we update the parameter by
\[
x_{k+1}=x_k-\alpha_k g_k.
\]
Based on the above discussion and preliminary developments, we can now consider the algorithm presented in~\cref{algo:subgradient1}.
\begin{algorithm}[ht!]
\caption{Generalized gradient method}
\begin{algorithmic}[1]

\State Initialize $x_0$

\For{$k \in \{0,1,2,\ldots\}$}
\State Calculate $g_k \in \argmin_{v\in\partial f(x_k)} \|v\|_2^2$
\State Perform a subgradient descent step $x_{k+1}=x_k-\alpha_k g_k$, where the step size $\alpha_k>0$ is selected based on the Armijo rule
\EndFor

\end{algorithmic}\label{algo:subgradient1}
\end{algorithm}

The following lemma provides a guarantee on the existence of a step size that satisfies the Armijo rule.
\begin{lemma}[Existence of an Armijo step]\label{lem:Armijo-exists}
Suppose $f$ is locally Lipschitz and $f^\circ(x;d)<0$.
Fix any $\sigma\in(0,1)$. Then there exists $\delta>0$ such that for all $\alpha\in(0,\delta)$,
\[
f(x+\alpha d)\ \le\ f(x) + \sigma \alpha f^\circ(x;d).
\]
Consequently, the backtracking rule \eqref{eq:Armijo} terminates in finitely many reductions.
\end{lemma}
\begin{proof}
First of all, by the definition of the generalized directional derivative, one has
\begin{align*}
{f^\circ }(x;d): = \limsup_{y \to x,\alpha \downarrow 0} \frac{{f(y + \alpha d) - f(y)}}{\alpha}\ge \limsup_{\alpha \downarrow 0} \frac{{f(x + \alpha d) - f(x)}}{\alpha}.
\end{align*}
By definition of $\limsup$, for any $\varepsilon>0$ there exists $\delta>0$ such that for all
$\alpha\in(0,\delta)$,
\[\limsup_{t \downarrow 0} \frac{{f(x + td) - f(x)}}{t} + \varepsilon  \ge \frac{{f(x + \alpha d) - f(x)}}{\alpha }.\]
Combining the last two inequalities leads to
\[\frac{{f(x + \alpha d) - f(x)}}{\alpha } \le \mathop {\lim \sup }\limits_{t \downarrow 0} \frac{{f(x + td) - f(x)}}{t} + \varepsilon  \le {f^\circ }(x;d) + \varepsilon .\]
Now, let us choose $\varepsilon=(1-\sigma)(-f^\circ(x;d))>0$ (possible since $f^\circ(x;d)<0$), so that
$f^\circ(x;d)+\varepsilon=\sigma f^\circ(x;d)$. Multiplying by $\alpha$ yields the desired inequality.
\end{proof}

\begin{lemma}[Uniform Armijo step for finite piecewise quadratic functions]\label{lem:uniform-Armijo-pq}
Let $f$ be finite piecewise quadratic, let $K\subset\mathbb R^n$ be compact, and let $g(x)={\rm Proj}_{\partial f(x)}(0)$. Suppose that $\|g(x)\|_2\ge \eta>0$ for all $x\in K$. Then, for any $\sigma\in(0,1)$, there exists $\tau>0$ such that for every $x\in K$ and every $\alpha\in(0,\tau]$,
\[
f(x-\alpha g(x))\le f(x)+\sigma\alpha f^\circ(x;-g(x)).
\]
\end{lemma}
\begin{proof}
Write the finitely many quadratic pieces as $f_i(x)=\frac12x^\top H_i x+b_i^\top x+c_i$ and set $M:=\max_i\|H_i\|_2$. The compactness of $K$ and local Lipschitzness imply that $G:=\max_{x\in K}\|g(x)\|_2<\infty$. Along any segment starting at $x\in K$, the function crosses only finitely many pieces. Applying the quadratic identity on each subsegment and summing yields, for $d=-g(x)$,
\[
f(x+\alpha d)\le f(x)+\alpha f^\circ(x;d)+\frac{M}{2}\alpha^2\|d\|_2^2.
\]
By \cref{thm:descent-direction2}, $f^\circ(x;-g(x))\le-\|g(x)\|_2^2\le-\eta^2$. Hence it is enough to choose
\[
0<\tau\le \frac{2(1-\sigma)\eta^2}{M G^2},
\]
with the convention that any positive $\tau$ works when $M=0$ or $G=0$. This gives the desired uniform Armijo inequality.
\end{proof}

We are now ready to prove the convergence of~\cref{algo:subgradient1}. To guarantee convergence, we require several assumptions, which we summarize below.
\begin{assumption}\label{ass:levelset}
We assume
\begin{enumerate}
\item $f$ is a finite piecewise quadratic function, locally Lipschitz, and bounded below: $\inf_{x\in\mathbb{R}^n} f(x)>-\infty$.
\item The initial sublevel set $\mathcal{L}_c:=\{x:\ f(x)\le c\}$ with $c=f(x_0)$ is bounded (hence compact).
\end{enumerate}
\end{assumption}
The following result establishes the convergence of the algorithm under the above assumptions.
\begin{theorem}\label{thm:convergence-subgradient1}
Let $(x_k)_{k\geq 0}$ be generated by~\cref{algo:subgradient1}
under Assumption~\ref{ass:levelset}. Then the following results hold:
\begin{enumerate}
\item (Monotone decrease) $f(x_{k+1})\le f(x_k)$ for all $k \ge 0$, and $(f(x_k))_{k\ge 0}$ converges.
\item (Sufficient decrease) For all $k\geq 0$,
\begin{equation}\label{eq:suff-decrease}
f(x_{k+1}) \le f(x_k) - \sigma \alpha_k \|g_k\|_2^2,
\end{equation}
and hence $\sum_{k=0}^\infty \alpha_k \|g_k\|_2^2 < \infty$.
\item (Stationarity of limit points) The sequence $\{x_k\}\subset\mathcal{L}_c$ with $c=f(x_0)$ admits at least one
limit point. Every limit point $\bar x$ satisfies
\[
0\in \partial f(\bar x),
\]
i.e., $\bar x$ is a stationary point.
\end{enumerate}
\end{theorem}
\begin{proof}
Let us fix $k\geq 0$. If $g_k=0$, then $0\in\partial f(x_k)$ and the algorithm stops at a stationary
point. Therefore, let us assume $g_k\neq 0$.
By~\cref{thm:descent-direction2}, $f^\circ(x_k;d_k)=f^\circ(x_k;-g_k)\le -\|g_k\|_2^2<0$.
Moreover, by~\cref{lem:Armijo-exists}, since the Armijo step size exists, the backtracking rule terminates and yields $\alpha_k>0$ satisfying~\eqref{eq:Armijo}. Combining~\eqref{eq:Armijo} with $f^\circ(x_k;d_k)\le -\|g_k\|_2^2$ gives
\[
f(x_{k+1})=f(x_k+\alpha_k d_k)
\le f(x_k) + \sigma \alpha_k f^\circ(x_k;d_k)
= f(x_k) - \sigma \alpha_k \|g_k\|_2^2,
\]
which is~\eqref{eq:suff-decrease} and implies $f(x_{k+1})\le f(x_k)$.
Since $f$ is bounded below, the monotone sequence $(f(x_k))_{k\geq 0}$ converges.
Summing~\eqref{eq:suff-decrease} over $k=0,\dots,N$ yields
\[
\sum_{k=0}^N \sigma \alpha_k \|g_k\|_2^2
\le f(x_0)-f(x_{N+1})
\le f(x_0)-\inf f < \infty,
\]
and letting $N\to\infty$ gives $\sum_{k=0}^\infty \alpha_k\|g_k\|_2^2<\infty$.
By~\eqref{eq:suff-decrease}, $f(x_k)\le f(x_0)$ for all $k$, hence $x_k\in\mathcal{L}_c$ with $c=f(x_0)$ for all $k\ge 0$.
By~\cref{ass:levelset}(2), $\mathcal{L}_c$ is compact, and hence, $(x_k)_{k\geq 0}$ has at least one limit point~\citep[Theorem~3.6]{rudin1976principles}.
Let $\bar x$ be a limit point and take a subsequence $x_{k_j}\to \bar x$.
Now, let us assume for contradiction that $\bar x$ is not stationary, i.e., $0\notin\partial f(\bar x)$.
Let us define the distance $\delta := \operatorname{dist}(0,\partial f(\bar x))>0$. For locally Lipschitz $f$, the set-valued mapping $x\mapsto \partial f(x)$ is upper semicontinuous.
Therefore, by the definition of the upper semicontinuous set-valued mapping, there exists a neighborhood $U$ of $\bar x$ such that
\begin{equation}\label{eq:dist-lb}
\operatorname{dist}(0,\partial f(x)) \ge \delta/2
\qquad \forall x\in U.
\end{equation}
For each $x\in U$, let $g(x)$ denote the minimum-norm element of $\partial f(x)$.
Then,~\eqref{eq:dist-lb} yields $\|g(x)\|_2\ge \delta/2$ for all $x\in U$.
By~\cref{thm:descent-direction2}, we have
\[
f^\circ(x;-g(x)) \le -\|g(x)\|_2^2 \le -\delta^2/4
\qquad \forall x\in U.
\]
Choose a compact neighborhood $K\subset U$ of $\bar x$. Since $x_{k_j}\to\bar x$, we have $x_{k_j}\in K$ for all sufficiently large $j$. By~\cref{lem:uniform-Armijo-pq} with $\eta=\delta/2$, there exists $\tau>0$ such that the Armijo condition~\eqref{eq:Armijo} holds for every $x_k\in K$ and every trial step $t\in(0,\tau]$. Because the backtracking steps are selected from $\{\bar \alpha,\beta\bar \alpha,\beta^2\bar \alpha,\dots\}$, this yields a uniform lower bound $\alpha_k\ge \alpha_{\min}>0$ whenever $x_k\in K$.
Therefore, for such $k$,
\[
f(x_{k+1}) \le f(x_k) - \sigma \alpha_k \|g_k\|_2^2
\le f(x_k) - \sigma \alpha_{\min}\, (\delta/2)^2
= f(x_k) - c,
\]
where $c:=\sigma \alpha_{\min}\delta^2/4>0$ is a constant. Since $x_{k_j}\to \bar x$, we have $x_{k_j}\in K$ for all sufficiently large $j\ge 0$.
Therefore, the inequality above would force $f(x_k)$ to decrease by at least $c>0$ infinitely often, which contradicts the convergence of $f(x_k)$. Therefore, the contradiction proves $0\in \partial f(\bar x)$. This shows that every limit point is Clarke stationary.
\end{proof}

\section{CBR tabular case}

Let us consider the CBR objective
\[
f(Q)
:= \frac12\sum_{(s,a)\in\mathcal S\times\mathcal A}((TQ)(s,a)-Q(s,a))^2
= \frac12\|TQ-Q\|_2^2.
\]
Let $\Theta $ denote the set of deterministic policies $\pi:\mathcal S\to\mathcal A$, so $|\Theta|=|\mathcal A|^{|\mathcal S|}<\infty$.
Let us define the set
\[{S_\pi }: = \left\{ Q \in {\mathbb R}^{|{\cal S} \times {\cal A}|}:\pi (s) \in  \argmax_{a \in {\cal A}}Q(s,a) \right\}\]
for each $\pi \in \Theta$, where $\argmax$ is interpreted as a set-valued map, i.e., $\pi(s)$ is one of the maximizers; ties are allowed.

\begin{proposition}\label{thm:CBR:piecewise-quadratic}
$f$ is piecewise quadratic (so piecewise smooth) and continuous.
\end{proposition}
\begin{proof}
First of all, $f$ can be written as
\begin{align*}
f(Q) =& \frac{1}{2}{(R + \gamma P{\Pi ^{{\pi _Q}}}Q - Q)^\top}(R + \gamma P{\Pi ^{{\pi _Q}}}Q - Q)\\
=& \frac{1}{2}{Q^\top}({\gamma ^2}{(P{\Pi ^{{\pi _Q}}})^\top}P{\Pi ^{{\pi _Q}}} - \gamma P{\Pi ^{{\pi _Q}}} - \gamma {(P{\Pi ^{{\pi _Q}}})^\top} + I)Q + (\gamma {R^\top}P{\Pi ^{{\pi _Q}}} - {R^\top})Q + \frac{1}{2}{R^\top}R,
\end{align*}
where ${\pi _Q}(s) \in  {\argmax _{a \in {\cal A}}}Q(s,a)$, ${\Pi ^{{\pi _Q}}}$ depends on $Q$ and is constant within the set $S_\pi$ for each $\pi \in \Theta$.
Therefore, $f$ is quadratic on each set $S_\pi$. Since the family $\{S_\pi\}_{\pi\in\Theta}$ is finite and each $S_\pi$ is defined by finitely many linear inequalities, the arrangement of the corresponding hyperplanes induces a finite polyhedral partition refining this cover. On each cell of the refined partition, the active-policy set is constant, and hence $f$ coincides with one of the quadratic expressions above. Therefore, $f$ is finite piecewise quadratic in the sense of~\cref{def:piecewise-quadratic}. Moreover, it is continuous because for any $Q \in S_\pi$, on the boundary of $S_\pi$, it satisfies ${\pi _1}(s),{\pi _2}(s), \ldots ,{\pi _N}(s) \in  {\argmax _{a \in {\cal A}}}Q(s,a)$, and the corresponding quadratic functions, $f_1,f_2,\ldots, f_N$, are identical at the point, i.e., ${f_1}(Q) = {f_2}(Q) =  \cdots  = {f_N}(Q)$.
\end{proof}

From the above results, we know that $f$ is piecewise quadratic. However, we can obtain further insights into the regions $S_\pi$ over which it is quadratic.
\begin{lemma}
For each $\pi \in \Theta$, the set $S_\pi$ is an intersection of homogeneous half-spaces.
\end{lemma}
\begin{proof}
For each $s\in {\cal S}$, the condition $\pi(s)\in \argmax_{a\in {\cal A}} Q(s,a)$ is equivalent to the family of linear inequalities
\[Q(s,\pi (s))\; \ge \;Q(s,a),\quad \forall (s,a) \in {\cal S}\times {\cal A}.\]
Hence, the set $S_\pi$ can be written as
\[S_\pi = \bigcap\limits_{s \in {\cal S}} \; \bigcap\limits_{a \in {\cal A}} {\left\{ {Q \in {\mathbb R}^{|{\cal S} \times {\cal A}|}:\;Q(s,\pi (s)) - Q(s,a) \ge 0} \right\}}  = \left\{ {Q \in {\mathbb R}^{|{\cal S} \times {\cal A}|}:\;L{\Pi ^\pi }Q - Q \ge 0} \right\},\]
where $L\in {\mathbb R}^{|{\cal S}\times {\cal A}|\times |{\cal S}|}$ is a lifting matrix defined as $L: = I_{|{\cal S}|} \otimes {{\bf 1}_{|{\cal A}|}}$, where $I_{|{\cal S}|}$ is the identity matrix of dimension $|{\cal S}|$, $\otimes$ denotes the Kronecker product, ${\bf 1}_{|{\cal A}|}$ is the all-ones vector of dimension $|{\cal A}|$. Each set in the intersection is a closed half-space containing the origin. Therefore, it is an intersection of homogeneous half-spaces.
\end{proof}

\begin{lemma}
For each $\pi \in \Theta$, the set $S_\pi$ is a convex cone.
\end{lemma}
\begin{proof}
First of all, let us fix any $\pi \in \Theta$ and consider any $Q\in {S_\pi }$. Then, for any $\alpha \geq 0$, we have $\alpha Q\in {S_\pi }$ because
\[ \argmax _{a \in {\cal A}}Q(s,a) =  \argmax _{a \in {\cal A}}\alpha Q(s,a).\]
Therefore, ${S_\pi }$ is a cone.
Moreover, let $Q_1,Q_2\in S_\pi$ and $\alpha\in[0,1]$. For all $s\in {\cal S}$ and $a\in {\cal A}$,
\[{Q_1}(s,\pi (s)) \ge {Q_1}(s,a),\qquad {Q_2}(s,\pi (s)) \ge {Q_2}(s,a).\]
Multiplying the first inequality by $\alpha \in [0,1]$ and the second by $1-\alpha$, and adding, yields
\[\alpha {Q_1}(s,\pi (s)) + (1 - \alpha ){Q_2}(s,\pi (s)) \ge \alpha {Q_1}(s,a) + (1 - \alpha ){Q_2}(s,a),\]
which shows that $\alpha Q_1+(1-\alpha)Q_2\in S_\pi$. Hence $S_\pi$ is convex.
Combining the two parts, we conclude that $S_\pi$ is a convex cone.
\end{proof}

\begin{lemma}\label{thm:CBR:properties1}
The following statements hold:
\begin{enumerate}
\item $f$ is bounded below.

\item $f$ is locally Lipschitz continuous.

\item The Clarke subdifferential of $f$ given by $\partial f(\theta) = \{ {(\gamma P{\Pi ^\beta } - I)^\top}(TQ - Q):\beta  \in \mathrm{conv}(\Lambda (Q))\}$ is nonempty, convex, and bounded.

\item We have $\|TQ\|_\infty \le R_{\max} + \gamma \|Q\|_\infty$, and ${\left\| {TQ - Q} \right\|_\infty } \ge (1 - \gamma ){\left\| Q \right\|_\infty } - {R_{\max }}$.

\item $f$ is coercive.

\item Every sublevel set, $\mathcal L_c:=\{Q\in\mathbb R^{|\mathcal S||\mathcal A|}: f(Q)\le c\}$, is bounded for any $c\ge 0$.

\item $Q\in {\cal L}_c$ implies ${\left\| Q \right\|_\infty } \le \frac{{{R_{\max }} + \sqrt {2c} }}{{1 - \gamma }}$.

\item There exists $\mu>0$ such that for all $Q$,
\[
f(Q)-f^* \;\le\; \frac{1}{2\mu}\,\mathrm{dist}\!\left(0,\partial f(Q)\right)^2,
\]
where $f^* :={\min _{Q\in {\mathbb R}^{|{\cal S}\times {\cal A}|}}}f(Q)$ and $\mathrm{dist}(0,\partial f(Q)):=\inf\{\|g\|_2:\ g\in\partial f(Q)\}$.
\end{enumerate}
\end{lemma}
\begin{proof}
\begin{enumerate}
\item Since $f$ is a nonnegative sum of squares, $f$ is bounded below by $0$.

\item Moreover, by~\cref{thm:CBR:piecewise-quadratic}, $f$ is piecewise quadratic with finitely many pieces.
Therefore, it is locally Lipschitz by~\cref{thm:piecewise-quadratic-locally-Lipschitz}.

\item Since $f$ is locally Lipschitz by the second statement, the desired conclusion is directly obtained using~\cref{thm:locally-lipschitz-compact-subdiff}.

\item For any $(s,a)\in {\cal S}\times {\cal A}$, we can derive the following bounds:
\[|(TQ)(s,a)| \le |R(s,a)| + \gamma \sum\limits_{s' \in {\cal S}} {P(s'|s,a)} {\max _{a' \in {\cal A}}}|Q(s',a')| \le {R_{\max }} + \gamma {\left\| Q \right\|_\infty }.\]
This implies $\|TQ\|_\infty \le R_{\max} + \gamma \|Q\|_\infty$.
Moreover, we have
\[
\|TQ-Q\|_\infty
\ge \|Q\|_\infty - \|TQ\|_\infty
\ge (1-\gamma)\|Q\|_\infty - R_{\max},
\]
where the last inequality uses $\|TQ\|_\infty \le R_{\max} + \gamma \|Q\|_\infty$.

\item Using the bounds in the fourth statement, we have
\[f(Q) = \frac{1}{2}\left\| {TQ - Q} \right\|_2^2 \ge \frac{1}{2}\left\| {TQ - Q} \right\|_\infty ^2 \ge \frac{1}{2}{\left\{ {(1 - \gamma ){{\left\| Q \right\|}_\infty } - {R_{\max }}} \right\}^2}\]
Hence $f(Q)\to\infty$ as $\|Q\|_\infty\to\infty$, i.e., $f$ is coercive.

\item The coercivity in the fifth statement implies that the level sets are bounded.

\item For the last statement, using the bounds in the fourth statement, we have
\[f(Q) = \frac{1}{2}\left\| {TQ - Q} \right\|_2^2 \ge \frac{1}{2}\left\| {TQ - Q} \right\|_\infty ^2 \ge \frac{1}{2}{\left\{ {(1 - \gamma ){{\left\| Q \right\|}_\infty } - {R_{\max }}} \right\}^2}\]
Fix $c\ge 0$ and let $Q\in\mathcal L_c$, i.e., $f(Q)\le c$. Then, it follows that
\[\frac{1}{2}{\left\{ {(1 - \gamma ){{\left\| Q \right\|}_\infty } - {R_{\max }}} \right\}^2} \le c.\]
Taking square roots gives $(1 - \gamma ){\left\| Q \right\|_\infty } - {R_{\max }} \le \sqrt {2c} $. Therefore, one gets the desired result.

\item First of all, we can derive the following inequalities:
\begin{align*}
{\rm{dist}}(0,\partial f(Q)) =& \mathop {\min }\limits_{\beta  \in {\mathrm{conv}}\{ \Lambda (Q)\} } {\left\| {{{(I - \gamma P{\Pi ^\beta })}^\top}(TQ - Q)} \right\|_2}\\
\ge& \mathop {\min }\limits_{\beta  \in \mathrm{conv}\{ \Lambda (Q)\} } {\sigma _{\min }}({(I - \gamma P{\Pi ^\beta })^\top}){\left\| TQ - Q \right\|_2}.
\end{align*}
where $\sigma _{\min }$ denotes the minimum singular value. Then, we obtain
\[f(Q) = \frac{1}{2}\left\| {TQ - Q} \right\|_2^2\; \le \;\frac{1}{{{\sigma _{\min }}{{({{(I - \gamma P{\Pi ^\beta })}^\top})}^2}}}{\rm{dist}}{(0,\partial f(Q))^2}.\]
Since ${\min _{Q\in {\mathbb R}^{|{\cal S}\times {\cal A}|}}}f(Q) = f^* = 0$, this is exactly the desired bound.
\end{enumerate}

\end{proof}

\begin{theorem}\label{thm:subdifferential-1}
The Clarke subdifferential of $f$ is given by
\[\partial f(Q) = \{ {(\gamma P{\Pi ^\beta } - I)^\top}(TQ - Q):\beta  \in \mathrm{conv}(\Lambda (Q))\} \]
where $\Lambda (Q): = \left\{ {\pi  \in \Theta :\pi (s) \in  {\argmax_{a \in {\cal A}}}Q(s,a)} \right\}$ is the set of all possible greedy policies for $Q$.
\end{theorem}
\begin{proof}
Let us consider any point $\bar Q$, where ${\Lambda(\bar Q)}$ potentially includes multiple policies $\{\pi_1,\pi_2,\ldots, \pi_N\}$.
Moreover, let us consider the sequence $(Q_k^{\pi_i})_{k = 1}^\infty  \in {\mathbb R}^{|{\cal S}\times {\cal A}|}$ such that ${\Lambda(Q_k^{\pi_i})}$ is a singleton $\pi_i$, i.e., $f$ is differentiable.
Moreover, $\lim_{k \to \infty } {Q_k^{\pi_i}} = \bar Q$.
Let us consider the set
\[{S_\pi }: = \left\{ {Q \in {\mathbb R}^{|{\cal S} \times {\cal A}|}:\pi (s) \in  \argmax_{a \in {\cal A}} Q(s,a),\forall (s,a)\in {\cal S}\times {\cal A}} \right\}.\]
At any $Q$ such that $Q \in {S_{\pi} }$ for a single $\pi$, we have
\[f(Q) = \frac{1}{2}\left\| {TQ - Q} \right\|_2^2 = \frac{1}{2}\left\| {{T^\pi }Q - Q} \right\|_2^2\]
which is a quadratic function. Therefore, by direct calculations, we can obtain
\[{\nabla _Q}f(Q) = {(\gamma P{\Pi ^\pi } - I)^\top}({T^\pi }Q - Q) = {(\gamma P{\Pi ^\pi } - I)^\top}(TQ - Q)\]
where ${\pi}(s) = \argmax _{a \in {\cal A}}{Q}(s,a)$.
By~\cref{thm:subdifferential-6},~\cref{thm:CBR-LFA:piecewise-quadratic}, and~\cref{thm:piecewise-quadratic-locally-Lipschitz}, since
\begin{align*}
\mathop {\lim }\limits_{k \to \infty } {\nabla _Q}f(Q_k^{{\pi _i}}) =& \mathop {\lim }\limits_{k \to \infty } {(\gamma P{\Pi ^{{\pi _i}}} - I)^\top}({T^{{\pi _i}}}Q_k^{{\pi _i}} - Q_k^{{\pi _i}})\\
=& {(\gamma P{\Pi ^{{\pi _i}}} - I)^\top}({T^{{\pi _i}}}\bar Q - \bar Q)\\
=& {(\gamma P{\Pi ^{{\pi _i}}} - I)^\top}(T\bar Q - \bar Q),
\end{align*}
and the subdifferential is the convex hull of all gradients corresponding to $\{\pi_1,\pi_2,\ldots,\pi_N\}$. This completes the proof.
\end{proof}

\begin{theorem}\label{thm:CBR:stationary-1}
The stationary point $0 \in \partial f(Q)$ in the tabular case is unique and is given by $Q =Q^*$, which is the optimal Q-function.
\end{theorem}
\begin{proof}
Suppose that $\bar Q$ is a stationary point. Then, it implies ${(\gamma P{\Pi ^\beta } - I)^\top}(T\bar Q - \bar Q) = 0$ for some $\beta  \in \mathrm{conv}\{ {\Lambda _{\bar Q}}\}$.
Since $\gamma P{\Pi^\beta } - I$ is invertible by~\cref{thm:matrix_inversion2}, we have $T\bar Q = \bar Q$, which concludes the proof.
\end{proof}

\begin{theorem}\label{thm:CBR:gradient-descent-convergence}
Let $(Q_k)_{k\geq 0}$ be generated by
\begin{align*}
Q_{k+1}=Q_k-\alpha_k g_k,
\end{align*}
where $g_k \in \argmin_{g\in\partial f(Q_k)} \|g\|_2=\mathrm{Proj}_{\partial f(Q_k)}(0)$ and $\alpha_k>0$ is a step size generated by backtracking search with the Armijo rule~\citep{Boyd2004}.
Then, the sequence $(Q_k)_{k\geq 0}\subset \mathcal{L}_c$ with $c=f(\theta_0)$ converges to $Q^*$.
\end{theorem}
\begin{proof}
Using~\cref{thm:CBR:properties1}, since $f$ is locally Lipschitz continuous, bounded below, and its level set is bounded, the convergence can be proved using~\cref{thm:convergence-subgradient1}. In particular, the sequence $(Q_k)_{k\geq 0}\subset\mathcal{L}_c$ with $c=f(\theta_0)$ admits at least one limit point $\bar Q$. Every limit point $\bar Q$ satisfies $0\in \partial f(\bar Q)$, i.e., $\bar Q$ is a stationary point. However, since there exists a unique stationary point $Q^*$, the limit point is unique. Therefore, the entire sequence converges to $Q^*$.
\end{proof}

\subsection{Examples}\label{sec:CBR-tabular:example}

We provide a simple MDP example for the Clarke subdifferential of the CBR.
Let us consider a discounted MDP with $\mathcal S=\{1\}$, $\mathcal A=\{1,2\}$, and $\gamma=0.9$.
The transition is deterministic and self-looping for both actions: $P(1\mid 1,1)=P(1\mid 1,2)=1$. Moreover, rewards satisfy $R(1,1)=0, R(1,2)=1$.
Let
\[Q: = \left[ {\begin{array}{*{20}{c}}
{Q(1,1)}\\
{Q(1,2)}
\end{array}} \right] \in {\mathbb R}^2\]
The control Bellman operator is $(TQ)(1,a) = R(1,a) + \gamma {\max _{a' \in \{ 1,2\} }}Q(1,a'), a \in \{ 1,2\} $, i.e.,
\[TQ = \left[ {\begin{array}{*{20}{c}}
{R(1,1) + \gamma {{\max }_{a' \in \{ 1,2\} }}Q(1,a')}\\
{R(1,2) + \gamma {{\max }_{a' \in \{ 1,2\} }}Q(1,a')}
\end{array}} \right]\]
and the CBR objective is $f(Q):=\frac12\|Q-TQ\|_2^2$. The surface of $f(Q)$ is illustrated in~\cref{fig:ex-fig7}.
\begin{figure}[ht!]
\centering\includegraphics[width=0.7\textwidth, keepaspectratio]{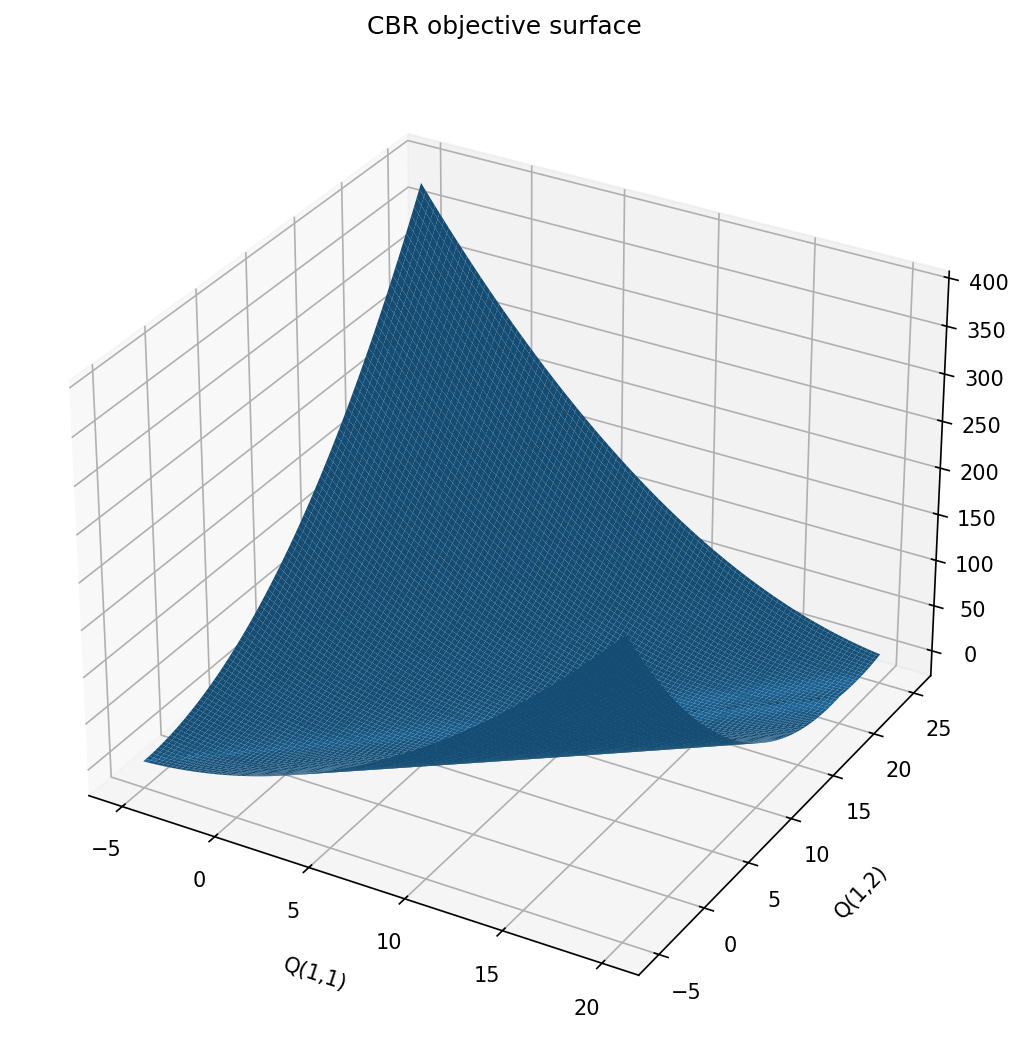}
\caption{The surface of $f(Q)$}\label{fig:ex-fig7}
\end{figure}

Non-differentiability occurs precisely on the tie set $\{ Q \in {\mathbb R}^{|{\cal S} \times {\cal A}|}:Q(1,1) = Q(1,2)\}$ because of the max operator.
For a (possibly stochastic) policy $\beta$, we can represent
\[{\Pi ^{\beta}} = e_s^\top \otimes \beta{( \cdot |1)^\top} = \left[ {\begin{array}{*{20}{c}}
\beta(1|1) &{1 - \beta(1|1) }
\end{array}} \right],\qquad \beta(1|1)  \in [0,1],\]
where $e_s\in {\mathbb R}^{|{\cal S}|}$ is the standard basis vector whose $s$-th element is one and whose other elements are zero.
Moreover, since both state-action pairs transition back to $1$, the tabular transition
matrix from state-action to next-state is
\[
P=\begin{bmatrix}1\\1\end{bmatrix}\in {\mathbb R}^{2\times 1}.
\]
Therefore, $P{\Pi ^\beta}$ is given as
\[P{\Pi ^\beta} = \left[ {\begin{array}{*{20}{c}}
1\\
1
\end{array}} \right]\left[ {\begin{array}{*{20}{c}}
\beta(1|1) &{1 - \beta(1|1) }
\end{array}} \right] = \left[ {\begin{array}{*{20}{c}}
\beta(1|1) &{1 - \beta(1|1) }\\
\beta(1|1) &{1 - \beta(1|1) }
\end{array}} \right] \in {\mathbb R}^{2 \times 2},\]
which is row-stochastic for every $\beta(1|1) \in[0,1]$, and
\[
\gamma P\Pi^\beta-I
=
\begin{bmatrix}
\gamma \beta(1|1)-1 & \gamma (1-\beta(1|1))\\
\gamma \beta(1|1) & \gamma (1-\beta(1|1))-1
\end{bmatrix},
\]
The Clarke subdifferential of the CBR objective is
\[\partial f(Q) = \left\{ {{{(\gamma P{\Pi ^{\beta}} - I)}^\top}(TQ - Q)\;:\;\beta(1|1)\in [0,1]} \right\},\]
where
\begin{align*}
{(\gamma P{\Pi ^\beta } - I)^\top}(TQ - Q)=& {\left[ {\begin{array}{*{20}{c}}
{\gamma \beta (1|1) - 1}&{\gamma (1 - \beta (1|1))}\\
{\gamma \beta (1|1)}&{\gamma (1 - \beta (1|1)) - 1}
\end{array}} \right]^\top}\\
&\times \left[ {\begin{array}{*{20}{c}}
{R(1,1) + \gamma {{\max }_{a' \in \{ 1,2\} }}Q(1,a') - Q(1,1)}\\
{R(1,2) + \gamma {{\max }_{a' \in \{ 1,2\} }}Q(1,a') - Q(1,2)}
\end{array}} \right].
\end{align*}

\paragraph{Case 1: $Q(1,2)>Q(1,1)$.}
The greedy policy selects $2$, and $\beta (1|1) = 0,\beta (2|1) = 1$.
In this case,
\begin{align*}
{(\gamma P{\Pi ^\beta } - I)^\top}(TQ - Q) =& {\left[ {\begin{array}{*{20}{c}}
{ - 1}&\gamma \\
0&{\gamma  - 1}
\end{array}} \right]^\top}\left[ {\begin{array}{*{20}{c}}
{\gamma Q(1,2) - Q(1,1)}\\
{1 + \gamma Q(1,2) - Q(1,2)}
\end{array}} \right]\\
=& \left[ {\begin{array}{*{20}{c}}
1&{ - \gamma }\\
{ - \gamma }&{{\gamma ^2} - {{(\gamma  - 1)}^2}}
\end{array}} \right]Q + \left[ {\begin{array}{*{20}{c}}
0\\
{\gamma  - 1}
\end{array}} \right]\\
=:&g_1.
\end{align*}
Therefore, the subdifferential is the singleton, ${\left. {{\partial}f(Q)} \right|_{Q(1,2) > Q(1,1)}} = \left\{ {{g_1}} \right\}$.

\paragraph{Case 2: $Q(1,1)>Q(1,2)$.}
The greedy policy selects $1$, and hence, $\beta (1|1) = 1,\beta (2|1) = 0$.
In this case, the subdifferential is given as
\begin{align*}
{(\gamma P{\Pi ^\beta } - I)^\top}(TQ - Q) =& {\left[ {\begin{array}{*{20}{c}}
{\gamma  - 1}&0\\
\gamma &{ - 1}
\end{array}} \right]^\top}\left[ {\begin{array}{*{20}{c}}
{\gamma Q(1,1) - Q(1,1)}\\
{1 + \gamma Q(1,1) - Q(1,2)}
\end{array}} \right]\\
=& \left[ {\begin{array}{*{20}{c}}
{\gamma  - 1}&\gamma \\
0&{ - 1}
\end{array}} \right]\left\{ {\left[ {\begin{array}{*{20}{c}}
{\gamma  - 1}&0\\
\gamma &{ - 1}
\end{array}} \right]Q + \left[ {\begin{array}{*{20}{c}}
0\\
1
\end{array}} \right]} \right\}\\
=& \left[ {\begin{array}{*{20}{c}}
{{{(\gamma  - 1)}^2} + {\gamma ^2}}&{ - \gamma }\\
{ - \gamma }&1
\end{array}} \right]Q + \left[ {\begin{array}{*{20}{c}}
\gamma \\
{ - 1}
\end{array}} \right]\\
=:&g_2.
\end{align*}
Therefore, the subdifferential is the singleton, ${\left. {{\partial}f(Q)} \right|_{Q(1,1) > Q(1,2)}} = \left\{ {{g_2}} \right\}$.

\paragraph{Case 3: $Q(1,1)=Q(1,2)$.}
Both actions are greedy, and the set of greedy policies given below
contains all mixtures
\[\beta( \cdot |1) = \left[ {\begin{array}{*{20}{c}}
{\beta(1|1)}\\
{\beta(2|1)}
\end{array}} \right] = \left[ {\begin{array}{*{20}{c}}
\beta(1|1) \\
{1 - \beta(1|1) }
\end{array}} \right],\qquad \beta(1|1)  \in [0,1].\]

In this case, an element of the subdifferential is given as
\begin{align*}
{(\gamma P{\Pi ^\beta } - I)^\top}(TQ - Q)
&= {\left[ {\begin{array}{*{20}{c}}
{\gamma \beta (1|1) - 1}&{\gamma (1 - \beta (1|1))}\\
{\gamma \beta (1|1)}&{\gamma (1 - \beta (1|1)) - 1}
\end{array}} \right]^\top}\\
&\quad\times\left[ {\begin{array}{*{20}{c}}
{\gamma Q(1,1) - Q(1,1)}\\
{1 + \gamma Q(1,1) - Q(1,2)}
\end{array}} \right]\\
&= \beta (1|1){g_1} + [1 - \beta (1|1)]{g_2}.
\end{align*}
for all $\beta(1|1)\in [0,1]$. Therefore,
\[
{\left. \partial f(Q) \right|_{Q(1,1) = Q(1,2)}}
= \left\{ {\beta (1|1){g_1} + [1 - \beta (1|1)]{g_2}:\beta (1|1) \in [0,1]} \right\}.
\].

Next, let us consider the tie point
\[
Q=\begin{bmatrix}0\\0\end{bmatrix}.
\]
Then, ${\max _{a \in \{ 1,2\} }}Q(1,a) = 0$ and thus
\[TQ = \left[ {\begin{array}{*{20}{c}}
0\\
1
\end{array}} \right],\qquad TQ - Q = \left[ {\begin{array}{*{20}{c}}
0\\
1
\end{array}} \right].\]
Using~\cref{thm:subdifferential-1}, the subdifferential of the CBR objective can be written as
\[\partial f(Q) = \left\{ {{{(\gamma P{\Pi ^{\beta}} - I)}^\top}(TQ - Q)\;:\;\beta(1|1)\in [0,1]} \right\},\]
where
\begin{align*}
{(\gamma P{\Pi ^{\beta}} - I)^\top}(TQ - Q)
=& {\left[ {\begin{array}{*{20}{c}}
{\gamma \beta(1|1)  - 1}&{\gamma (1 - \beta(1|1) )}\\
{\gamma \beta(1|1) }&{\gamma (1 - \beta(1|1) ) - 1}
\end{array}} \right]^\top}
\left[ {\begin{array}{*{20}{c}}
0\\
1
\end{array}} \right]\\
=& \left[ {\begin{array}{*{20}{c}}
{\gamma \beta(1|1) }\\
{ \gamma-1 - \gamma\beta(1|1) }
\end{array}} \right].
\end{align*}
Hence, the subdifferential at $Q=(0,0)$ is the following line segment:
\begin{align*}
\partial f(0)
&= \left\{ {\left[ {\begin{array}{*{20}{c}}
{\gamma\beta(1|1) }\\
{ \gamma-1 - \gamma\beta(1|1) }
\end{array}} \right]:\;\beta(1|1)  \in [0,1]} \right\}\\
&= {\mathop{\rm conv}\nolimits} \left\{ {\left[ {\begin{array}{*{20}{c}}
0\\
{ - 0.1}
\end{array}} \right],\left[ {\begin{array}{*{20}{c}}
{0.9}\\
{ - 1}
\end{array}} \right]} \right\}.
\end{align*}
In particular, $0 \notin \partial f(0)$, and hence, $0$ is not a stationary point.

For any $\beta\in[0,1]$, $I-\gamma P\Pi^{\beta}$ is invertible by~\cref{thm:matrix_inversion2}.
Therefore, $0 \in {\partial}f(\bar Q)$ for some $\bar Q$ implies that there exists a $\beta(1|1) \in [0,1]$ such that
${(\gamma P{\Pi ^\beta } - I)^\top}(T\bar Q - \bar Q) = 0$. This implies that $T\bar Q = \bar Q$.
Therefore, any stationary point must satisfy the Bellman optimality equation $TQ=Q$.

Solving $TQ=Q$, which is equivalent to
\[Q(1,1) = \gamma \max \{ Q(1,1),Q(1,2)\} ,\qquad Q(1,2) = 1 + \gamma \max \{ Q(1,1),Q(1,2)\}, \]
yields the unique fixed point:
\[{Q^*} = \left[ {\begin{array}{*{20}{c}}
{\frac{\gamma }{{1 - \gamma }}}\\
{\frac{1}{{1 - \gamma }}}
\end{array}} \right] = \left[ {\begin{array}{*{20}{c}}
9\\
{10}
\end{array}} \right].\]

We can apply a subgradient descent method to the above objective, and the results are shown in~\cref{fig:ex-fig8} and~\cref{fig:ex-fig9}. Convergence is guaranteed by~\cref{thm:CBR:gradient-descent-convergence}, and the plots confirm that the iterates converge to $Q^*$.

\begin{figure}[ht!]
\centering\includegraphics[width=0.7\textwidth, keepaspectratio]{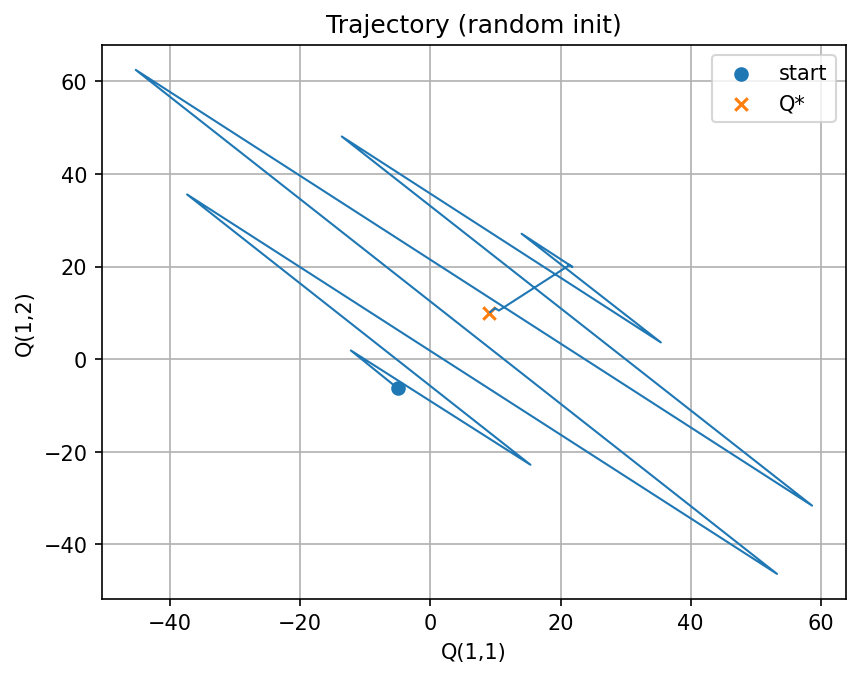}
\caption{Convergence trajectory of $Q_k$ with the subgradient descent method.}\label{fig:ex-fig8}
\end{figure}
\begin{figure}[ht!]
\centering\includegraphics[width=0.7\textwidth, keepaspectratio]{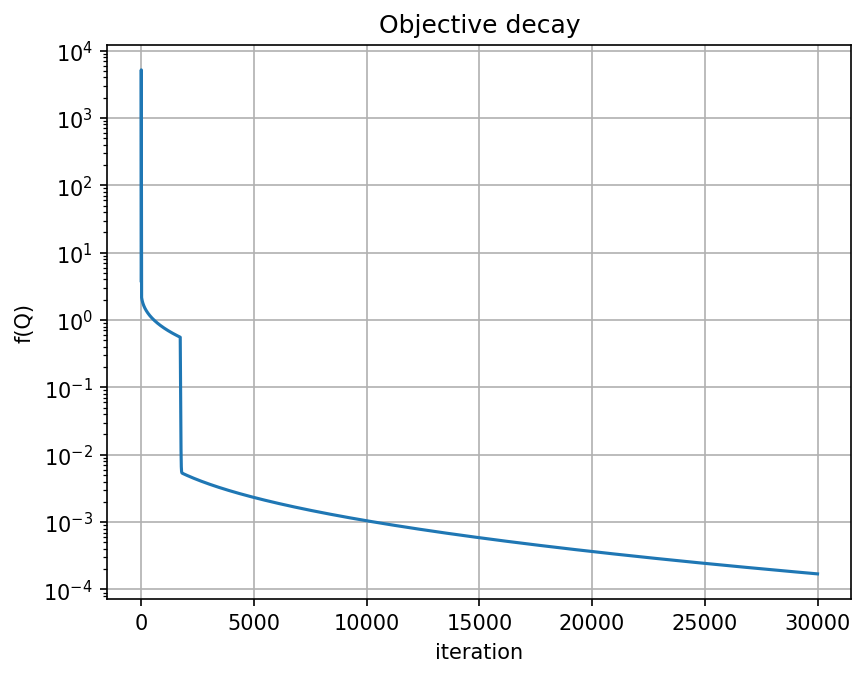}
\caption{CBR objective function value $f(Q_k)$ with $Q_k$ generated by the subgradient descent method.}\label{fig:ex-fig9}
\end{figure}

\newpage
\section{CBR with linear function approximation}

Let $\Theta $ denote the set of deterministic policies $\pi:\mathcal S\to\mathcal A$, so $|\Theta|=|\mathcal A|^{|\mathcal S|}<\infty$.
Let us consider the CBR objective with LFA
\[f(\theta ): = \frac{1}{2}\left\| {T\Phi \theta  - \Phi \theta } \right\|_2^2,\]
Moreover, let us define the set
\[{S_\pi }: = \left\{ {\theta  \in {\mathbb R}^m:\pi (s) \in  {\argmax_{a \in {\cal A}}}{Q_\theta }(s,a)},\forall (s,a)\in {\cal S}\times {\cal A} \right\}\]
for each $\pi \in \Theta$, where $\argmax$ is interpreted as a set-valued map, i.e., $\pi(s)$ is one of the maximizers; ties are allowed.

\begin{proposition}\label{thm:CBR-LFA:piecewise-quadratic}
$f$ is piecewise quadratic (so piecewise smooth) and continuous.
\end{proposition}
\begin{proof}
First of all, $f$ can be written as
\begin{align*}
f(\theta ) =& \frac{1}{2}{(R + \gamma P{\Pi ^{{\pi _{{Q_\theta }}}}}{Q_\theta } - {Q_\theta })^\top}(R + \gamma P{\Pi ^{{\pi _{{Q_\theta }}}}}{Q_\theta } - {Q_\theta })\\
=& \frac{1}{2}{\theta ^\top}{\Phi ^\top}({\gamma ^2}{(P{\Pi ^{{\pi _{{Q_\theta }}}}})^\top}P{\Pi ^{{\pi _{{Q_\theta }}}}} - \gamma P{\Pi ^{{\pi _{{Q_\theta }}}}} - \gamma {(P{\Pi ^{{\pi _{{Q_\theta }}}}})^\top} + I)\Phi \theta \\
& + (\gamma {R^\top}P{\Pi ^{{\pi _{{Q_\theta }}}}} - {R^\top})\Phi \theta  + \frac{1}{2}{R^\top}R,
\end{align*}
where ${\Pi ^{\pi_{Q_\theta}}}$ depends on $\theta$ and is constant within the set $S_\pi$ for each $\pi \in \Theta$. Therefore, $f$ is quadratic on each set $S_\pi$. Since the family $\{S_\pi\}_{\pi\in\Theta}$ is finite and each $S_\pi$ is defined by finitely many linear inequalities, the arrangement of the corresponding hyperplanes induces a finite polyhedral partition refining this cover. On each cell of the refined partition, the active-policy set is constant, and hence $f$ coincides with one of the quadratic expressions above. Therefore, $f$ is finite piecewise quadratic in the sense of~\cref{def:piecewise-quadratic}. Moreover, it is continuous because for any $\theta \in {S_\pi }$ on the boundary of $S_\pi$, it satisfies ${\pi _1}(s),{\pi _2}(s), \ldots ,{\pi _N}(s) \in  {\argmax _{a \in {\cal A}}}Q_\theta (s,a)$, and the corresponding quadratic functions, $f_1,f_2,\ldots, f_N$, are identical at the point, i.e., ${f_1}(\theta) = {f_2}(\theta) =  \cdots  = {f_N}(\theta)$.
\end{proof}

\begin{proposition}\label{thm:CBR-LFA:partition-set}
For each $\pi \in \Theta$, the set ${S_\pi }$ is an intersection of half-spaces and is a convex cone.
\end{proposition}
\begin{proof}
For each $s\in {\cal S}$, the condition $\pi(s)\in \argmax_{a\in {\cal A}} Q_\theta(s,a)$ is equivalent to the family of linear inequalities $Q_\theta(s,\pi (s))\; \ge \;Q_\theta(s,a), \forall (s,a) \in {\cal S}\times {\cal A}$.
Hence, $S_\pi $ can be written as
\[{S_\pi } = \bigcap\limits_{s \in {\cal S}} \; \bigcap\limits_{a \in {\cal A}} {\left\{ {\theta  \in {\mathbb R}^m:\;{Q_\theta }(s,\pi (s)) - {Q_\theta }(s,a) \ge 0} \right\}}  = \left\{ {\theta  \in {\mathbb R}^m:\;{L \Pi ^\pi }\Phi \theta  - \Phi \theta  \ge 0} \right\},\]
where $L\in {\mathbb R}^{|{\cal S}\times {\cal A}|\times |{\cal S}|}$ is a lifting matrix defined as $L: = I_{|{\cal S}|} \otimes {{\bf{1}}_{|{\cal A}|}}$, where $I_{|{\cal S}|}$ is the identity matrix of dimension $|{\cal S}|$, $\otimes$ denotes the Kronecker product, ${\bf{1}}_{|{\cal A}|}$ is the all-ones vector of dimension $|{\cal A}|$. Each set in the intersection is a closed half-space containing the origin. Therefore, it is a polyhedral set.

To prove the second statement, let us fix any $\pi \in \Theta$ and consider any $\theta \in {S_\pi }$. Then, for any $\alpha > 0$, we have $\alpha \theta\in {S_\pi }$ because
\[ {\argmax _{a \in {\cal A}}}Q_\theta(s,a) =  {\argmax _{a \in {\cal A}}}Q_{\alpha\theta}(s,a)\]
and $Q_{\alpha\theta}=\alpha Q_\theta$. Therefore, ${S_\pi }$ is a cone.
Moreover, let $\theta_1,\theta_2\in S_\pi$ and $\alpha\in[0,1]$. For all $s\in {\cal S}$ and $a\in {\cal A}$,
\[{Q_{{\theta _1}}}(s,\pi (s)) \ge {Q_{{\theta _1}}}(s,a),\qquad {Q_{{\theta _2}}}(s,\pi (s)) \ge {Q_{{\theta _2}}}(s,a).\]
Multiplying the first inequality by $\alpha \in [0,1]$ and the second by $1-\alpha$, and adding, yields
\[\alpha {Q_{{\theta _1}}}(s,\pi (s)) + (1 - \alpha ){Q_{{\theta _2}}}(s,\pi (s)) \ge \alpha {Q_{{\theta _1}}}(s,a) + (1 - \alpha ){Q_{{\theta _2}}}(s,a),\]
Because $Q_{\alpha\theta_1+(1-\alpha)\theta_2}=\alpha Q_{\theta_1}+(1-\alpha)Q_{\theta_2}$ by linearity of the approximator, this shows that $\alpha\theta_1+(1-\alpha)\theta_2\in S_\pi$. Hence $S_\pi$ is convex.
Combining the two parts, we conclude that $S_\pi$ is a convex cone.
\end{proof}

\begin{proposition}\label{thm:CBR-LFA:quadratic-bound}
$f$ is bounded by strongly convex quadratic functions as ${q_1}(\theta ) \le f(\theta ) \le {q_2}(\theta )$, where
\[{q_1}(\theta ): = \frac{{{{(1 - \gamma )}^2}}}{{2|{\cal S} \times {\cal A}|}}\left\| {{Q_\theta } - Q^*} \right\|_2^2,\quad {q_2}(\theta ): = \frac{{{{(1 + \gamma )}^2}|{\cal S} \times {\cal A}|}}{2}\left\| {{Q_\theta } - {Q^*}} \right\|_2^2.\]
\end{proposition}
\begin{proof}
Using~\cref{thm:fundamental4}, we have
\begin{align*}
\frac{1}{\sqrt {|{\cal S} \times {\cal A}|} }{\left\| {{Q_\theta } - {Q^*}} \right\|_2}
&\le {\left\| {{Q_\theta } - {Q^*}} \right\|_\infty }
\le \frac{1}{1 - \gamma}{\left\| {T{Q_\theta } - {Q_\theta }} \right\|_\infty }\\
&\le \frac{1}{1 - \gamma}{\left\| {T{Q_\theta } - {Q_\theta }} \right\|_2}.
\end{align*}
Algebraic manipulations lead to $\frac{{{{(1 - \gamma )}^2}}}{{2|{\cal S} \times {\cal A}|}}\left\| {{Q_\theta } - {Q^*}} \right\|_2^2 \le f(\theta )$.
Similarly, using~\cref{thm:fundamental5}, one gets
\begin{align*}
{\left\| {T{Q_\theta } - {Q_\theta }} \right\|_2} \le& \sqrt {|{\cal S} \times {\cal A}|} {\left\| {T{Q_\theta } - {Q_\theta }} \right\|_\infty }\\
\le& (1 + \gamma )\sqrt {|{\cal S} \times {\cal A}|} {\left\| {{Q_\theta } - {Q^*}} \right\|_\infty }\\
\le& (1 + \gamma )\sqrt {|{\cal S} \times {\cal A}|} {\left\| {{Q_\theta } - {Q^*}} \right\|_2},
\end{align*}
which leads to $f(\theta ) \le \frac{{{{(1 + \gamma )}^2}|{\cal S} \times {\cal A}|}}{2}\left\| {{Q_\theta } - {Q^*}} \right\|_2^2$. This completes the proof.
\end{proof}

\begin{proposition}\label{thm:CBR-LFA:quadratic-bound2}
Let us define
\[\theta_1^*: =  \argmin _{\theta  \in {\mathbb R}^m} {\left\| Q_\theta - Q^* \right\|_2},\quad \theta _2^*: =  {\argmin _{\theta  \in {\mathbb R}^m}}f(\theta ).\]
Let $e_\Phi:=\left\|\Gamma_{\Phi|\Phi}Q^*-Q^*\right\|_2$ and $n:=|{\cal S}\times{\cal A}|$. Then, we have
\begin{align*}
\frac{(1-\gamma)^2}{2n}e_\Phi^2
&\le f(\theta _1^*)
\le \frac{(1+\gamma)^2 n}{2}e_\Phi^2,\\
\frac{(1-\gamma)^2}{2n}e_\Phi^2
&\le f(\theta _2^*)
\le \frac{(1+\gamma)^2 n}{2}e_\Phi^2,\\
0 \le f(\theta _1^*) - f(\theta _2^*)
&\le \left\{\frac{(1+\gamma)^2 n}{2}-\frac{(1-\gamma)^2}{2n}\right\}e_\Phi^2,\\
\left\|Q_{\theta _1^*} - Q_{\theta _2^*}\right\|_2
&\le \left\{1+\frac{(1+\gamma)n}{1-\gamma}\right\}e_\Phi.
\end{align*}
\end{proposition}
\begin{proof}
By plugging $\theta_1^*$ and $\theta_2^*$ into the quadratic bounds in~\cref{thm:CBR-LFA:quadratic-bound}, one has
\begin{align*}
\frac{(1-\gamma)^2}{2n}e_\Phi^2
&= q_1(\theta _1^*)
\le f(\theta _1^*)
\le q_2(\theta _1^*)
= \frac{(1+\gamma)^2n}{2}e_\Phi^2,\\
\frac{(1-\gamma)^2}{2n}\left\|Q_{\theta _2^*}-Q^*\right\|_2^2
&= q_1(\theta _2^*)
\le f(\theta _2^*)
\le q_2(\theta _2^*)\\
&= \frac{(1+\gamma)^2n}{2}\left\|Q_{\theta _2^*}-Q^*\right\|_2^2,
\end{align*}
respectively. The first inequality above proves the first statement.
Moreover, since $f(\theta _2^*) \le f(\theta _1^*)$ and $\left\| {{\Gamma _{\Phi |\Phi }}{Q^*} - {Q^*}} \right\|_2^2 = \left\| {{Q_{\theta _1^*}} - {Q^*}} \right\|_2^2 \le \left\| {{Q_{\theta _2^*}} - {Q^*}} \right\|_2^2$, it follows that
\begin{align}
\frac{{{{(1 - \gamma )}^2}}}{{2|{\cal S} \times {\cal A}|}}\left\| {{\Gamma _{\Phi |\Phi }}Q^* - Q^*} \right\|_2^2 \le& \frac{{{{(1 - \gamma )}^2}}}{{2|{\cal S} \times {\cal A}|}}\left\| {Q_{\theta _2^*} - Q^*} \right\|_2^2\nonumber\\
\le& f(\theta _2^*)\nonumber\\
\le& f(\theta _1^*)\nonumber\\
\le& \frac{{{{(1 + \gamma )}^2}|{\cal S} \times {\cal A}|}}{2}\left\| {{\Gamma _{\Phi |\Phi }}Q^* - Q^*} \right\|_2^2,\label{eq:7}
\end{align}
which leads to
\[\frac{{{{(1 - \gamma )}^2}}}{{2|{\cal S} \times {\cal A}|}}\left\| {{\Gamma _{\Phi |\Phi }}Q^* - Q^*} \right\|_2^2 \le f(\theta _2^*) \le \frac{{{{(1 + \gamma )}^2}|{\cal S} \times {\cal A}|}}{2}\left\| {{\Gamma _{\Phi |\Phi }}Q^* - Q^*} \right\|_2^2,\]
and it proves the second statement. From~\eqref{eq:7} again, we obtain the following two inequalities:
\begin{align*}
\frac{{{{(1 - \gamma )}^2}}}{{2|{\cal S} \times {\cal A}|}}\left\| {{\Gamma _{\Phi |\Phi }}Q^* - Q^*} \right\|_2^2 \le& f(\theta _2^*) \le f(\theta _1^*),\\
f(\theta _1^*) - \frac{{{{(1 + \gamma )}^2}|{\cal S} \times {\cal A}|}}{2}\left\| {{\Gamma _{\Phi |\Phi }}Q^* - Q^*} \right\|_2^2 \le& 0.
\end{align*}
Adding the two inequalities yields the third statement.
For the last statement, we derive
\begin{align}
{\left\| {{Q_{\theta _1^*}} - {Q_{\theta _2^*}}} \right\|_2 } =& {\left\| {{Q_{\theta _1^*}} - {Q^*} + {Q^*} - {Q_{\theta _2^*}}} \right\|_2 }\nonumber\\
\le& {\left\| {{Q_{\theta _1^*}} - {Q^*}} \right\|_2} + {\left\| {{Q_{\theta _2^*}} - {Q^*}} \right\|_2}\nonumber\\
\le& {\left\| {{\Gamma _{\Phi |\Phi }}{Q^*} - {Q^*}} \right\|_2} + {\left\| {{Q_{\theta _2^*}} - {Q^*}} \right\|_2}.\label{eq:8}
\end{align}
To bound the second term in~\eqref{eq:8},~\eqref{eq:7} is used to get
\begin{align}
{\left\| {{Q_{\theta _2^*}} - {Q^*}} \right\|_2} \le \frac{(1+\gamma)|{\cal S}\times{\cal A}|}{1-\gamma}{\left\| {{\Gamma _{\Phi |\Phi }}{Q^*} - {Q^*}} \right\|_2}.\label{eq:9}
\end{align}
Now, combining~\eqref{eq:8} and~\eqref{eq:9} leads to the last statement. This completes the proof.
\end{proof}

For the LFA case, define the set of policies that are active from full-dimensional differentiability regions by
\[
\Lambda_\Phi(\theta):=\left\{\pi\in\Theta:\theta\in \overline{\operatorname{int}(S_\pi)}\right\}.
\]
\begin{theorem}\label{thm:subdifferential-2}
The subdifferential of $f$ is given by
\begin{align*}
{\partial}f(\theta ) =& \{ {\Phi ^\top}{(\gamma P{\Pi ^\beta } - I)^\top}(T{Q_\theta } - {Q_\theta }): \beta  \in {\rm conv}(\Lambda_\Phi(\theta))\} ,
\end{align*}
where only the greedy policies that are reachable from differentiability regions in the parameter space are included.
\end{theorem}
\begin{proof}
Let us consider any point $\bar \theta$. By the definition of $\Lambda_\Phi(\bar\theta)$, each $\pi_i\in\Lambda_\Phi(\bar\theta)$ admits a sequence $(\theta_k^{\pi_i})_{k = 1}^\infty\in {\mathbb R}^m$ such that $\theta_k^{\pi_i}\to\bar\theta$, $f$ is differentiable at $\theta_k^{\pi_i}$, and the unique active policy at $\theta_k^{\pi_i}$ is $\pi_i$.
At any $\theta\in {\mathbb R}^m$ such that $\theta \in S_\pi$ for some deterministic policy $\pi$, where
\[{S_\pi }: = \left\{ \theta  \in {\mathbb R}^m:\pi (s) \in {\argmax _{a \in {\cal A}}}{Q_\theta }(s,a),\forall (s,a)\in {\cal S}\times {\cal A} \right\}\]
we have
\[f(\theta ) = \frac{1}{2}\left\| {T{Q_\theta } - {Q_\theta }} \right\|_2^2 = \frac{1}{2}\left\| {{T^\pi }{Q_\theta } - {Q_\theta }} \right\|_2^2,\]
which is a quadratic function. Therefore, by direct calculations, we can obtain
\[{\nabla _\theta }f(\theta ) = {\Phi ^\top}{(\gamma P{\Pi ^{\pi}} - I)^\top}({T^{\pi}}{Q_\theta } - {Q_\theta }) = {\Phi ^\top}{(\gamma P{\Pi ^{\pi}} - I)^\top}(T{Q_\theta } - {Q_\theta }),\]
where $\pi(s) =  {\argmax _{a \in {\cal A}}}{Q_\theta }(s,a)$.
By~\cref{thm:subdifferential-6},~\cref{thm:CBR-LFA:piecewise-quadratic}, and~\cref{thm:piecewise-quadratic-locally-Lipschitz}, since
\[\mathop {\lim }\limits_{k \to \infty } {\nabla _\theta }f(\theta _k^{\pi _i}) = \mathop {\lim }\limits_{k \to \infty } {\Phi ^\top}{(\gamma P{\Pi ^{{\pi _i}}} - I)^\top}({T^{{\pi _i}}}{Q_{\theta _k^{{\pi _i}}}} - {Q_{\theta _k^{{\pi _i}}}}) = {\Phi ^\top}{(\gamma P{\Pi ^{{\pi _i}}} - I)^\top}({T^{{\pi _i}}}{Q_{\bar \theta }} - {Q_{\bar \theta }})\]
and the subdifferential is the convex hull of all such limiting gradients corresponding to policies in $\Lambda_\Phi(\bar\theta)$.
This completes the proof.
\end{proof}

\begin{theorem}\label{thm:stationary-2}
Under Assumption~\ref{ass:oblique-regularity}, the stationary point $0 \in \partial f(\bar \theta)$ satisfies the OP-CBE, ${Q_{\bar \theta }} = {\Gamma _{\Phi |\Psi_{\bar\beta} }}T{Q_{\bar \theta }}$, where $\Psi_{\bar\beta} = (\gamma P{\Pi ^{\bar\beta} } - I)\Phi$ and ${\bar\beta} \in {\rm conv}(\Lambda_\Phi(\bar\theta))$ is given in~\eqref{eq:beta-bar}.
\end{theorem}
\begin{proof}
By~\cref{thm:subdifferential-2}, the stationary point $\bar \theta$ with $0 \in \partial f(\bar \theta )$ satisfies ${\Phi ^\top}{(\gamma P{\Pi ^{\bar \beta} } - I)^\top}(T{Q_{\bar \theta }} - {Q_{\bar \theta }}) = 0$ for some ${\bar \beta} \in {\rm conv}(\Lambda_\Phi(\bar\theta))$, which can be written as
\begin{align*}
{\Phi ^\top}{(\gamma P{\Pi ^{\bar \beta} } - I)^\top}\Phi \bar \theta  = {\Phi ^\top}{(\gamma P{\Pi ^{\bar \beta} } - I)^\top}T{Q_{\bar \theta }}.
\end{align*}
Since ${\Phi ^\top}{(\gamma P{\Pi ^{\bar \beta} } - I)^\top}\Phi$ is invertible by Assumption~\ref{ass:oblique-regularity}, the above equation can be equivalently written as
\begin{align*}
{Q_{\bar \theta }} = \Phi {\left[ {{\Phi ^\top}{{(\gamma P{\Pi ^{\bar \beta} } - I)}^\top}\Phi } \right]^{ - 1}}{\Phi ^\top}{(\gamma P{\Pi ^{\bar \beta} } - I)^\top}T{Q_{\bar \theta }},
\end{align*}
which is ${Q_{\bar \theta }} = {\Gamma _{\Phi |\Psi_{\bar \beta} }}T{Q_{\bar \theta }}$.
This completes the proof.
\end{proof}

Many properties of the CBR objective function continue to hold even when linear function approximation is employed.
The following summarizes these properties.
\begin{lemma}\label{thm:CBR-LFA:properties}
\begin{enumerate}
\item $f$ is bounded below, locally Lipschitz continuous, and is coercive.

\item The Clarke subdifferential of $f$ given by $\partial f(\theta) = \{ {\Phi^\top (\gamma P{\Pi ^\beta } - I)^\top}(TQ_\theta - Q_\theta):\beta  \in \mathrm{conv}(\Lambda_\Phi(\theta))\}$ is nonempty, convex, and bounded.

\item We have $\|TQ_\theta\|_\infty \le R_{\max} + \gamma \|Q_\theta\|_\infty$, and ${\left\| {TQ_\theta - Q_\theta} \right\|_\infty } \ge (1 - \gamma ){\left\| Q_\theta \right\|_\infty } - {R_{\max }}$.

\item Every sublevel set ${\cal L}_c: = \{ \theta  \in {\mathbb R}^m:f(\theta ) \le c\}$ is bounded for any $c\ge 0$.

\item $\theta\in {\cal L}_c$ implies
\[{\left\| Q_\theta \right\|_\infty } \le \frac{{{R_{\max }} + \sqrt {2c} }}{{1 - \gamma }}.\]
\end{enumerate}
\end{lemma}

\begin{proof}
The lower boundedness and the estimate in item~3 follow exactly as in the tabular case, with $Q$ replaced by $Q_\theta=\Phi\theta$.
Moreover, by~\cref{thm:CBR-LFA:piecewise-quadratic}, $f$ is piecewise quadratic with finitely many pieces, and hence it is locally Lipschitz by~\cref{thm:piecewise-quadratic-locally-Lipschitz}.
If $\theta\in {\cal L}_c$, then item~3 yields
\[
\|Q_\theta\|_\infty \le \frac{R_{\max}+\sqrt{2c}}{1-\gamma}.
\]
Therefore,
\[
\|Q_\theta\|_2 \le \sqrt{|{\cal S}\times {\cal A}|}\,\|Q_\theta\|_\infty \le \sqrt{|{\cal S}\times {\cal A}|}\,\frac{R_{\max}+\sqrt{2c}}{1-\gamma}.
\]
Since $\Phi$ has full column rank,
\[
\sigma_{\min}(\Phi)\,\|\theta\|_2 \le \|\Phi\theta\|_2 = \|Q_\theta\|_2,
\]
and thus
\[
\|\theta\|_2 \le \frac{\sqrt{|{\cal S}\times {\cal A}|}}{\sigma_{\min}(\Phi)}\,\frac{R_{\max}+\sqrt{2c}}{1-\gamma}.
\]
This proves item~4. Item~5 is exactly the displayed bound above.
Finally, the same bound in item~3 together with the full column rank of $\Phi$ implies coercivity: if $\|\theta\|_2\to\infty$, then $\|Q_\theta\|_2=\|\Phi\theta\|_2\to\infty$, hence $\|Q_\theta\|_\infty\to\infty$, and therefore $f(\theta)\to\infty$.
The subdifferential claim in item~2 follows from local Lipschitz continuity and~\cref{thm:locally-lipschitz-compact-subdiff}.
\end{proof}

The PL-type inequality that holds in the tabular case does not, in general, hold when linear function approximation is used.

\begin{lemma}\label{thm:CBR-LFA:stationary1}
A stationary point of $f$ always exists.
\end{lemma}
\begin{proof}
Since $f$ is bounded below and continuous by~\cref{thm:CBR-LFA:properties}, a minimizer always exists.
Since the subdifferential of $f$ always exists at any $\theta\in {\mathbb R}^m$ by~\cref{thm:CBR-LFA:properties}, at the minimizer $\theta$, we have $0 \in \partial f(\theta )$ by~\cref{thm:optimality1}, i.e., $\theta$ is a stationary point. This completes the proof.
\end{proof}

\begin{lemma}\label{thm:CBR-LFA:bound0}
For any $\theta \in {\mathbb R}^m$, we have
\begin{align*}
\left\| {{Q_\theta } - {Q^*}} \right\|_\infty \le& \frac{1}{{1 - \gamma }}{\left\| {T{Q_\theta } - {Q_\theta }} \right\|_2},\\
\left\| {{Q_\theta } - {Q^*}} \right\|_\infty ^2 \le& 2{\left( {\frac{1}{{1 - \gamma }}} \right)^2}f(\theta ).\\
\end{align*}
Moreover, we have
\begin{align*}
{\left\| {Q^{\pi _\theta } - Q^*} \right\|_\infty } \le& \frac{2\gamma }{{1 - \gamma }}{\left\| {T{Q_\theta } - {Q_\theta }} \right\|_2},\\
\left\| {{Q^{{\pi _\theta }}} - {Q^*}} \right\|_\infty ^2 \le& \frac{8{\gamma ^2}}{{(1 - \gamma)^2}}f(\theta ).
\end{align*}
where ${\pi _\theta }(s) =  {\argmax _{a \in {\cal A}}}{Q_\theta }(s,a)$.
\end{lemma}
\begin{proof}
The first statement can be directly obtained from~\cref{thm:fundamental4}.
The second statement can be obtained via simple calculations from the first statement.
For the third statement, using $T{Q_\theta } = {T^{{\pi _\theta }}}{Q_\theta }$, the triangle inequality, and the contraction of $T$ and $T^{\pi_\theta}$, we obtain
\begin{align*}
{\left\| {{Q^{{\pi _\theta }}} - {Q^*}} \right\|_\infty }
=& {\left\| {{T^{{\pi _\theta }}}{Q^{{\pi _\theta }}} - T{Q^*}} \right\|_\infty }\\
\le& {\left\| {{T^{{\pi _\theta }}}{Q^{{\pi _\theta }}} - {T^{{\pi _\theta }}}{Q_\theta }} \right\|_\infty } + {\left\| {T{Q_\theta } - T{Q^*}} \right\|_\infty }\\
\le& \gamma {\left\| {{Q^{{\pi _\theta }}} - {Q_\theta }} \right\|_\infty } + \gamma {\left\| {{Q_\theta } - {Q^*}} \right\|_\infty }.
\end{align*}
By~\cref{thm:fundamental1} applied to $\pi_\theta$ and by~\cref{thm:fundamental4},
\begin{align*}
{\left\| {{Q^{{\pi _\theta }}} - {Q_\theta }} \right\|_\infty }
&\le \frac{1}{{1 - \gamma }}{\left\| {{T^{\pi_\theta}}{Q_\theta } - {Q_\theta }} \right\|_\infty }
= \frac{1}{{1 - \gamma }}{\left\| {T{Q_\theta } - {Q_\theta }} \right\|_\infty },\\
{\left\| {{Q_\theta } - {Q^*}} \right\|_\infty }
&\le \frac{1}{{1 - \gamma }}{\left\| {T{Q_\theta } - {Q_\theta }} \right\|_\infty }.
\end{align*}
Combining the last three inequalities yields
\[{\left\| {{Q^{{\pi _\theta }}} - {Q^*}} \right\|_\infty } \le \frac{2\gamma }{{1 - \gamma }}{\left\| {T{Q_\theta } - {Q_\theta }} \right\|_\infty } \le \frac{2\gamma }{{1 - \gamma }}{\left\| {T{Q_\theta } - {Q_\theta }} \right\|_2}.\]
The last statement follows by squaring this bound and using $\|TQ_\theta-Q_\theta\|_2^2=2f(\theta)$. This completes the proof.
\end{proof}

\begin{proposition}\label{thm:CBR-LFA:bound4}
Suppose that $\theta _2^*: =  {\argmin _{\theta  \in {\mathbb R}^m}}f(\theta )$. Then, we have
\[{\left\| {{Q_{\theta _2^*}} - {Q^*}} \right\|_\infty } \le \frac{{(1 + \gamma )\sqrt {|{\cal S} \times {\cal A}|} }}{{1 - \gamma }}{\left\| {{\Gamma _{\Phi |\Phi }}{Q^*} - {Q^*}} \right\|_2}.\]
\end{proposition}
\begin{proof}
From the first statement of~\cref{thm:CBR-LFA:bound0}, we have
\[{\left\| {{Q_{\theta _2^*}} - {Q^*}} \right\|_\infty } \le \frac{1}{{1 - \gamma }}{\left\| {T{Q_{\theta _2^*}} - {Q_{\theta _2^*}}} \right\|_2}.\]
Moreover,~\cref{thm:CBR-LFA:quadratic-bound2} leads to
\[{\left\| {TQ_{\theta _2^*} - Q_{\theta _2^*}} \right\|_2} \le (1 + \gamma )\sqrt {|{\cal S} \times {\cal A}|} {\left\| {{\Gamma _{\Phi |\Phi }}{Q^*} - {Q^*}} \right\|_2}.\]
Combining the two inequalities concludes the proof.
\end{proof}

\begin{proposition}\label{thm:CBR-LFA:bound1}
Suppose that $\bar \theta$ is a stationary point satisfying Assumption~\ref{ass:oblique-regularity}. Let $\bar\pi:=\pi_{\bar\theta}$ and let $\bar\beta\in{\rm conv}(\Lambda_\Phi(\bar\theta))$ satisfy the stationary condition. Define
\[
A_{\bar\pi}:=\gamma P\Pi^{\bar\pi}-I,\qquad A_{\bar\beta}:=\gamma P\Pi^{\bar\beta}-I,
\qquad \Omega_{\bar\beta,\bar\pi}:=A_{\bar\pi}^\top A_{\bar\beta}\Phi.
\]
Assume that $\Omega_{\bar\beta,\bar\pi}^\top\Phi$ is nonsingular. Then,
\begin{align*}
\left\|Q^{\bar\pi}-Q^*\right\|_\infty
&\le \min\left\{
\frac{2\gamma}{1-\gamma}
\left\|I-\Gamma_{\Phi|\Omega_{\bar\beta,\bar\pi}}\right\|_2
\left\|Q^{\bar\pi}-\Gamma_{\Phi|\Phi}Q^{\bar\pi}\right\|_2,\right.\\
&\qquad\left.
\frac{\left\|T^{\bar\pi}Q^{\bar\pi}-TQ^{\bar\pi}\right\|_\infty}{1-\gamma}
\right\}.
\end{align*}
\end{proposition}
\begin{proof}
For any policy $\pi\in \Delta_{|{\cal A}|}$, using the triangle inequality, we obtain
\begin{align}
{\left\| {{Q^\pi } - {Q^*}} \right\|_\infty } \le \frac{1}{{1 - \gamma }}{\left\| {{T^\pi }{Q^\pi } - T{Q^\pi }} \right\|_\infty }.\label{eq:2}
\end{align}
For $\bar\pi=\pi_{\bar\theta}$, using $TQ_{\bar\theta}=T^{\bar\pi}Q_{\bar\theta}$, $TQ^*=Q^*$, and the contraction of $T$ and $T^{\bar\pi}$, we have
\begin{align}
{\left\| {{Q^{\bar\pi}} - {Q^*}} \right\|_\infty }
&= {\left\| {{T^{\bar\pi}}{Q^{\bar\pi}} - T{Q^*}} \right\|_\infty }\nonumber\\
&\le \gamma {\left\| {{Q^{\bar\pi}} - {Q_{\bar \theta} }} \right\|_\infty } + \gamma {\left\| {{Q_{\bar \theta}} - {Q^*}} \right\|_\infty }\nonumber\\
&\le \gamma {\left\| {{Q^{\bar\pi}} - {Q_{\bar \theta} }} \right\|_\infty } + \gamma\Bigl({\left\| {{Q_{\bar \theta}} - {Q^{\bar\pi}}} \right\|_\infty } + {\left\| {{Q^{\bar\pi}} - {Q^*}} \right\|_\infty }\Bigr),\label{eq:1-pre}
\end{align}
which implies
\begin{align}
{\left\| {{Q^{\bar\pi}} - {Q^*}} \right\|_\infty } \le \frac{2\gamma }{{1 - \gamma }}{\left\| {{Q^{\bar\pi}} - {Q_{\bar \theta} }} \right\|_\infty }.\label{eq:1}
\end{align}
Since $\bar\theta$ is stationary, $Q_{\bar\theta}$ satisfies
\[
\Phi^\top A_{\bar\beta}^\top(T^{\bar\pi}Q_{\bar\theta}-Q_{\bar\theta})=0.
\]
Equivalently, $Q_{\bar\theta}=\Gamma_{\Phi|\Omega_{\bar\beta,\bar\pi}}Q^{\bar\pi}$, and hence
\[
{\left\| {{Q^{\bar\pi}} - {Q_{\bar \theta }}} \right\|_2}
\le {\left\| {I - {\Gamma _{\Phi |\Omega _{\bar\beta,\bar\pi}}}} \right\|_2}{\left\| {{Q^{\bar\pi}} - {\Gamma _{\Phi |\Phi }}{Q^{\bar\pi}}} \right\|_2}.
\]
Combining this inequality with~\eqref{eq:1} and also using~\eqref{eq:2} with $\pi=\bar\pi$ gives the desired bound.
\end{proof}

\begin{theorem}\label{thm:CBR-LFA:bound2}
Let $\pi^*(s)=\argmax_{a \in {\cal A}}Q^*(s,a)$ be unique for each $s\in{\cal S}$, and define
\[
\Delta:=\min_{s\in{\cal S}}\min_{a\ne\pi^*(s)}\left\{Q^*(s,\pi^*(s))-Q^*(s,a)\right\}>0.
\]
Assume $0<\varepsilon<\Delta/2$. If ${\cal L}_{c_0}$ is nonempty with $c_0 = \frac{{{{(1 - \gamma )}^2}}}{2}{\varepsilon ^2}$, then the set
\[K:=\left\{ {\theta  \in {\mathbb R}^m:{{\left\| Q_\theta - Q ^* \right\|}_\infty } \le \varepsilon } \right\}\]
is nonempty, $K\subset S_{\pi^*}$, $f$ is strongly convex in $K$, and the global minimizer of $f$ is unique and lies in $K$. It is given by
\[\theta ^* =  {\argmin _{\theta  \in {\mathbb R}^m}}\frac{1}{2}\left\| {T^{\pi^*} Q_\theta - {Q_\theta }} \right\|_2^2.\]
Moreover, $\pi_{\theta^*}=\pi^*$, so ${\left\| Q^{\pi_{\theta^*}} - Q^* \right\|_\infty }=0$.
If $\Omega_{\pi^*}:=(\gamma P\Pi^{\pi^*}-I)^\top(\gamma P\Pi^{\pi^*}-I)\Phi$ satisfies $\Omega_{\pi^*}^\top\Phi$ nonsingular, then
\[
\|Q_{\theta^*}-Q^*\|_2\le \left\|I-\Gamma_{\Phi|\Omega_{\pi^*}}\right\|_2\left\|Q^*-\Gamma_{\Phi|\Phi}Q^*\right\|_2.
\]
\end{theorem}
\begin{proof}
The action-gap condition and $0<\varepsilon<\Delta/2$ imply that every $Q$ satisfying $\|Q-Q^*\|_\infty\le\varepsilon$ has the same greedy policy $\pi^*$ as $Q^*$. Hence $K\subset S_{\pi^*}$.
By~\cref{thm:fundamental4},
\[(1-\gamma){\left\| {{Q_\theta } - {Q^*}} \right\|_\infty } \le {\left\| {T{Q_\theta } - {Q_\theta }} \right\|_\infty } \le {\left\| {T{Q_\theta } - {Q_\theta }} \right\|_2},\]
so ${\cal L}_{c_0}\subseteq K$. If ${\cal L}_{c_0}$ is nonempty, then $K$ is nonempty. Moreover, any global minimizer $\theta_g$ of $f$ satisfies $f(\theta_g)\le c_0$ and therefore lies in $K$.
On $K$, the control Bellman operator agrees with $T^{\pi^*}$, so $f$ coincides with the strongly convex quadratic
\[
f^{\pi^*}(\theta):=\frac12\|T^{\pi^*}Q_\theta-Q_\theta\|_2^2.
\]
Since $\gamma P\Pi^{\pi^*}-I$ is nonsingular and $\Phi$ has full column rank, the Hessian of $f^{\pi^*}$ is positive definite. Therefore, $f^{\pi^*}$ has a unique global minimizer on $\mathbb R^m$; denote it by $\theta_{\pi^*}$.
Because $K$ is nonempty, choose $\theta_0\in K$. Then $f^{\pi^*}(\theta_0)=f(\theta_0)\le c_0$, and by optimality of $\theta_{\pi^*}$,
\[
f^{\pi^*}(\theta_{\pi^*})\le f^{\pi^*}(\theta_0)\le c_0.
\]
Applying~\cref{thm:fundamental1} to the policy $\pi^*$ yields
\[
(1-\gamma)\|Q_{\theta_{\pi^*}}-Q^*\|_\infty
\le \|T^{\pi^*}Q_{\theta_{\pi^*}}-Q_{\theta_{\pi^*}}\|_2
= \sqrt{2f^{\pi^*}(\theta_{\pi^*})}
\le \sqrt{2c_0}
= (1-\gamma)\varepsilon.
\]
Hence $\theta_{\pi^*}\in K$. Since every global minimizer of $f$ lies in $K$ and $f=f^{\pi^*}$ on $K$, the unique global minimizer of $f$ coincides with $\theta_{\pi^*}$ and therefore equals the unconstrained minimizer of $f^{\pi^*}$ displayed above. The action-gap condition also gives $\pi_{\theta^*}=\pi^*$, hence $Q^{\pi_{\theta^*}}=Q^*$.
The final approximation bound follows from~\cref{thm:fundamental3} applied to the policy $\pi^*$.
\end{proof}

\begin{theorem}\label{thm:CBR-LFA:gradient-descent-convergence}
Let $(\theta_k)_{k\geq 0}$ be generated by
\begin{align*}
\theta_{k+1}=\theta_k-\alpha_k g_k,
\end{align*}
where $g_k \in \argmin_{g\in\partial f(\theta_k)} \|g\|_2=\mathrm{Proj}_{\partial f(\theta_k)}(0)$ and $\alpha_k>0$ is a step size generated by backtracking search with the Armijo rule~\citep{Boyd2004}.
Then, the sequence $(\theta_k)_{k\geq 0} \subset\mathcal{L}_c$ with $c=f(\theta_0)$ admits at least one limit point $\bar \theta$. Every limit point $\bar \theta$ satisfies $0\in \partial f(\bar \theta)$, i.e., $\bar \theta$ is a stationary point.
\end{theorem}
\begin{proof}
The proof is almost identical to the proof of~\cref{thm:CBR:gradient-descent-convergence}.
Since $f$ is locally Lipschitz continuous, bounded below, and its level set is bounded by~\cref{thm:CBR-LFA:properties}, the convergence can be proved using~\cref{thm:convergence-subgradient1}.
\end{proof}

\subsection{Examples}
In this example, we verify~\cref{thm:stationary-2} on a simple linear function approximation example.
Let us consider an MDP with $\mathcal S=\{1\}$, $\mathcal A=\{1,2\}$, and $\gamma=0.9$ with deterministic self-loop transitions $P(1 | 1,1)=P(1 | 1,2)=1$, and rewards with $R(1,1)=0,R(1,2)=1$.
Let $Q: = \left[ {\begin{array}{*{20}{c}}
{Q(1,1)}\\
{Q(1,2)}
\end{array}} \right] \in {\mathbb R}^2$. We use a one-dimensional linear function approximation
\[{Q_\theta } = \Phi \theta ,\qquad \Phi  = \left[ {\begin{array}{*{20}{c}}
1\\
0
\end{array}} \right] \in {\mathbb R}^{2 \times 1},\qquad \theta  \in {\mathbb R}.\]
Hence, $Q_\theta=\begin{bmatrix}\theta\\0\end{bmatrix}$.

\subsubsection{Subdifferential and stationary point}

Since the transition is a self-loop, the control Bellman operator is
\[T{Q_\theta } = \left[ {\begin{array}{*{20}{c}}
{\gamma {{\max }_{a' \in \{ 1,2\} }}{Q_\theta }(1,a')}\\
{1 + \gamma {{\max }_{a' \in \{ 1,2\} }}{Q_\theta }(1,a')}
\end{array}} \right] = \left[ {\begin{array}{*{20}{c}}
{\gamma \max \{ \theta ,0\} }\\
{1 + \gamma \max \{ \theta ,0\} }
\end{array}} \right],\]
and the residual is
\[T{Q_\theta } - {Q_\theta } = \left[ {\begin{array}{*{20}{c}}
{\gamma \max \{ \theta ,0\}  - \theta }\\
{1 + \gamma \max \{ \theta ,0\} }
\end{array}} \right].\]
The CBR objective, $f(\theta)=\frac12\|T(\Phi\theta)-\Phi\theta\|_2^2$, is nondifferentiable at $\theta=0$ because of the max operator.
The graph of $f$ is depicted in~\cref{fig:ex-fig10}.
\begin{figure}[ht!]
\centering\includegraphics[width=0.7\textwidth, keepaspectratio]{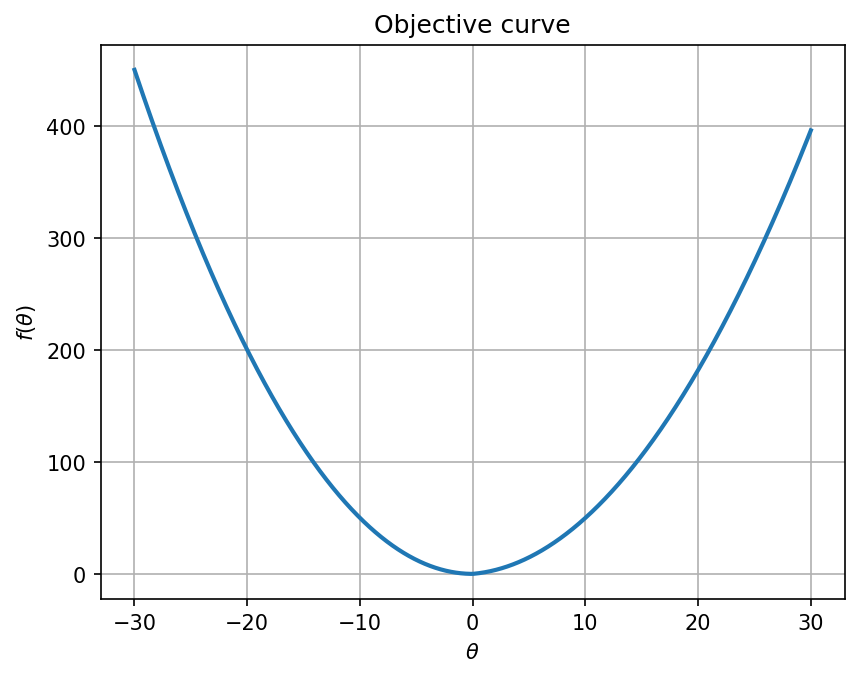}
\caption{The graph of the CBR objective, $f(\theta)=\frac12\|T(\Phi\theta)-\Phi\theta\|_2^2$.}\label{fig:ex-fig10}
\end{figure}
For any $\theta$ and any $\beta\in {\rm conv}(\Lambda_\Phi(\theta))$, we can represent
\[{\Pi ^\beta } = e_s^\top \otimes \beta {( \cdot |1)^\top} = \left[ {\begin{array}{*{20}{c}}
{\beta (1|1)}&{\beta (2|1)}
\end{array}} \right]\]
Moreover, since both state-action pairs transition back to $1$, the tabular transition
matrix from state-action to next-state is
\[
P=\begin{bmatrix}1\\1\end{bmatrix}\in\mathbb R^{2\times 1}.
\]
Therefore,
\[P{\Pi ^{\beta}} = \left[ {\begin{array}{*{20}{c}}
1\\
1
\end{array}} \right]\left[ {\begin{array}{*{20}{c}}
{\beta (1|1)}&{1 - \beta (1|1)}
\end{array}} \right] = \left[ {\begin{array}{*{20}{c}}
{\beta (1|1)}&{1 - \beta (1|1)}\\
{\beta (1|1)}&{1 - \beta (1|1)}
\end{array}} \right] \in {\mathbb R}^{2 \times 2},\]
which is row-stochastic. Using~\cref{thm:subdifferential-2}, we have
\[\partial f(\theta ) = \left\{ {{\Phi ^\top}{{(\gamma P{\Pi ^\beta} - I)}^\top}(TQ - Q)\;:\;\beta\in {\rm conv}(\Lambda_\Phi(\theta))} \right\},\]
and an element of this set can be parameterized by $\beta(1|1)\in [0,1]$ as
\begin{align*}
{\Phi ^\top}{(\gamma P{\Pi ^\beta} - I)^\top}(TQ - Q) =& \left[ {\begin{array}{*{20}{c}}
1\\
0
\end{array}} \right]{\left[ {\begin{array}{*{20}{c}}
{\gamma \beta(1|1)  - 1}&{\gamma(1 - \beta(1|1) )}\\
{\gamma\beta(1|1)}&{\gamma(1 - \beta(1|1)) - 1}
\end{array}} \right]^\top}\left[ {\begin{array}{*{20}{c}}
{\gamma \max \{ \theta ,0\}  - \theta }\\
{1 + \gamma \max \{ \theta ,0\} }
\end{array}} \right]\\
=& \left[ {\begin{array}{*{20}{c}}
{\gamma\beta(1|1)  - 1}&{\gamma\beta(1|1) }
\end{array}} \right]\left[ {\begin{array}{*{20}{c}}
{\gamma \max \{ \theta ,0\}  - \theta }\\
{1 + \gamma \max \{ \theta ,0\} }
\end{array}} \right]\\
=& (\gamma\beta(1|1) - 1)(\gamma \max \{ \theta ,0\}  - \theta ) + \gamma\beta(1|1) (1 + \gamma \max \{ \theta ,0\} )
\end{align*}

\paragraph{Case 1: $\theta>0$.}
$Q_\theta=[\theta,0]^\top$ and the greedy action is uniquely $a=1$. Therefore, $\beta (1|1) = 1,\beta (2|1) = 0$ and
${\rm conv}(\Lambda_\Phi(\theta))=\{\beta\}$. A direct calculation yields the singleton subdifferential
\[{\partial}f(\theta ) = \gamma  + ({\gamma ^2} + {(1 - \gamma )^2})\theta ,\qquad \theta  > 0.\]

\paragraph{Case 2: $\theta<0$.}
The greedy action is uniquely $a=2$, and hence $\beta (1|1) = 0,\beta (2|1) = 1$ and
${\rm conv}(\Lambda_\Phi(\theta))=\{\beta\}$.
A direct calculation yields
\[
\partial f(\theta)= \theta,
\qquad \theta<0.
\]

\paragraph{Case 3: $\theta=0$}
In this case, $Q_\theta=[0,0]^\top$, and hence, both actions $a=1,2$ are greedy. Therefore, ${\rm conv}(\Lambda_\Phi(\theta))$ contains all mixtures
\[\beta ( \cdot |1) = \left[ {\begin{array}{*{20}{c}}
{\beta (1|1)}\\
{\beta (2|1)}
\end{array}} \right] = \left[ {\begin{array}{*{20}{c}}
{\beta (1|1)}\\
{1 - \beta (1|1)}
\end{array}} \right],\qquad \beta (1|1) \in [0,1].\]
Moreover, the control Bellman residual is
\[
TQ_\theta-Q_\theta=\begin{bmatrix}0\\1\end{bmatrix}.
\]
Substituting the above formulations into the subdifferential formula gives
\[
\Phi^\top(\gamma P\Pi^\beta-I)^\top (TQ_\theta-Q_\theta)=\gamma\beta(1|1).
\]
Therefore, the Clarke subdifferential at $\theta=0$ satisfies
\[{\partial }f(\theta) = \{ 0.9\beta (1|1):\beta (1|1) \in [0,1]\}. \]

In summary, we have
\[{\partial }f(\theta) = \left\{ {\begin{array}{*{20}{c}}
{\gamma  + ({\gamma ^2} + {{(1 - \gamma )}^2})\theta ,\quad \theta  > 0}\\
{0.9\beta (1|1),\beta (1|1) \in [0,1],\quad \theta  = 0}\\
{\theta ,\quad \theta  < 0}
\end{array}} \right.\]
Since $0 \in {\partial }f(0)$, choosing $\beta(1|1)=0\in[0,1]$ makes the differential equal to $0$. Therefore, $\theta = 0$ is a stationary point.
At the stationary point $\theta = 0$, the CBR objective value is $f(0) = \frac{1}{2}\left\| {T{Q_\theta } - {Q_\theta }} \right\|_2^2 = \frac{1}{2}$.

\subsubsection{Oblique projected Bellman equation}

With the choice $\beta(1|1)=0$, $\Psi_\beta$ is given by
\[\Psi_\beta : = (\gamma P{\Pi ^\beta } - I)\Phi  = \left[ {\begin{array}{*{20}{c}}
{ - 1}&\gamma \\
0&{\gamma  - 1}
\end{array}} \right]\left[ {\begin{array}{*{20}{c}}
1\\
0
\end{array}} \right] = \left[ {\begin{array}{*{20}{c}}
{ - 1}\\
0
\end{array}} \right].\]
Moreover, the corresponding $\Psi_\beta^\top\Phi$ is
\[
\Psi_\beta^\top\Phi=
\begin{bmatrix}-1&0\end{bmatrix}
\begin{bmatrix}1\\0\end{bmatrix}
=-1.
\]
Therefore, $(\Psi_\beta^\top \Phi)^{-1}$ exists and equals $-1$.
\cref{thm:stationary-2} asserts that any Clarke-stationary point $\bar\theta$ satisfies the
oblique projected Bellman equation
\[{Q_{\bar \theta }} = {\Gamma _{\Phi \mid {\Psi _\beta }}}{\mkern 1mu} T{Q_{\bar \theta }},\quad {\Gamma _{\Phi \mid {\Psi _\beta }}}(x) = \Phi {(\Psi _\beta ^\top \Phi)^{- 1}}\Psi _\beta^\top x,\qquad {\Psi _\beta } = (\gamma P{\Pi ^\beta } - I)\Phi \]
for some $\beta\in\operatorname{co}(\Lambda(Q_{\bar\theta}))$.
In our example with $\bar\theta=0$ and $\beta(1|1)=0$, for any $x=\begin{bmatrix}x_1\\x_2\end{bmatrix}\in\mathbb R^2$,
\[{\Gamma _{\Phi \mid {\Psi _\beta }}}(x) = \Phi {(\Psi _\beta ^\top\Phi )^{ - 1}}\Psi _\beta ^\top x = \left[ {\begin{array}{*{20}{c}}
1\\
0
\end{array}} \right]( - 1)\left[ {\begin{array}{*{20}{c}}
{ - 1}&0
\end{array}} \right]\left[ {\begin{array}{*{20}{c}}
{{x_1}}\\
{{x_2}}
\end{array}} \right] = \left[ {\begin{array}{*{20}{c}}
{{x_1}}\\
0
\end{array}} \right].\]
Moreover, we have
\[
Q_{\bar\theta}=\begin{bmatrix}0\\0\end{bmatrix},
\qquad
TQ_{\bar\theta}=
\begin{bmatrix}0\\1\end{bmatrix}.
\]
Therefore, the OP-BE is computed as
\[{\Gamma _{\Phi \mid {\Psi _\beta }}}(T{Q_0}) = {\Gamma _{\Phi \mid {\Psi _\beta }}}\left( {\left[ {\begin{array}{*{20}{c}}
0\\
1
\end{array}} \right]} \right) = \left[ {\begin{array}{*{20}{c}}
0\\
0
\end{array}} \right] = {Q_{\bar \theta }}.\]
Therefore, the Clarke-stationary point $\bar\theta=0$ indeed satisfies the OP-BE.

\subsubsection{Subgradient method}
We can apply the subgradient descent method described in~\cref{thm:CBR-LFA:gradient-descent-convergence} to this example, and the results are presented in~\cref{fig:ex-fig11}, which shows the evolution of the iterates of the subgradient descent method converging to the stationary point.
\begin{figure}[ht!]
\centering\includegraphics[width=0.7\textwidth, keepaspectratio]{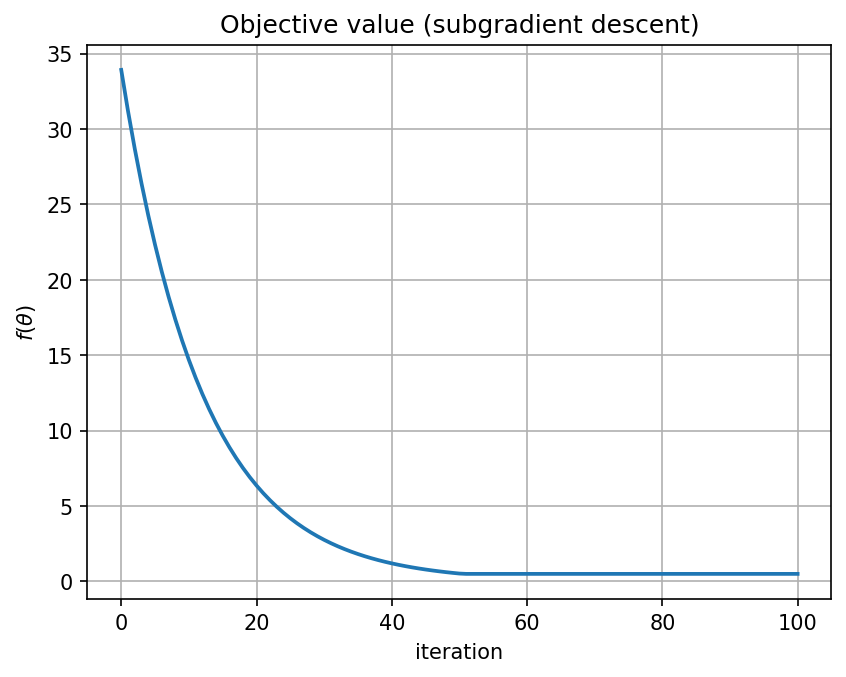}
\caption{The evolution of the iterates of the subgradient descent method converging to the stationary point.}\label{fig:ex-fig11}
\end{figure}

\newpage
\section{SCBR tabular case}

\begin{proposition}\label{thm:SCBR-tabular:levelset-1}
The level set ${\cal L}_c = \{ Q \in {\mathbb R}^{|{\cal S}\times {\cal A}|}:f(Q) \le c\}$ is bounded and closed (compact).
\end{proposition}
\begin{proof}
First of all, we can derive the following bounds:
\begin{align*}
{\left\| {Q - F_\lambda Q} \right\|_\infty } =& {\left\| {Q - Q_\lambda^* + F_\lambda Q_\lambda^* - F_\lambda Q} \right\|_\infty }\\
\ge& {\left\| {Q - {Q_\lambda^*}} \right\|_\infty } - {\left\| {F_\lambda Q_\lambda^* - F_\lambda Q} \right\|_\infty }\\
\ge& {\left\| {Q - {Q_\lambda^*}} \right\|_\infty } - \gamma {\left\| {{Q_\lambda^*} - Q} \right\|_\infty }\\
=& (1 - \gamma ){\left\| {Q - {Q_\lambda^*}} \right\|_\infty }\\
\ge& (1 - \gamma )\frac{1}{{\sqrt {|{\cal S}\times {\cal A}|} }}{\left\| {Q - {Q_\lambda^*}} \right\|_2},
\end{align*}
where the first inequality comes from the reverse triangular inequality and the second inequality is due to the contraction property of the soft control Bellman operator $F_\lambda$~\citep{dai2018sbeed}.
Using the above bound, the objective $f$ can be bounded as
\begin{align*}
f(Q) = \frac{1}{2}\left\| {Q - {F_\lambda }Q} \right\|_2^2 \ge \frac{1}{2}\left\| {Q - {F_\lambda }Q} \right\|_\infty ^2 \ge \frac{{{{(1 - \gamma )}^2}}}{{2|{\cal S}\times {\cal A}|}}\left\| {Q - {Q_\lambda^*}} \right\|_2^2.
\end{align*}

Therefore, $f(Q) \le c$ implies ${\left\| Q - Q_\lambda^* \right\|_2} \le \sqrt {\frac{{2|{\cal S}\times {\cal A}|c}}{(1 - \gamma )^2}} $. This implies the following inclusion relation:
\begin{align*}
{\cal L}_c \subseteq \left\{ {Q \in {\mathbb R}^{|{\cal S}\times {\cal A}|}:{{\left\| {Q - {Q_\lambda^*}} \right\|}_2} \le \sqrt {\frac{{2|{\cal S}\times {\cal A}|c}}{{{{(1 - \gamma )}^2}}}} } \right\}.
\end{align*}
Since the right-hand side above is bounded, so is ${\cal L}_c$.
Moreover, since $f$ is continuous, its level set is closed~\citep[Theorem~1.6]{rockafellar1998variational}. This completes the proof.
\end{proof}

\begin{lemma}\label{thm:SCBR-tabular:C-inf}
$f$ is $C^\infty$.
\end{lemma}
\begin{proof}
Note that the linear mapping $\theta  \mapsto {Q_\theta } = \Phi \theta $ is $C^\infty$, $x \mapsto \frac{1}{2}\left\| x \right\|_2^2$ is $C^\infty$, and $Q \mapsto {F_\lambda }(Q) - Q$ is $C^\infty$. Since $f$ is their composition, $f$ is also $C^\infty$.
\end{proof}

\begin{lemma}\label{thm:subdifferential-3}
The gradient of $f$ is given by
\[{\nabla_Q}f(Q) = (\gamma P{\Pi ^{{\pi _Q}}} - I)^\top({F_\lambda }Q - Q),\]
where ${\pi _Q}(j|i): = \frac{{\exp (Q(i,j)/\lambda )}}{{\sum_{u \in {\cal A}} {\exp (Q(i,u)/\lambda )} }}$ is the Boltzmann policy of $Q$.
\end{lemma}
\begin{proof}
By direct calculations, we can derive
\begin{align*}
\frac{{\partial f(Q)}}{{\partial Q(i,j)}} = \sum\limits_{(s,a) \in {\cal S} \times {\cal A}} {\delta (s,a)\frac{\partial }{{\partial Q(i,j)}}\left[ {({F_\lambda }Q)(s,a) - Q(s,a)} \right]} ,
\end{align*}
where $\delta (s,a): = ({F_\lambda }Q)(s,a) - Q(s,a)$. In the equation, the term $\frac{\partial }{\partial Q(i,j)}[({F_\lambda }Q)(s,a) - Q(s,a)]$ can be further expressed as
\begin{align*}
&\frac{\partial }{{\partial Q(i,j)}}[({F_\lambda }Q)(s,a) - Q(s,a)]\\
=& \gamma \sum\limits_{s' \in {\cal S}} {P(s'|s,a)\frac{\partial }{\partial Q(i,j)}\lambda \ln \left( {\sum\limits_{u \in {\cal A}} {\exp (Q(s',u)/\lambda )} } \right)}  - {\bf{1}}[(i,j) = (s,a)]\\
=& \gamma \sum\limits_{s' \in {\cal S}} {P(s'|s,a){\pi _Q}(j|i)}  - {\bf 1}[(i,j) = (s,a)],
\end{align*}
where ${\bf 1}[(i,j) = (s,a)]$ is the indicator function whose value is one if $(i,j) = (s,a)$ and zero otherwise, and
\begin{align*}
{\pi _Q}(j|i): = \frac{{\exp (Q(i,j)/\lambda )}}{{\sum\limits_{u \in {\cal A}} {\exp (Q(i,u)/\lambda )} }}
\end{align*}
is the Boltzmann policy corresponding to $Q$. Combining the last results leads to
\begin{align*}
\frac{{\partial f(Q)}}{{\partial Q(i,j)}} =& \sum_{(s,a) \in {\cal S} \times {\cal A}} {\delta (s,a)[\gamma P(i|s,a){\pi _Q}(j|i) - {\bf{1}}[(i,j) = (s,a)]]}\\
=& \gamma \sum_{(s,a) \in {\cal S} \times {\cal A}} {\delta (s,a)P(i|s,a){\pi _Q}(j|i)}  - \delta (i,j).
\end{align*}
Now, the gradient can be written as
\begin{align*}
{\nabla _Q}f(Q) = \left[ {\begin{array}{*{20}{c}}
{\frac{{\partial f(Q)}}{{\partial Q(1,1)}}}\\
{\frac{{\partial f(Q)}}{{\partial Q(1,2)}}}\\
 \vdots \\
{\frac{{\partial f(Q)}}{{\partial Q(|{\cal S}|,|{\cal A}|)}}}
\end{array}} \right] = {(\gamma P{\Pi ^{{\pi _Q}}} - I)^\top}({F_\lambda }Q - Q) \in {\mathbb R}^{|{\cal S} \times {\cal A}|}.
\end{align*}
This completes the proof.
\end{proof}

\begin{theorem}\label{thm:stationary-3}
The stationary point, ${\nabla _Q}f(\bar Q) = 0$, is the unique solution $\bar Q =Q_\lambda ^*$ to the soft control Bellman equation ${F_\lambda }Q_\lambda ^* = Q_\lambda ^*$.
\end{theorem}
\begin{proof}
Suppose that $\bar Q$ is a stationary point, which implies ${\nabla _Q}f(\bar Q) = {(\gamma P{\Pi ^{{\pi _{\bar Q}}}} - I)^\top}({F_\lambda }\bar Q - \bar Q) = 0$. Since $(\gamma P{\Pi ^{{\pi _{\bar Q}}}} - I)$ is invertible from~\cref{thm:matrix_inversion2}, we have ${F_\lambda }\bar Q = \bar Q$.
This completes the proof.
\end{proof}

\begin{lemma}\label{thm:SCBR-tabular:hessian-1}
The Hessian of $f$ (under the state-major ordering introduced in Section~3.2) is given by
\[\nabla _Q^2f(Q): = {(\gamma P{\Pi ^{{\pi _Q}}} - I)^\top}(\gamma P{\Pi ^{{\pi _Q}}} - I) + \frac{\gamma }{\lambda }\left[ {\begin{array}{*{20}{c}}
{{\Omega _1}(Q)}&0& \cdots &0\\
0&{{\Omega _2}(Q)}& \ddots & \vdots \\
 \vdots & \ddots & \ddots &0\\
0& \cdots &0&{{\Omega _{|{\cal S}|}}(Q)}
\end{array}} \right]\]
where
\begin{align*}
{\Omega _s}(Q):=& {m_s}\left( {{\rm{diag}}({\pi _Q}( \cdot |s)) - {\pi _Q}( \cdot |s){\pi _Q}{{( \cdot |s)}^\top}} \right)\\
{\rm{diag}}({\pi _Q}( \cdot |s)):=& \left[ {\begin{array}{*{20}{c}}
{{\pi _Q}(1|s)}&0& \cdots &0\\
0&{{\pi _Q}(2|s)}& \ddots & \vdots \\
 \vdots & \ddots & \ddots &0\\
0& \cdots &0&{{\pi _Q}(|{\cal A}|\mid s)}
\end{array}} \right]\\
{m_i}: =& \sum\limits_{(s,a) \in {\cal S} \times {\cal A}} {\delta (s,a)P(i|s,a)} \\
\delta (s,a): =& ({F_\lambda }Q)(s,a) - Q(s,a).
\end{align*}
\end{lemma}
The proof can be completed via direct calculations.

\begin{lemma}\label{thm:SCBR-tabular:property1}
We have
\[{\rm{diag}}({\pi _Q}( \cdot |i)) - {\pi _Q}( \cdot |i){\pi _Q}{( \cdot |i)^\top} \succeq 0,\quad \forall Q \in {\mathbb R}^{|{\cal S}\times {\cal A}|},\]
and this matrix is singular because it annihilates the all-ones vector. Moreover,
\[{\left\| {{\rm{diag}}({\pi _Q}( \cdot |s)) - {\pi _Q}( \cdot |s){\pi _Q}{{( \cdot |s)}^\top}} \right\|_2} \le 2,\]
where $\rm{diag}$ is defined in~\cref{thm:SCBR-tabular:hessian-1}.
\end{lemma}
\begin{proof}
For any $x \in {\mathbb R}^{|{\cal A}|}$, we have
\begin{align*}
{x^\top}\left\{ {{\rm{diag}}({\pi _Q}( \cdot |i)) - {\pi _Q}( \cdot |i){\pi _Q}{{( \cdot |i)}^\top}} \right\}x=& \sum\limits_{j \in {\cal A}} {x_j^2{\pi _Q}(j|i)}  - {\left( {\sum\limits_{j \in {\cal A}} {{x_j}{\pi _Q}(j|i)} } \right)^2}\\
=& {\mathbb E}_{j \sim {\pi _Q}( \cdot |i)}[x_j^2] - {{\mathbb E}_{j \sim {\pi _Q}( \cdot |i)}}{[{x_j}]^2}\\
\ge& 0
\end{align*}
where the last inequality follows from Jensen's inequality.
This implies that ${\rm{diag}}({\pi _Q}( \cdot |i)) - {\pi _Q}( \cdot |i){\pi _Q}{( \cdot |i)^\top}$ is positive semidefinite for any $Q \in {\mathbb R}^{|{\cal S}\times {\cal A}|}$. The second statement can be proved using the following bounds:
\begin{align*}
{\left\| {{\rm{diag}}({\pi _Q}( \cdot |s)) - {\pi _Q}( \cdot |s){\pi _Q}{{( \cdot |s)}^\top}} \right\|_2} \le& {\left\| {{\rm{diag}}({\pi _Q}( \cdot |s))} \right\|_2} + {\left\| {{\pi _Q}( \cdot |s){\pi _Q}{{( \cdot |s)}^\top}} \right\|_2}\\
\le& 1 + {\left\| {{\pi _Q}( \cdot |i){\pi _Q}{{( \cdot |i)}^\top}} \right\|_2}\\
\le& 1 + \sqrt {{{\left\| {{\pi _Q}( \cdot |i){\pi _Q}{{( \cdot |i)}^\top}} \right\|}_1}{{\left\| {{\pi _Q}( \cdot |i){\pi _Q}{{( \cdot |i)}^\top}} \right\|}_\infty }} \\
\le& 2
\end{align*}
This completes the proof.
\end{proof}

\begin{lemma}
The set $M: = \{ Q \in {\mathbb R}^{|{\cal S}\times {\cal A}|}:{F_\lambda }Q - Q \ge 0\}$ satisfies the following properties:
\begin{enumerate}
\item $M$ is closed
\item $Q^*_\lambda \in M$
\item $M \subseteq \{ Q \in {\mathbb R}^{|{\cal S}\times {\cal A}|}:Q_\lambda ^* \ge Q\} $
\end{enumerate}
\end{lemma}
\begin{proof}
Since ${F_\lambda }$ is continuous and $K:=\{ Q \in {\mathbb R}^{|{\cal S}\times {\cal A}|}:Q \ge 0\}$ is a closed set, $M$ is closed. In particular, let us define the mapping $g(Q): = {F_\lambda }(Q) - Q$, which is continuous. Then, $M = {g^{ - 1}}(K)$, which is the inverse mapping of the closed set $K$. This implies that $M$ is closed~\citep[Theorem~4.8]{rudin1976principles}.
Moreover, if $Q \in M$, i.e., ${F_\lambda }Q \ge Q$, then $F_\lambda ^kQ \ge  \cdots  \ge F_\lambda^3Q \ge {F_\lambda }Q \ge Q$ because $F_\lambda$ is monotone~\citep{bertsekas1996neuro}. Therefore, $\lim_{k \to \infty } F_\lambda ^kQ = Q_\lambda ^* \ge Q$, and $Q \in \{ Q \in {\mathbb R}^{|{\cal S}\times {\cal A}|}:Q_\lambda ^* \ge Q\}$.
\end{proof}

\begin{lemma}
For $M: = \{ Q \in {\mathbb R}^{|{\cal S}\times {\cal A}|}:{F_\lambda }Q - Q \ge 0\}$, the Hessian is uniformly positive definite on $M$. Consequently, $f$ is strongly convex on every convex subset of $M$.
\end{lemma}
\begin{proof}
Let $J(Q):=\gamma P\Pi^{\pi_Q}-I$. By~\cref{thm:matrix_inversion2}, $J(Q)$ is nonsingular for every $Q$. Moreover, the inverse is uniformly bounded because $\|(I-\gamma P\Pi^{\pi_Q})^{-1}\|_\infty\le (1-\gamma)^{-1}$, and hence $J(Q)^\top J(Q)\succeq c_0 I$ for some $c_0>0$ independent of $Q$. In $M$, we have $m_i\ge0$, so the covariance term in~\cref{thm:SCBR-tabular:hessian-1} is positive semidefinite by~\cref{thm:SCBR-tabular:property1}. Therefore, $\nabla_Q^2 f(Q)\succeq c_0I$ on $M$. The Hessian criterion for strong convexity then applies on any convex subset of $M$. This completes the proof.
\end{proof}

\begin{lemma}\label{thm:SCBR-tabular:L-smooth-1}
$f$ is locally $L$-smooth, i.e., for any compact $C$, there exists $L(C)>0$ such that
\[{\left\| {{\nabla _Q}f({Q_1}) - {\nabla _Q}f({Q_2})} \right\|_2} \le L(C){\left\| {{Q_1} - {Q_2}} \right\|_2}\]
for all ${Q_1},{Q_2} \in C$, where one may take
\[L(C) = {(\gamma \sqrt{|{\cal S}\times {\cal A}|}  + 1)^2} + \frac{2\gamma}{\lambda}\|P^\top\|_\infty(\gamma  + 1){\max _{Q \in C}}{\left\| {Q_\lambda ^* - Q} \right\|_\infty }.\]
\end{lemma}
\begin{proof}
To prove the Lipschitz smoothness of the gradient, we use the Hessian in~\cref{thm:SCBR-tabular:hessian-1} and the Hessian-based condition in~\citet[Lemma~1.2.2]{nesterov2018lectures}.
First,
\begin{align*}
{\left\| {(\gamma P{\Pi ^{{\pi _Q}}} - I)^\top(\gamma P{\Pi ^{{\pi _Q}}} - I)} \right\|_2}
\le {\left(\gamma\|P\Pi^{\pi_Q}\|_2+1\right)^2}
\le {(\gamma \sqrt {|{\cal S} \times {\cal A}|}  + 1)^2}.
\end{align*}
Moreover,
\begin{align*}
\max_i |m_i|=
\left\|P^\top(F_\lambda Q-Q)\right\|_\infty
\le \|P^\top\|_\infty\|F_\lambda Q-Q\|_\infty
\le \|P^\top\|_\infty(\gamma+1)\|Q_\lambda^*-Q\|_\infty.
\end{align*}
Combining this with~\cref{thm:SCBR-tabular:property1} and with the factor $\gamma/\lambda$ in the Hessian gives
\begin{align*}
{\left\| {\nabla _Q^2f(Q)} \right\|_2}
\le& {(\gamma \sqrt{|{\cal S}\times {\cal A}|}  + 1)^2} + \frac{2\gamma}{\lambda}\max_{s\in{\cal S}}|m_s|\\
\le& {(\gamma \sqrt{|{\cal S}\times {\cal A}|}  + 1)^2} + \frac{2\gamma}{\lambda}\|P^\top\|_\infty(\gamma+1)\|Q_\lambda^*-Q\|_\infty.
\end{align*}
Taking the maximum over $Q\in C$ yields the stated $L(C)$, and~\citet[Lemma~1.2.2]{nesterov2018lectures} completes the proof.
\end{proof}

\begin{lemma}\label{thm:SCBR:strong-convex-1}
$f$ is locally strongly convex around $Q^*_\lambda$, i.e., there exists a ball ${B_r}(Q_\lambda ^*)$ centered at $Q_\lambda ^*$ with radius $r>0$ and a number $\mu>0$ such that
\[\nabla _Q^2f(Q) \succeq \mu I,\quad \forall Q \in {B_r}(Q_\lambda ^*)\]
\end{lemma}
\begin{proof}
Let us consider the Hessian given in~\cref{thm:SCBR-tabular:hessian-1}. Since
\[\mathop {\lim }\limits_{Q \to Q_\lambda ^*} m_i = \lim_{Q \to Q_\lambda ^*} \sum\limits_{(s,a) \in {\cal S} \times {\cal A}} {\delta (s,a)P(i|s,a)}  = 0,\]
we have
\[\mathop {\lim }\limits_{Q \to Q_\lambda ^*} \nabla _Q^2f(Q) = {(\gamma P{\Pi ^{{\pi _{Q_\lambda ^*}}}} - I)^\top}(\gamma P{\Pi ^{{\pi _{Q_\lambda ^*}}}} - I) \succ 0\]
By~\citet[Theorem~2.1.11]{nesterov2018lectures}, this implies that $f$ is locally strongly convex around $Q^*_\lambda$.
\end{proof}

\begin{lemma}\label{thm:SCBR:PL-condition-1}
$f$ is Polyak--{\L}ojasiewicz in any compact set $C$, i.e., it satisfies the following inequality:
\[f(Q) - f(Q_\lambda ^*) \le L(C)\left\| {{\nabla _Q}f(Q)} \right\|_2^2,\quad \forall Q \in C\]
for a constant $L(C)>0$.
\end{lemma}
\begin{proof}
Using the fact that $\gamma P{\Pi ^{{\pi _Q}}} - I$ is invertible by~\cref{thm:matrix_inversion2}, we can derive the following lower bound:
\begin{align*}
\left\| {{\nabla _Q}f(Q)} \right\|_2^2 =& \left\| {{{(\gamma P{\Pi ^{{\pi _Q}}} - I)}^\top}({F_\lambda }Q - Q)} \right\|_2^2\\
=& \left\| {{F_\lambda }Q - Q} \right\|_{(\gamma P{\Pi ^{{\pi _Q}}} - I){{(\gamma P{\Pi ^{{\pi _Q}}} - I)}^\top}}^2\\
\ge& {\lambda _{\min }}((\gamma P{\Pi ^{{\pi _Q}}} - I){(\gamma P{\Pi ^{{\pi _Q}}} - I)^\top})\left\| {{F_\lambda }Q - Q} \right\|_2^2\\
\ge& {\min _{Q \in C}}{\lambda _{\min }}((\gamma P{\Pi ^{{\pi _Q}}} - I){(\gamma P{\Pi ^{{\pi _Q}}} - I)^\top})\left\| {{F_\lambda }Q - Q} \right\|_2^2\\
=& l(C)\left\| {{F_\lambda }Q - Q} \right\|_2^2
\end{align*}
where $l(C)>0$ is a constant and $\lambda _{\min }$ stands for the minimum eigenvalue.
Therefore, using the last inequality, one has the following upper bound:
\begin{align*}
f(Q) - f(Q_\lambda ^*) =& \frac{1}{2}\left\| {Q - {F_\lambda }Q} \right\|_2^2 - \frac{1}{2}\left\| {Q_\lambda ^* - {F_\lambda }Q_\lambda ^*} \right\|_2^2\\
=& \frac{1}{2}\left\| {Q - {F_\lambda }Q} \right\|_2^2\\
\le& \frac{1}{{2l(C)}}\left\| {{\nabla _Q}f(Q)} \right\|_2^2.
\end{align*}
This completes the proof.
\end{proof}

Before proving the convergence of the gradient descent algorithm, we first present the following descent lemma, which is essential for our analysis.
\begin{lemma}[Theorem~2.1.5 in~\citet{nesterov2018lectures}]\label{lemma:descent-lemma}
If $f:{\mathbb R}^n \to {\mathbb R}$ is $L$-smooth, then
\begin{align*}
f(y) \le f(x) + \left\langle {\nabla f(x),y - x} \right\rangle  + \frac{L}{2}\left\| y - x \right\|_2^2,\quad \forall x,y \in {\mathbb R}^n.
\end{align*}
\end{lemma}
We now present the convergence result for the gradient descent algorithm.
\begin{theorem}\label{thm:SCBR-tabular:convergence}
Let us consider the gradient descent iterates, $Q_{k + 1} = Q_k - \alpha {\nabla _Q}f(Q_k)$, for $k=0,1,\ldots$ with any initial point $Q_0 \in {\mathbb R}^{|{\cal S}\times {\cal A}|}$. Fix $c>f(Q_0)$ and set ${\cal L}_c:=\{Q:f(Q)\le c\}$. For the chosen step size $\alpha$, define
\[
D_\alpha:=\operatorname{conv}\left({\cal L}_c\cup\{Q-\alpha\nabla_Q f(Q):Q\in{\cal L}_c\}\right),
\]
and let $L_\alpha>0$ be a Lipschitz constant of $\nabla_Qf$ on $D_\alpha$. Suppose $l_c>0$ satisfies
$f(Q)-f(Q_\lambda^*)\le l_c\|\nabla_Q f(Q)\|_2^2$ for all $Q\in {\cal L}_c$.
If $0<\alpha < \frac{2}{L_\alpha}$ and ${l_c} > \alpha\left(1-\frac{\alpha L_\alpha}{2}\right)$, the iterates satisfy
\begin{align*}
f(Q_k) - f(Q_\lambda ^*)\le {\left( {1 - \alpha \left( {1 - \frac{{\alpha {L_\alpha}}}{2}} \right)\frac{1}{{{l_c}}}} \right)^k}\left[ {f({Q_0}) - f(Q_\lambda ^*)} \right],
\end{align*}
where $l_c$ is any such Polyak--{\L}ojasiewicz constant on ${\cal L}_c$.
\end{theorem}
\begin{proof}
For $Q\in{\cal L}_c$, the point $Q^+ := Q-\alpha\nabla_Q f(Q)$ belongs to $D_\alpha$ by definition, and the whole line segment between $Q$ and $Q^+$ is contained in the convex set $D_\alpha$. Applying the descent lemma with the Lipschitz constant $L_\alpha$ gives
\[
f(Q^+)\le f(Q)-\alpha\left(1-\frac{\alpha L_\alpha}{2}\right)\|\nabla_Q f(Q)\|_2^2.
\]
Since $0<\alpha<2/L_\alpha$, this implies $f(Q^+)\le f(Q)$, so by induction all iterates remain in ${\cal L}_c$.
Using the Polyak--{\L}ojasiewicz condition on ${\cal L}_c$,
\begin{align*}
f(Q^+) - f(Q_\lambda ^*)
&\le f(Q)-f(Q_\lambda^*)-\alpha\left(1-\frac{\alpha L_\alpha}{2}\right)\|\nabla_Q f(Q)\|_2^2\\
&\le \left(1-\alpha\left(1-\frac{\alpha L_\alpha}{2}\right)\frac{1}{l_c}\right)(f(Q)-f(Q_\lambda^*)).
\end{align*}
Iterating the inequality proves the claim.
\end{proof}

\subsection{Examples}\label{sec:SCBR:example:1}

Let us consider a discounted MDP with a single state and two actions: $\mathcal{S}=\{1\}, \mathcal{A}=\{1,2\}, 0<\gamma<1, \lambda>0$.
Moreover, let us assume that the transition is self-looping for both actions: $P(1\mid 1,1)=P(1\mid 1,2)=1$, and
\[
Q :=\begin{bmatrix} Q(1,1)\\ Q(1,2) \end{bmatrix}.
\]

For this MDP, the soft Bellman operator $F_\lambda$ satisfies, for $a\in\{1,2\}$,
\[({F_\lambda }Q)(1,a) = R(1,a) + \gamma {\mkern 1mu} \lambda \ln (\exp (Q(1,1)/\lambda ) + \exp (Q(1,2)/\lambda )).\]

Define the soft control Bellman residual (SCBR)
\begin{align*}
\delta (Q): =& {F_\lambda }Q - Q\\
=& \left[ {\begin{array}{*{20}{c}}
{R(1,1) + \gamma \lambda \ln \left( {\exp (Q(1,1)/\lambda ) + \exp (Q(1,2)/\lambda )} \right) - Q(1,1)}\\
{R(1,2) + \gamma \lambda \ln \left( {\exp (Q(1,1)/\lambda ) + \exp (Q(1,2)/\lambda )} \right) - Q(1,2)}
\end{array}} \right]\\
& = \left[ {\begin{array}{*{20}{c}}
{{\delta _1}(Q)}\\
{{\delta _2}(Q)}
\end{array}} \right],
\end{align*}
and the SCBR objective
\[f(Q)\;: = \;\frac{1}{2}\left\| {\delta (Q)} \right\|_2^2\;.\]

The Boltzmann (softmax) policy induced by $Q$ is
\[{\pi _\lambda }(1|1) = \frac{{\exp (Q(1,1)/\lambda )}}{{\exp (Q(1,1)/\lambda ) + \exp (Q(1,2)/\lambda )}},\qquad {\pi _\lambda }(2|1) = \frac{{\exp (Q(1,2)/\lambda )}}{{\exp (Q(1,1)/\lambda ) + \exp (Q(1,2)/\lambda )}}\]

Let $\pi_1:={\pi _\lambda }(1\mid 1)$ and $\pi_2:={\pi _\lambda }(2\mid 1)$. Then the Jacobian of $F_\lambda$ is
\[\nabla_Q ({F_\lambda }Q) = \gamma \left[ {\begin{array}{*{20}{c}}
{{\pi _\lambda }(1|1)}&{{\pi _\lambda }(2|1)}\\
{{\pi _\lambda }(1|1)}&{{\pi _\lambda }(2|1)}
\end{array}} \right].\]

Using the identity ${\nabla _Q}f(Q) = {({\nabla _Q}({F_\lambda }Q) - I)^\top}\delta (Q)$, we obtain the explicit gradient
\[{\nabla _Q}f(Q) = \left[ {\begin{array}{*{20}{c}}
{(\gamma {\pi _\lambda }(1|1) - 1){\delta _1}(Q) + \gamma {\pi _\lambda }(1|1){\mkern 1mu} {\delta _2}(Q)}\\
{\gamma {\pi _\lambda }(2|1){\mkern 1mu} {\delta _1}(Q) + (\gamma {\pi _\lambda }(2|1) - 1){\delta _2}(Q)}
\end{array}} \right].\]

A stationary point satisfies $\nabla_Q f(Q)=0$. Since
\begin{align*}
\det ({\nabla _Q}({F_\lambda }Q) - I) =& \det \left( {\left[ {\begin{array}{*{20}{c}}
{\gamma {\pi _\lambda }(1|1) - 1}&{\gamma {\pi _\lambda }(2|1)}\\
{\gamma {\pi _\lambda }(1|1)}&{\gamma {\pi _\lambda }(2|1) - 1}
\end{array}} \right]} \right)\\
=& (\gamma {\pi _\lambda }(1|1) - 1)(\gamma {\pi _\lambda }(2|1) - 1) - {\gamma ^2}{\pi _\lambda }(1|1){\pi _\lambda }(2|1)\\
=& 1 - \gamma \\
\ne& 0
\end{align*}
we conclude ${\nabla _Q}f(Q) = 0 \Leftrightarrow \delta (Q) = 0 \Leftrightarrow {F_\lambda }Q = Q$, i.e., the stationary point coincides with the unique fixed point of the soft Bellman operator. Solving $F_\lambda Q=Q$ yields a closed-form expression. From $\delta_1(Q)=\delta_2(Q)=0$, we obtain
\[Q(1,1) - Q(1,2) = R(1,1) - R(1,2)\]
and substituting $Q(1,2) = Q(1,1) - \{ R(1,1) - R(1,2)\} $ into $\delta_1(Q)=0$ gives
\begin{align*}
&R(1,1) + \gamma \lambda \ln (\exp (Q(1,1)/\lambda ) + \exp (Q(1,2)/\lambda )) - Q(1,1)\\
=& R(1,1) + \gamma \lambda \ln (\exp (Q(1,1)/\lambda ) + \exp (Q(1,1)/\lambda  - \{ R(1,1) - R(1,2)\} /\lambda )) - Q(1,1)\\
=& R(1,1) + \gamma \lambda \ln (\exp (Q(1,1)/\lambda )[1 + \exp (\{ R(1,2) - R(1,1)\} /\lambda )]) - Q(1,1)\\
=& R(1,1) + \gamma Q(1,1) + \gamma \lambda \ln (1 + \exp (\{ R(1,2) - R(1,1)\} /\lambda )) - Q(1,1)\\
=& 0
\end{align*}
which is equivalently written as $(1 - \gamma )Q(1,1) = R(1,1) + \gamma \lambda \ln (1 + \exp ( - (R(1,1) - R(1,2))/\lambda ))$. Therefore, the solution of the SCBE can be calculated as
\begin{align*}
Q^*_\lambda(1,1) =& \frac{{R(1,1) + \gamma \lambda \ln (1 + \exp ( - (R(1,1) - R(1,2))/\lambda ))}}{{1 - \gamma }},\\
Q^*_\lambda(1,2)=& Q_\lambda^*(1,1) - \{ R(1,1) - R(1,2)\}.
\end{align*}
Moreover, the induced Boltzmann policy at the stationary point has the following form:
\[\pi _\lambda ^*(1|1) = \frac{1}{{1 + \exp ( - (R(1,1) - R(1,2))/\lambda )}},\qquad \pi _\lambda ^*(2|1) = 1 - \pi _\lambda ^*(1|1).\]
As a concrete example, let us choose $\gamma=0.9, \lambda=1,\ R(1,1)=1, R(1,2)=0$. Then, we obtain
\begin{align*}
Q_\lambda ^*(1,1) =& \frac{{R(1,1) + \gamma \lambda \ln (1 + \exp ( - (R(1,1) - R(1,2))/\lambda ))}}{{1 - \gamma }}\\
& = \frac{{1 + 0.9\log (1 + {e^{ - 1}})}}{{0.1}}\\
& \cong 12.8193\\
Q_\lambda ^*(1,2) =& Q_\lambda ^*(1,1) - (R(1,1) - R(1,2)) \cong 11.8193
\end{align*}
The corresponding Boltzmann policy is computed as
\begin{align*}
\pi _\lambda ^*(1|1) = \frac{1}{{1 + {e^{ - 1}}}} \cong 0.7310,\quad \pi _\lambda ^*(2|1) = \frac{{{e^{ - 1}}}}{{1 + {e^{ - 1}}}} \cong 0.2689.
\end{align*}
For this example, we can apply gradient descent, and the results are illustrated in~\cref{fig:ex-fig5} and~\cref{fig:ex-fig6}. \cref{fig:ex-fig5} shows the evolution of the SCBR objective function value over iterations on a log scale. As can be seen, the objective decreases at a linear rate. \cref{fig:ex-fig6} shows the convergence of $Q_k(1,1)$ and $Q_k(1,2)$, where the dashed lines indicate the optimal Q-values, $Q_\lambda^*(1,1)$ and $Q_\lambda^*(1,2)$. As observed, both $Q_k(1,1)$ and $Q_k(1,2)$ converge well to their optimal Q-values.

\begin{figure}[ht!]
\centering\includegraphics[width=0.7\textwidth, keepaspectratio]{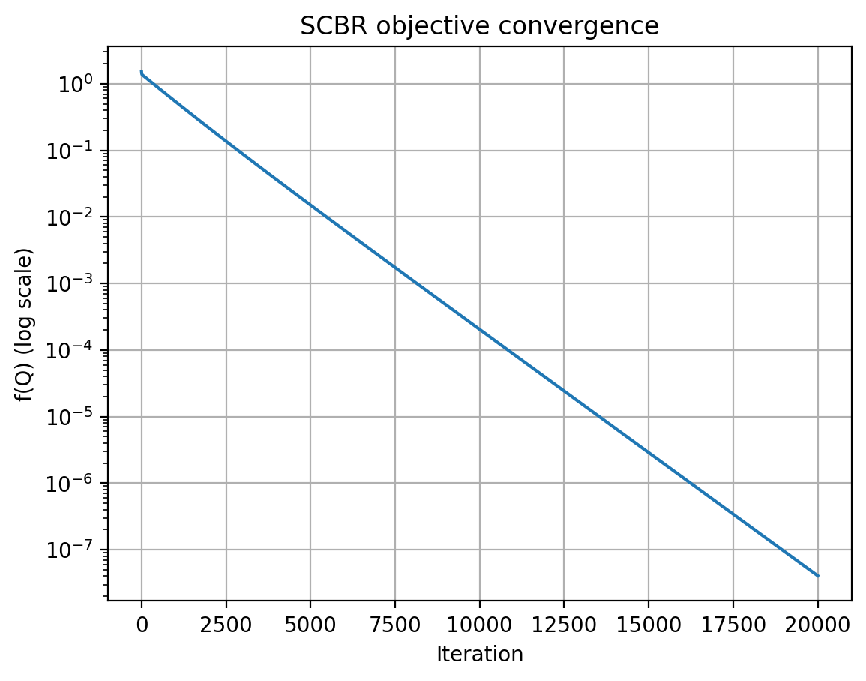}
\caption{The evolution of the SCBR objective function value over iterations on a log scale.}\label{fig:ex-fig5}
\end{figure}

\begin{figure}[ht!]
\centering\includegraphics[width=0.7\textwidth, keepaspectratio]{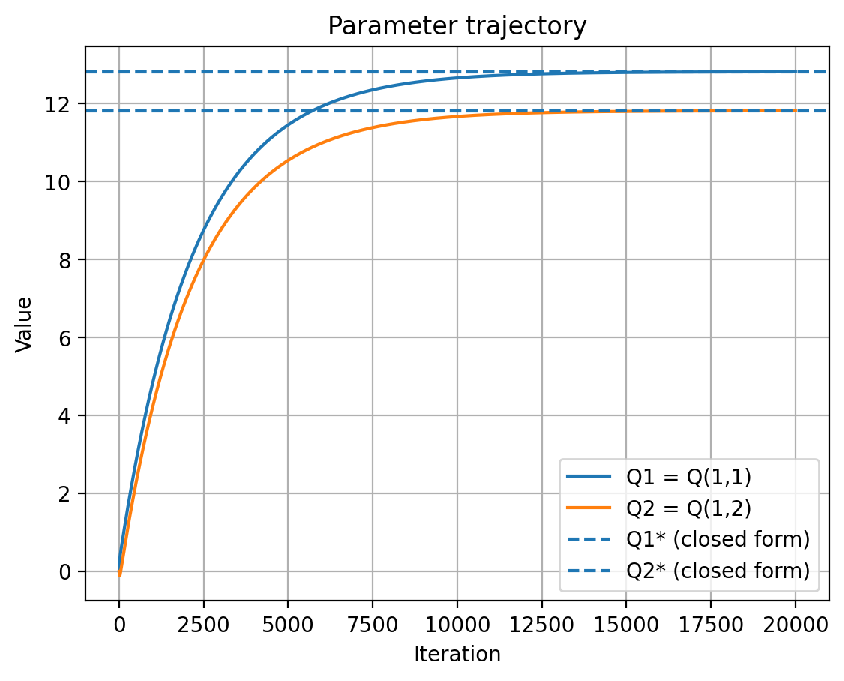}
\caption{The convergence of $Q_k(1,1)$ and $Q_k(1,2)$ to $Q_\lambda^*(1,1)$ and $Q_\lambda^*(1,2)$, respectively, where the dashed lines indicate the optimal Q-values. }\label{fig:ex-fig6}
\end{figure}

\section{SCBR with linear function approximation}

In this section, we derive and summarize several properties of the SCBR objective function under linear function approximation.
Let us consider the SCBR objective function
\[f(\theta ) = \frac{1}{2}\left\| {{F_\lambda }\Phi \theta  - \Phi \theta } \right\|_2^2.\]

\begin{lemma}\label{thm:levelset-2}
The level set ${\cal L}_c = \{ \theta  \in {\mathbb R}^m:f(\theta ) \le c\}$ is bounded and closed (compact).
\end{lemma}
\begin{proof}
First of all, we have
\begin{align*}
{\left\| {{Q_\theta } - {F_\lambda }{Q_\theta }} \right\|_\infty } =& {\left\| {{Q_\theta } - Q_\lambda ^* + {F_\lambda }Q_\lambda ^* - {F_\lambda }{Q_\theta }} \right\|_\infty }\\
\ge& {\left\| {{Q_\theta } - Q_\lambda ^*} \right\|_\infty } - {\left\| {{F_\lambda }Q_\lambda ^* - {F_\lambda }{Q_\theta }} \right\|_\infty }\\
\ge& {\left\| {{Q_\theta } - Q_\lambda ^*} \right\|_\infty } - \gamma {\left\| {Q_\lambda ^* - {Q_\theta }} \right\|_\infty }\\
=& (1 - \gamma ){\left\| {{Q_\theta } - Q_\lambda ^*} \right\|_\infty }\\
\ge& (1 - \gamma )\frac{1}{{\sqrt {|{\cal S}\times {\cal A}|} }}{\left\| {{Q_\theta } - Q_\lambda ^*} \right\|_2},
\end{align*}
where the second line is due to the reverse triangle inequality and the third line is due to the fact that $F_\lambda $ is a contraction. Using the above inequality, we can derive the lower bound
\[f(\theta ) =\frac{1}{2} \left\| {{Q_\theta } - {F_\lambda }{Q_\theta }} \right\|_2^2 \ge \frac{1}{2} \left\| {{Q_\theta } - {F_\lambda }{Q_\theta }} \right\|_\infty ^2 \ge \frac{{{{(1 - \gamma )}^2}}}{{2|{\cal S}\times {\cal A}|}}\left\| {{Q_\theta } - Q_\lambda ^*} \right\|_2^2.\]
Therefore, one can conclude that $f(\theta) \le c$ implies ${\left\| {{Q_\theta } - Q_\lambda ^*} \right\|_2} \le \frac{{\sqrt {2|{\cal S} \times {\cal A}|c} }}{{1 - \gamma }}$.
Hence,
\[
\|Q_\theta\|_2 \le \|Q_\theta-Q_\lambda^*\|_2 + \|Q_\lambda^*\|_2 \le \frac{\sqrt{2|{\cal S}\times {\cal A}|c}}{1-\gamma}+\|Q_\lambda^*\|_2.
\]
Since $\Phi$ has full column rank,
\[
\sigma_{\min}(\Phi)\,\|\theta\|_2 \le \|\Phi\theta\|_2 = \|Q_\theta\|_2,
\]
and therefore
\[
\|\theta\|_2 \le \frac{1}{\sigma_{\min}(\Phi)}\left(\frac{\sqrt{2|{\cal S}\times {\cal A}|c}}{1-\gamma}+\|Q_\lambda^*\|_2\right).
\]
Therefore, ${\cal L}_c$ is bounded. Moreover, since $f$ is continuous, its level set is closed~\citep[Theorem~1.6]{rockafellar1998variational}. This completes the proof.
\end{proof}

\begin{theorem}\label{thm:subdifferential-4}
The gradient of $f$ is given by
\[{\nabla _\theta }f(\theta ) = {\Phi ^\top}{(\gamma P{\Pi ^{{\pi _\theta }}} - I)^\top}({F_\lambda }(\Phi \theta ) - \Phi \theta ),\]
where ${\pi _\theta }(j|i): = \frac{{\exp ({Q_\theta }(i,j)/\lambda )}}{{\sum_{u \in {\cal A}} {\exp ({Q_\theta }(i,u)/\lambda )} }}$ is the Boltzmann policy of $Q_\theta$.
\end{theorem}
\begin{proof}
By direct calculation, we can prove that the gradient is given by
\begin{align*}
{\nabla _\theta }f(\theta ) = \sum\limits_{(s,a) \in {\cal S} \times {\cal A}} {\delta (s,a){\nabla _\theta }[({F_\lambda }{Q_\theta })(s,a) - {Q_\theta }(s,a)]},
\end{align*}
where
\begin{align*}
{\nabla _\theta }[({F_\lambda }{Q_\theta })(s,a) - {Q_\theta }(s,a)] =& \gamma \sum_{s' \in {\cal S}} {P(s'|s,a)\lambda {\nabla _\theta }\ln \left( {\sum_{u \in {\cal A}} {\exp (e_{s',u}^\top\Phi \theta /\lambda )} } \right)}  - {\Phi ^\top}{e_{s,a}}\\
=& \gamma \sum\limits_{s' \in {\cal S}} {P(s'|s,a)\lambda \frac{{{\nabla _\theta }\sum_{u \in {\cal A}} {\exp (e_{s',u}^\top\Phi \theta /\lambda )} }}{{\sum_{u \in {\cal A}} {\exp (e_{s',u}^\top\Phi \theta /\lambda )} }}}  - {\Phi ^\top}{e_{s,a}}\\
=& \gamma \sum_{s' \in {\cal S}} {P(s'|s,a)\sum_{u \in {\cal A}} {{\pi _\theta }(u|s'){\Phi ^\top}{e_{s',u}}} }  - {\Phi ^\top}{e_{s,a}}\\
=& {\Phi ^\top}\left\{ {\gamma \sum_{s' \in {\cal S}} {P(s'|s,a)\sum_{u \in {\cal A}} {{\pi _\theta }(u|s'){e_{s',u}}} }  - {e_{s,a}}} \right\},
\end{align*}
where ${e_{s,a}}: = {e_s} \otimes {e_a} \in {\mathbb R}^{|{\cal S} \times {\cal A}|}$, $e_s\in {\mathbb R}^{|{\cal S}|}$ is the standard basis vector whose $s$-th element is one and whose other elements are zero, $e_a\in {\mathbb R}^{|{\cal A}|}$ is the standard basis vector whose $a$-th element is one and whose other elements are zero. Combining the two results leads to
\begin{align*}
{\nabla _\theta }f(\theta ) =& \sum\limits_{(s,a) \in {\cal S} \times {\cal A}} {{\Phi ^\top}\left\{ {\gamma \sum_{s' \in {\cal S}} {P(s'|s,a)\sum_{u \in {\cal A}} {{\pi _\theta }(u|s'){e_{s',u}}} }  - {e_{s,a}}} \right\}\delta (s,a)} \\
=& {\Phi ^\top}{(\gamma P{\Pi ^{{\pi _\theta }}} - I)^\top}({F_\lambda }(\Phi \theta ) - \Phi \theta ).
\end{align*}
This completes the proof.
\end{proof}

\begin{theorem}\label{thm:stationary-4}
Under Assumption~\ref{ass:oblique-regularity}, the stationary point ${\nabla _\theta }f(\bar \theta ) = 0$ satisfies the oblique projected soft control Bellman equation (OP-SCBE) ${Q_{\bar \theta }} = {\Gamma _{\Phi |\Psi_{\bar \theta} }}{F_\lambda }{Q_{\bar \theta }}$, where $\Psi_\theta  = (\gamma P{\Pi ^{{\pi _{\theta }}}} - I)\Phi$.
\end{theorem}
\begin{proof}
By~\cref{thm:subdifferential-4}, ${\nabla _\theta }f(\bar \theta ) = 0$ implies ${\Phi ^\top}{(\gamma P{\Pi ^{{\pi _{\bar \theta }}}} - I)^\top}({F_\lambda }(\Phi \bar \theta ) - \Phi \bar \theta ) = 0$.
By Assumption~\ref{ass:oblique-regularity}, ${\Phi ^\top}{(\gamma P{\Pi ^{{\pi _{\bar \theta }}}} - I)^\top}\Phi$ is nonsingular. Hence, equivalently, we can write
\[{Q_{\bar \theta }} = \Phi {\left[ {{\Phi ^\top}{{(\gamma P{\Pi ^{{\pi _{\bar \theta }}}} - I)}^\top}\Phi } \right]^{ - 1}}{\Phi ^\top}{(\gamma P{\Pi ^{{\pi _{\bar \theta }}}} - I)^\top}{F_\lambda }{Q_{\bar \theta }} = {\Gamma _{\Phi |{\Psi _{\bar \theta }}}}{F_\lambda }{Q_{\bar \theta }}\]
This completes the proof.
\end{proof}

\begin{lemma}\label{thm:SCBR-LFA:hessian}
The Hessian (under the state-major ordering introduced in Section~3.2) is given by
\[\nabla _\theta ^2f(\theta ) = {\Phi ^\top}{(\gamma P{\Pi ^{{\pi _{{Q_\theta }}}}} - I)^\top}(\gamma P{\Pi ^{{\pi _{{Q_\theta }}}}} - I)\Phi  + \frac{\gamma }{\lambda }{\Phi ^\top}\left[ {\begin{array}{*{20}{c}}
{{\Omega _1}({Q_\theta })}&0& \cdots &0\\
0&{{\Omega _2}({Q_\theta })}& \ddots & \vdots \\
 \vdots & \ddots & \ddots &0\\
0& \cdots &0&{{\Omega _{|{\cal S}|}}({Q_\theta })}
\end{array}} \right]\Phi \]
where
\begin{align*}
{\Omega _s}(Q_\theta): =& m_s\left( {{\rm{diag}}({\pi_{Q_\theta}}( \cdot |s)) - {\pi_{Q_\theta}}( \cdot |s){\pi_{Q_\theta}}{{( \cdot |s)}^\top}} \right),\\
{\rm{diag}}({\pi_{Q_\theta}}( \cdot |s)): =& \left[ {\begin{array}{*{20}{c}}
{{\pi_{Q_\theta}}(1|s)}&0& \cdots &0\\
0&{{\pi_{Q_\theta}}(2|s)}& \ddots & \vdots \\
 \vdots & \ddots & \ddots &0\\
0& \cdots &0&{{\pi_{Q_\theta}}(|{\cal A}|\mid s)}
\end{array}} \right],\\
m_i: =& \sum_{(s,a) \in {\cal S} \times {\cal A}} {\delta (s,a)P(i|s,a)},\\
\delta (s,a): =& ({F_\lambda }Q_\theta)(s,a) - Q_\theta(s,a),
\end{align*}
\end{lemma}
\begin{proof}
The proof can be completed by direct calculations.
\end{proof}

\begin{lemma}\label{thm:SCBR-LFA:coercive}
$f$ is bounded below and coercive.
\end{lemma}
\begin{proof}
The proof can be completed by following similar lines as in the CBR case.
\end{proof}

\begin{lemma}\label{thm:SCBR-LFA:stationary1}
A stationary point of $f$ always exists.
\end{lemma}
\begin{proof}
Since $f$ is bounded below and continuous by~\cref{thm:SCBR-LFA:coercive}, a minimizer always exists.
Then, the minimizer is a stationary point. This completes the proof.
\end{proof}

\begin{lemma}\label{thm:SCBR-LFA:smooth1}
$f$ is locally $L$-smooth, i.e., for any compact $C$, there exists some $L(C)>0$ such that
\[{\left\| {{\nabla_\theta}f({\theta_1}) - {\nabla _\theta}f({\theta_2})} \right\|_2} \le L(C){\left\| {{\theta_1} - {\theta_2}} \right\|_2}\]
for all ${\theta_1},{\theta_2} \in C$.
\end{lemma}
\begin{proof}
The proof can be readily completed by following similar lines as in the proof of~\cref{thm:SCBR-tabular:L-smooth-1} with the Hessian given in~\cref{thm:SCBR-LFA:hessian}.
\end{proof}

\begin{lemma}\label{thm:CBR-LFA:property2}
Let us define
\[
H(\theta ): = {\Phi ^\top}{(\gamma P{\Pi ^{{\pi _{{Q_\theta }}}}} - I)^\top}(\gamma P{\Pi ^{{\pi _{{Q_\theta }}}}} - I)\Phi,
\qquad
\underline{\lambda}_H:=\inf_{\theta\in\mathbb R^m}\lambda_{\min}(H(\theta)).
\]
Under full column rank of $\Phi$, we have $\underline{\lambda}_H>0$. If there exists
\begin{align}
0< c \le \frac{1}{2}{\left( \frac{\underline{\lambda}_H}{4\frac{\gamma }{\lambda }\left\| \Phi  \right\|_2^2{{\left\| {{P^\top}} \right\|}_\infty }} \right)^2}\label{eq:4}
\end{align}
such that the level set
\[{\cal L}_c: = \left\{ {\theta  \in {\mathbb R}^m:f(\theta ) = \frac{1}{2}\left\| {{F_\lambda }{Q_\theta } - {Q_\theta }} \right\|_2^2 \le c} \right\}\]
is nonempty, then $f$ is strongly convex in ${\cal L}_c$.
\end{lemma}
\begin{proof}
Let $A(\theta):=\gamma P{\Pi ^{{\pi _{{Q_\theta }}}}} - I$. By~\cref{thm:matrix_inversion2}, $A(\theta)$ is nonsingular for every $\theta$. Moreover,
\[
\|A(\theta)^{-1}\|_\infty = \|(I-\gamma P\Pi^{\pi_{Q_\theta}})^{-1}\|_\infty \le \frac{1}{1-\gamma},
\]
so
\[
\sigma_{\min}(A(\theta))\ge \frac{1}{\|A(\theta)^{-1}\|_2}\ge \frac{1}{\sqrt{|{\cal S}\times {\cal A}|}\,\|A(\theta)^{-1}\|_\infty}\ge \frac{1-\gamma}{\sqrt{|{\cal S}\times {\cal A}|}}.
\]
Since $H(\theta)=(A(\theta)\Phi)^\top(A(\theta)\Phi)$ and $\Phi$ has full column rank,
\[
\lambda_{\min}(H(\theta))=\sigma_{\min}(A(\theta)\Phi)^2
\ge \sigma_{\min}(A(\theta))^2\,\lambda_{\min}(\Phi^\top\Phi)
\ge \frac{(1-\gamma)^2}{|{\cal S}\times {\cal A}|}\lambda_{\min}(\Phi^\top\Phi)>0.
\]
Hence $\underline{\lambda}_H>0$.

Now write the Hessian in~\cref{thm:SCBR-LFA:hessian} as
\[\nabla _\theta ^2f(\theta ) = H(\theta ) + E(\theta )\]
with
\[E(\theta ): = \frac{\gamma }{\lambda }{\Phi ^\top}\left[ {\begin{array}{*{20}{c}}
{{\Omega _1}({Q_\theta })}&0& \cdots &0\\
0&{{\Omega _2}({Q_\theta })}& \ddots & \vdots \\
 \vdots & \ddots & \ddots &0\\
0& \cdots &0&{{\Omega _{|{\cal S}|}}({Q_\theta })}
\end{array}} \right]\Phi .\]
Using~\cref{thm:SCBR-tabular:property1}, we can obtain the following bounds:
\begin{align*}
{\left\| {E(\theta )} \right\|_2} \le& \frac{\gamma }{\lambda }\left\| \Phi  \right\|_2^2 2{\max _{s \in {\cal S}}}|m_s|\\
=& \frac{\gamma }{\lambda }\left\| \Phi  \right\|_2^22{\left\| {{P^\top}({F_\lambda }{Q_\theta } - {Q_\theta })} \right\|_\infty }\\
\le& \frac{\gamma }{\lambda }\left\| \Phi  \right\|_2^22{\left\| {{P^\top}} \right\|_\infty }{\left\| {{F_\lambda }{Q_\theta } - {Q_\theta }} \right\|_2}\\
\le& \frac{\gamma }{\lambda }\left\| \Phi  \right\|_2^22{\left\| {{P^\top}} \right\|_\infty }\sqrt {2c}.
\end{align*}
Noting that $E(\theta )$ is symmetric, we have
\[{\left\| {E(\theta )} \right\|_2} \le \frac{\gamma }{\lambda }\left\| \Phi  \right\|_2^22{\left\| {{P^\top}} \right\|_\infty }\sqrt {2c}. \]
Therefore,
\begin{align*}
\nabla _\theta ^2f(\theta )
=& H(\theta ) + E(\theta )\\
\succeq& \underline{\lambda}_H I - \left\|E(\theta)\right\|_2 I\\
\succeq & \underline{\lambda}_H I - \left\{ {\frac{\gamma }{\lambda }\left\| \Phi  \right\|_2^22{{\left\| P^\top \right\|}_\infty }\sqrt {2c} } \right\}I.
\end{align*}
If $c >0$ is chosen such that
\begin{align}
\frac{\gamma }{\lambda }\left\| \Phi  \right\|_2^22{\left\| {{P^\top}} \right\|_\infty }\sqrt {2c}  \le \frac{1}{2}\underline{\lambda}_H,\label{eq:3}
\end{align}
then
\[\nabla _\theta ^2f(\theta ) \succeq \frac{1}{2}\underline{\lambda}_H I \succ 0.\]
Therefore, $f$ is strongly convex in ${\cal L}_c$. The condition in~\eqref{eq:3} can be equivalently written as~\eqref{eq:4}. This completes the proof.
\end{proof}

\begin{lemma}
Under Assumption~\ref{ass:oblique-regularity}, ${\cal R}(\Phi)$ and ${\cal R}(\Psi_\theta)$ with $\Psi_\theta  = (\gamma P{\Pi ^{{\pi _\theta }}} - I)\Phi $ are not orthogonal at any $\theta$ for which the assumption holds.
\end{lemma}
\begin{proof}
Note that ${\cal R}(\Phi ) = \{ \Phi x: x \in {\mathbb R}^m\}$ and ${\cal R}((\gamma P{\Pi ^{{\pi _\theta }}} - I)\Phi ) = \{ (\gamma P{\Pi ^{{\pi _\theta }}} - I)\Phi x:x \in {\mathbb R}^m\}$ for a fixed $\theta \in {\mathbb R}^m$. For the two subspaces to be orthogonal, we need to prove
\[\left\langle {u,v} \right\rangle  = {u^\top}v = 0,\quad \forall u \in {\cal R}(\Phi ),\quad \forall v \in {\cal R}((\gamma P{\Pi ^{{\pi _\theta }}} - I)\Phi ),\]
which is equivalent to
\[{x^\top}{\Phi ^\top}(\gamma P{\Pi ^{{\pi _\theta }}} - I)\Phi y = 0,\quad \forall x,y \in {\mathbb R}^m.\]
However, under Assumption~\ref{ass:oblique-regularity}, ${\Phi ^\top}(\gamma P{\Pi ^{{\pi _\theta }}} - I)\Phi$ is nonsingular because it is the transpose of $\Psi_\theta^\top\Phi$. Hence it cannot be the zero matrix, and the two subspaces cannot be orthogonal. This completes the proof.
\end{proof}

\begin{theorem}\label{thm:SCBR-LFA:convergence}
Let us consider the gradient descent iterates, ${\theta _{k + 1}} = {\theta _k} - \alpha {\nabla _\theta }f({\theta _k})$, for $k=0,1,\ldots$ with any initial point $\theta_0 \in {\mathbb R}^m$. Fix $c>f(\theta_0)$ and set ${\cal L}_c:=\{\theta:f(\theta)\le c\}$. For the chosen step size $\alpha$, define
\[
D_\alpha:=\operatorname{conv}\left({\cal L}_c\cup\{\theta-\alpha\nabla_\theta f(\theta):\theta\in{\cal L}_c\}\right),
\]
and let $L_\alpha>0$ be a Lipschitz constant of $\nabla_\theta f$ on $D_\alpha$. If $0<\alpha < 2/L_\alpha$, the iterates satisfy $\lim_{k \to \infty } {\left\| {{\nabla _\theta }f({\theta _k})} \right\|_2} = 0$ and
\[{\min _{0 \le i \le N}}\left\| {{\nabla _\theta }f({\theta _i})} \right\|_2^2 \le \frac{{f({\theta _0}) - {{\min }_{\theta  \in {\mathbb R}^m}}f(\theta )}}{{(N + 1)\alpha \left( {1 - \frac{{\alpha {L_\alpha}}}{2}} \right)}}.\]
\end{theorem}
\begin{proof}
For any $\theta\in{\cal L}_c$, the point $\theta^+ := \theta-\alpha\nabla_\theta f(\theta)$ and the line segment from $\theta$ to $\theta^+$ are contained in $D_\alpha$. The descent lemma with $L_\alpha$ gives
\[
f(\theta^+)\le f(\theta)-\alpha\left(1-\frac{\alpha L_\alpha}{2}\right)\|\nabla_\theta f(\theta)\|_2^2.
\]
Thus all iterates remain in ${\cal L}_c$. Summing the resulting inequality over $k=0,\ldots,N$ yields
\[
\alpha\left(1-\frac{\alpha L_\alpha}{2}\right)\sum_{k=0}^N\|\nabla_\theta f(\theta_k)\|_2^2
\le f(\theta_0)-f(\theta_{N+1})
\le f(\theta_0)-\min_{\theta\in\mathbb R^m}f(\theta).
\]
Letting $N\to\infty$ gives $\|\nabla_\theta f(\theta_k)\|_2\to0$, and taking the minimum term in the sum gives the displayed complexity bound.
\end{proof}

\begin{theorem}\label{thm:SCBR-LFA:bound1}
Let us define
\[
H(\theta ): = {\Phi ^\top}{(\gamma P{\Pi ^{{\pi _{{Q_\theta }}}}} - I)^\top}(\gamma P{\Pi ^{{\pi _{{Q_\theta }}}}} - I)\Phi,
\qquad
\underline{\lambda}_H:=\inf_{\theta\in\mathbb R^m}\lambda_{\min}(H(\theta)).
\]
If there exists
\[0 < \varepsilon  \le \frac{\lambda }{{4\sqrt {|{\cal S} \times {\cal A}|} \gamma (1 + \gamma )}}\frac{\underline{\lambda}_H}{{\left\| \Phi  \right\|_2^2{{\left\| {{P^\top}} \right\|}_\infty }}}\]
such that ${\cal L}_{c_0}$ is nonempty with $c_0 = \frac{{{{(1 - \gamma )}^2}}}{2}{\varepsilon ^2}$, then the following set
\[K:=\left\{ {\theta  \in {\mathbb R}^m:{{\left\| {Q_\theta - Q_\lambda ^*} \right\|}_\infty } \le \varepsilon } \right\}\]
is nonempty, $f$ is strongly convex in $K$, and the global minimizer of $f$ is unique and lies in $K$.
\end{theorem}
\begin{proof}
For any $\theta \in K$, using the contraction of $F_\lambda$,
\[f(\theta ) = \frac{1}{2}\left\| {F_\lambda }{Q_\theta } - Q_\theta \right\|_2^2 \le \frac{1}{2}|{\cal S} \times {\cal A}|{(1 + \gamma )^2}{\varepsilon ^2}.
\]
Thus $K$ is contained in the level set with radius $c_1=\frac12|{\cal S}\times{\cal A}|(1+\gamma)^2\varepsilon^2$. By~\cref{thm:CBR-LFA:property2}, the displayed upper bound on $\varepsilon$ guarantees strong convexity on that level set, and hence on the convex set $K$.
Next, the residual bound for the contraction $F_\lambda$ gives
\[(1-\gamma){\left\| {{Q_\theta } - Q_\lambda ^*} \right\|_\infty } \le {\left\| {{F_\lambda }{Q_\theta } - {Q_\theta }} \right\|_\infty } \le {\left\| {{F_\lambda }{Q_\theta } - {Q_\theta }} \right\|_2}.
\]
Therefore ${\cal L}_{c_0}\subseteq K$. Since ${\cal L}_{c_0}$ is nonempty, $K$ is nonempty. Moreover, if $\theta_g$ is any global minimizer of $f$, then $f(\theta_g)\le c_0$, so $\theta_g\in K$. Strong convexity on $K$ implies that such a global minimizer is unique. Since a global minimizer is an unconstrained minimizer, it is a stationary point. This completes the proof.
\end{proof}

\subsection{Examples}
In this example, we consider a simple Markov decision process with a single state and two actions: ${\cal S} = \{1\}$, ${\cal A} = \{1,2\}$, $\gamma\in(0,1)$ and $\lambda>0$. The linear function approximator for the soft $Q$-function is given as
\[
Q_\theta :=
\begin{bmatrix}
Q_\theta(1,1)\\
Q_\theta(1,2)
\end{bmatrix}
=
\begin{bmatrix}
\theta\\
-\theta
\end{bmatrix}
= \Phi \theta,
\qquad
\Phi :=
\begin{bmatrix}
1\\[2pt]
-1
\end{bmatrix} \in \mathbb{R}^{2\times 1},\quad \theta\in\mathbb{R}.
\]
The model parameters are
\[
R :=
\begin{bmatrix}
-1\\[2pt]
-1
\end{bmatrix},\quad P =
\begin{bmatrix}
1\\
1
\end{bmatrix} \in \mathbb{R}^{2\times 1}.
\]
The Boltzmann policy associated with $Q_\theta$ at state $s=1$ is the vector
\[{\pi _\theta }( \cdot |1): = \left[ {\begin{array}{*{20}{c}}
{{\pi _\theta }(1|1)}\\
{{\pi _\theta }(2|1)}
\end{array}} \right] = \frac{1}{{\exp ({Q_\theta }(1,1)/\lambda ) + \exp ({Q_\theta }(1,2)/\lambda )}}\left[ {\begin{array}{*{20}{c}}
{\exp ({Q_\theta }(1,1)/\lambda )}\\
{\exp ({Q_\theta }(1,2)/\lambda )}
\end{array}} \right]\]
and the corresponding policy matrix mapping the state to state-action pairs is
\[{\Pi ^{{\pi _\theta }}} = \left[ {\begin{array}{*{20}{c}}
{{\pi _\theta }(1|1)}&{{\pi _\theta }(2|1)}
\end{array}} \right] \in {\mathbb R}^{1 \times 2}.\]

The soft Bellman operator acting on $Q_\theta$ can then be written in vector form as
\[{F_\lambda }({Q_\theta }) = R + \gamma P\lambda \ln \left( {\exp \left( {\frac{{{Q_\theta }(1,1)}}{\lambda }} \right) + \exp \left( {\frac{{{Q_\theta }(1,2)}}{\lambda }} \right)} \right) \in {\mathbb R}^2.\]
For convenience, let us define the following notation:
\[V(\theta ): = \lambda \ln \left( {\exp \left( {\frac{{{Q_\theta }(1,1)}}{\lambda }} \right) + \exp \left( {\frac{{{Q_\theta }(1,2)}}{\lambda }} \right)} \right) = \lambda \ln \left( {\exp \left( {\frac{\theta }{\lambda }} \right) + \exp \left( {\frac{{ - \theta }}{\lambda }} \right)} \right).\]
Moreover, let
\[
\delta(\theta)
:= F_\lambda(\Phi\theta) - \Phi\theta
\in \mathbb{R}^2
\]
denote the residual vector. Then, the SCBR objective function can be written as
\[
f(\theta) = \frac{1}{2}\,\delta(\theta)^\top \delta(\theta).
\]

Following~\cref{thm:subdifferential-4}, the gradient of $f$ with respect to $\theta$ can be written in matrix-vector form as
\begin{align*}
{\nabla _\theta }f(\theta ) =& {\Phi ^\top}{(\gamma P{\Pi ^{{\pi _\theta }}} - I)^\top}({F_\lambda }(\Phi \theta ) - \Phi \theta )\\
=& \left[ {\begin{array}{*{20}{c}}
1&{ - 1}
\end{array}} \right]\left[ {\begin{array}{*{20}{c}}
{\gamma {\pi _\theta }(1|1) - 1}&{\gamma {\pi _\theta }(1|1)}\\
{\gamma {\pi _\theta }(2|1)}&{\gamma {\pi _\theta }(2|1) - 1}
\end{array}} \right]\left[ {\begin{array}{*{20}{c}}
{\gamma V(\theta ) - 1 - \theta }\\
{\gamma V(\theta ) - 1 + \theta }
\end{array}} \right]\\
=& \left[ {\begin{array}{*{20}{c}}
{\gamma \tanh (\theta /\lambda ) - 1}&{\gamma \tanh (\theta /\lambda ) + 1}
\end{array}} \right]\left[ {\begin{array}{*{20}{c}}
{\gamma V(\theta ) - 1 - \theta }\\
{\gamma V(\theta ) - 1 + \theta }
\end{array}} \right]\\
=& 2\gamma \tanh (\theta /\lambda )(\gamma V(\theta ) - 1) + 2\theta.
\end{align*}

The stationary points should satisfy
\[{\nabla _\theta }f(\bar \theta ) = 2\gamma \tanh (\bar \theta /\lambda )(\gamma V(\bar \theta ) - 1) + 2\bar \theta  = 0,\]
where $\bar \theta =0$ is obviously a solution.
In this MDP, due to the symmetry of the linear features and rewards, this condition admits multiple distinct solutions $\theta$ (in particular, one at $0$ and one positive solution together with its negative counterpart).
In particular, we can show that
\[{\nabla _\theta }f(\theta ) = 2\gamma \tanh (\theta /\lambda )(\gamma V(\theta ) - 1) + 2\theta  = :g(\theta )\]
is an odd function $g( - \theta ) =  - g(\theta )$. Therefore, if $g(\theta)=0$ admits a nonzero solution $\theta^*$, then $-\theta^*$ should also be a solution.
For the numerical values used in the figures below, $\gamma=0.9$ and $\lambda=0.5$. With these values, a fixed-point/root-finding computation gives the approximate solutions $\theta^* = 0, 0.2897, -0.2897$. Therefore, there are multiple stationary points.
Substituting the numerical values into the oblique projected soft control Bellman equation, we can easily verify that they are solutions of the oblique projected soft control Bellman equation.

Finally, \cref{fig:ex-fig1,fig:ex-fig2,fig:ex-fig3} illustrates ${\cal R}(\Phi)$, ${F_\lambda }(\Phi \bar \theta )$, ${\cal R}({\Psi _{\bar \beta }})$, and the corresponding projected point for the three stationary points obtained in this way. We can see that these figures are in exact agreement with all the results analyzed in this paper.
\begin{figure}[H]
\centering\includegraphics[width=0.7\textwidth, keepaspectratio]{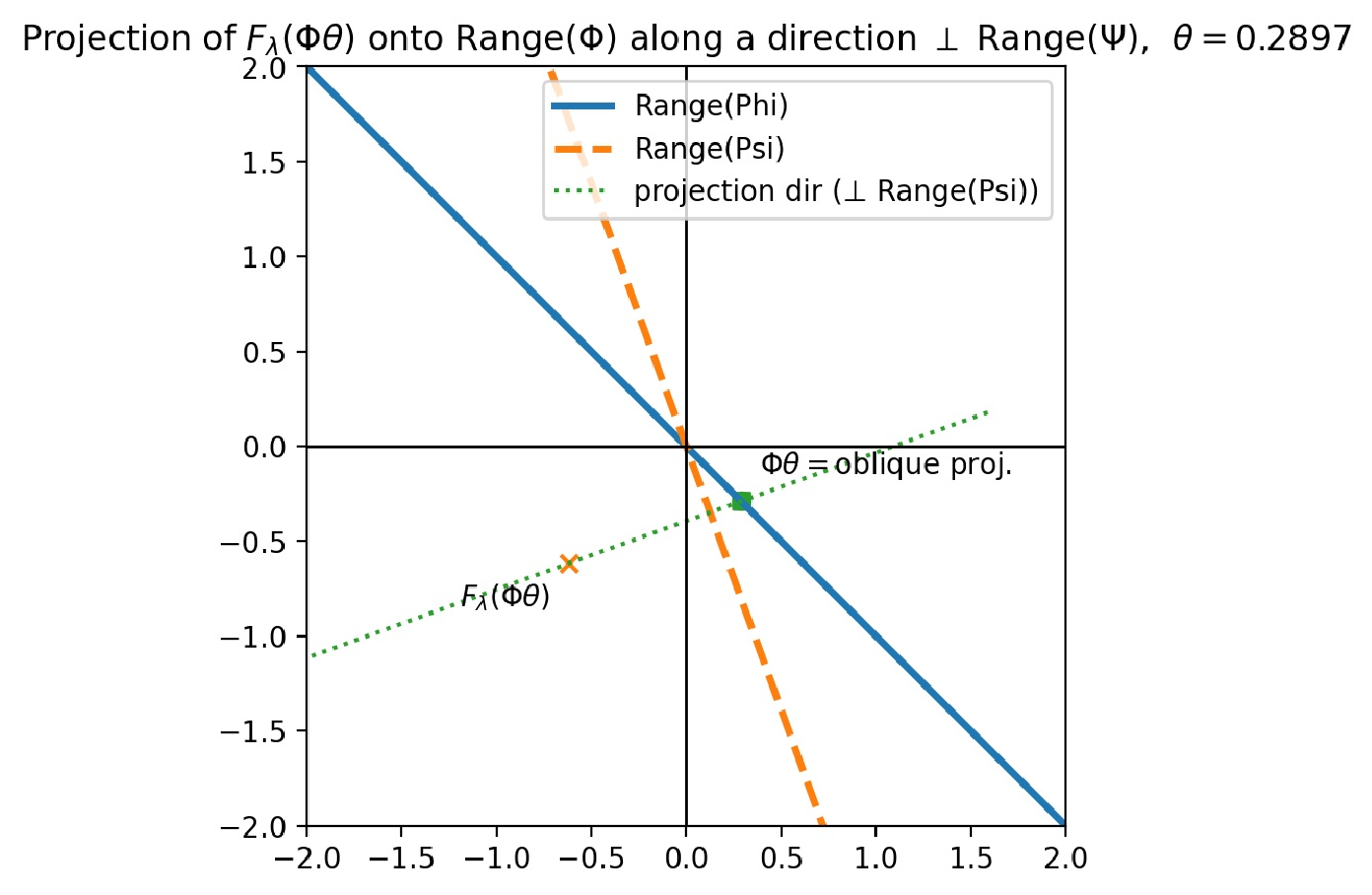}
\caption{This figure illustrates ${\cal R}(\Phi)$, ${F_\lambda }(\Phi \bar \theta )$, ${\cal R}({\Psi _{\bar \beta }})$, and the corresponding projected point for the stationary point $\bar \theta=0.2897$}\label{fig:ex-fig1}
\end{figure}
\begin{figure}[H]
\centering\includegraphics[width=0.7\textwidth, keepaspectratio]{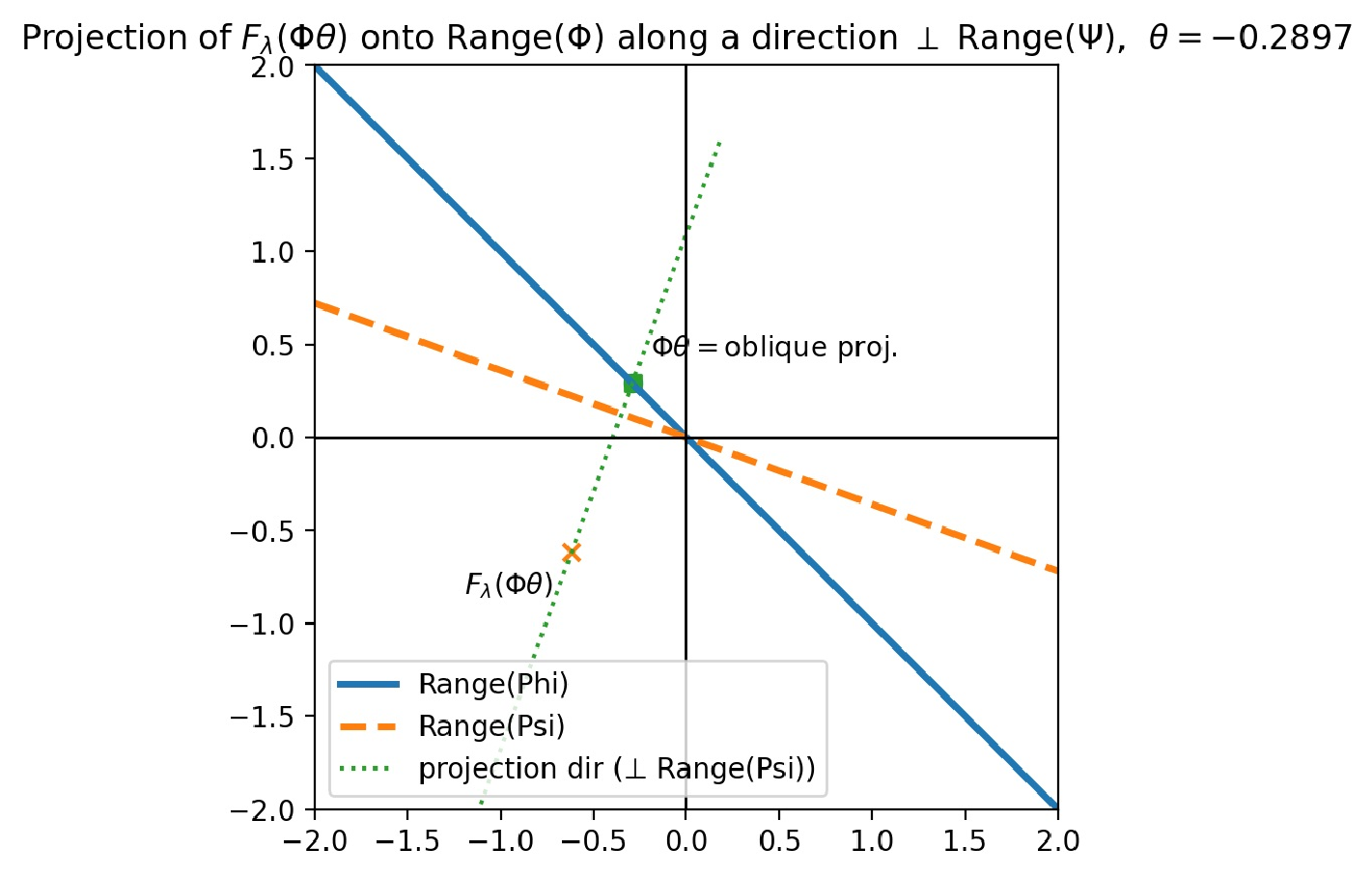}
\caption{This figure illustrates ${\cal R}(\Phi)$, ${F_\lambda }(\Phi \bar \theta )$, ${\cal R}({\Psi _{\bar \beta }})$, and the corresponding projected point for the stationary point $\bar \theta=-0.2897$}\label{fig:ex-fig2}
\end{figure}
\begin{figure}[H]
\centering\includegraphics[width=0.7\textwidth, keepaspectratio]{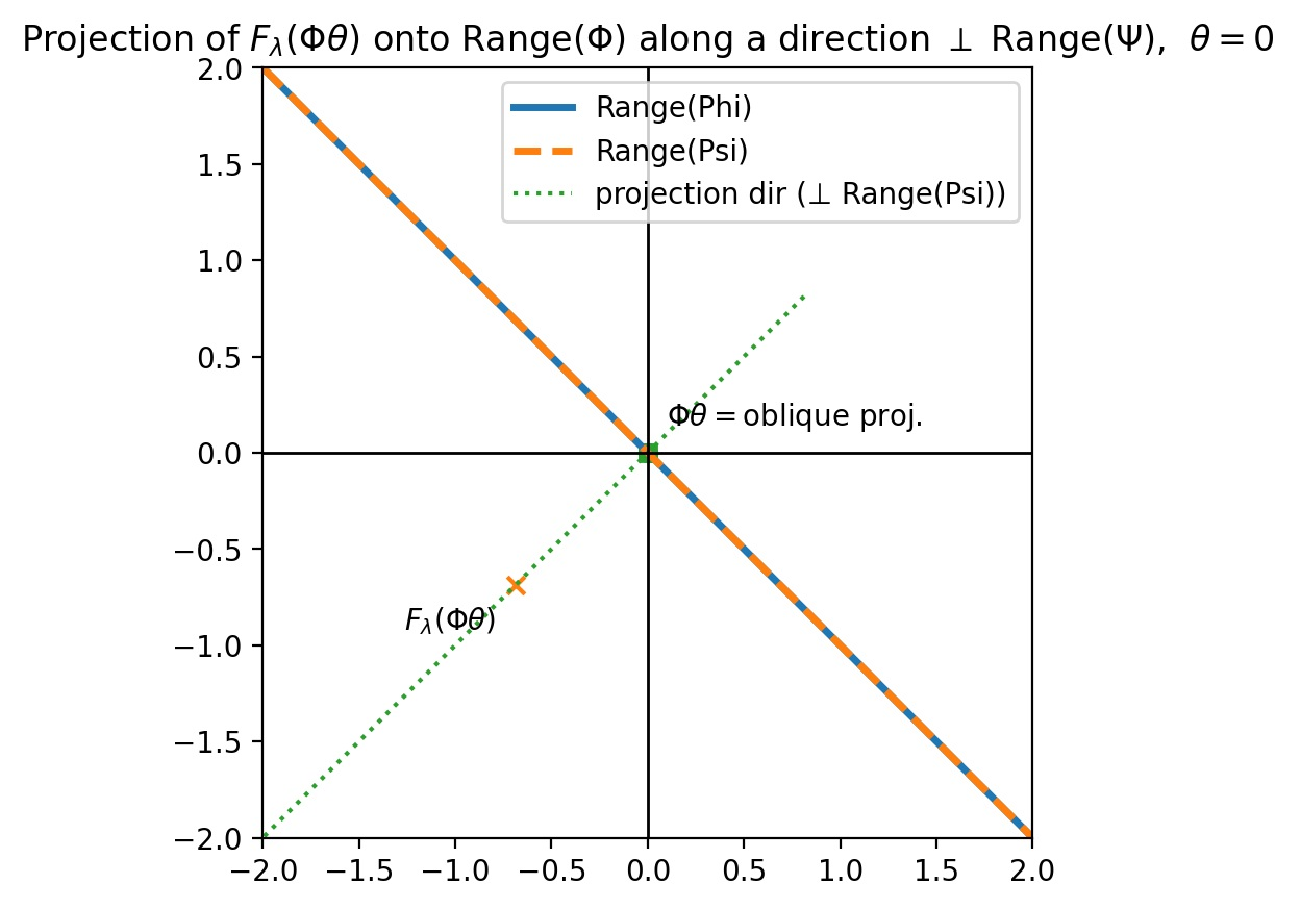}
\caption{This figure illustrates ${\cal R}(\Phi)$, ${F_\lambda }(\Phi \bar \theta )$, ${\cal R}({\Psi _{\bar \beta }})$, and the corresponding projected point for the stationary point $\bar \theta=0$}\label{fig:ex-fig3}
\end{figure}

The first plot of~\cref{fig:ex-fig4} shows the evolution of the objective function value $f$ when gradient descent is applied to the objective function in this example. We observe that the objective value decreases and converges to a constant. The second plot of~\cref{fig:ex-fig4} illustrates the convergence of $\theta_k$ under gradient descent. We can see that $\theta_k$ converges to $0.2897$, which is one of the stationary points.
\begin{figure}[H]
\centering\includegraphics[width=0.7\textwidth, keepaspectratio]{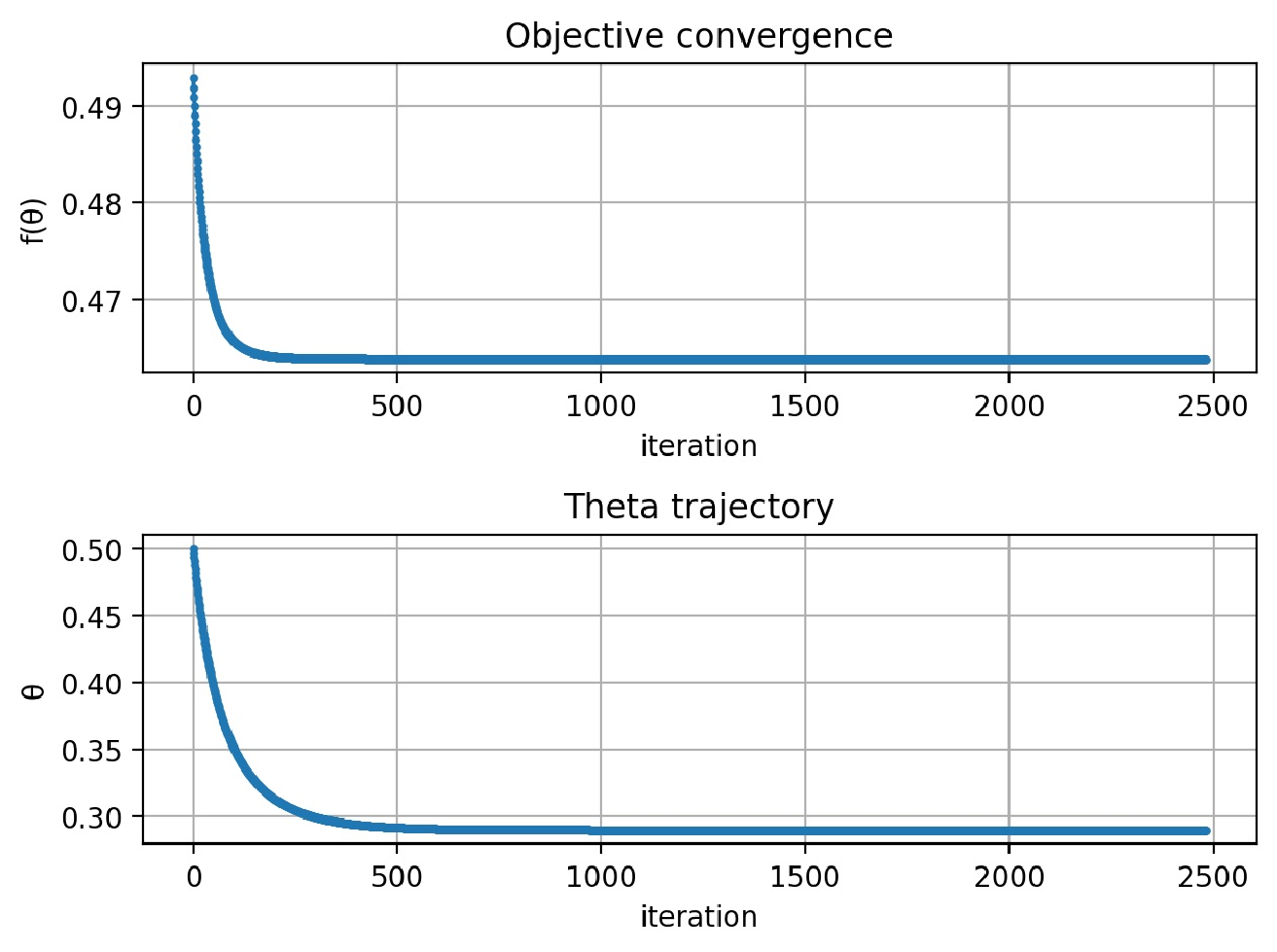}
\caption{(1) The evolution of the objective function value $f$ when gradient descent is applied to the objective function; (2) The convergence of $\theta_k$ under gradient descent}\label{fig:ex-fig4}
\end{figure}




\begin{example}\label{ex:SCBR-FrozenLake}
In this example, experiments on the~\texttt{Gym 8x8 FrozenLake-v1} environment are conducted with $|{\cal S} \times {\cal A}| = 256$, where we use a random feature matrix with $m=120$.
We compared the gradient descent on the SCBR objective with the projected value iteration (P-VI), $\Phi {\theta _{k + 1}} = {\Gamma _{\Phi |\Phi }}T(\Phi {\theta _k})$ for $k=0,1,\ldots$. Note that unless the composite operator $\Gamma _{\Phi |\Phi } T$ is a contraction, the iterates $\theta_k$ of the P-VI may not converge. In contrast, the gradient descent method applied to the SCBR objective function ensures convergence to an approximate solution.
The comparative results are given in~\cref{exp-fig1}, which compares the $\|Q^* - Q_k\|_\infty$ error of the proposed SCBR against the baseline P-VI.
\begin{figure}[H]
\centering\includegraphics[width=0.6\textwidth, keepaspectratio]{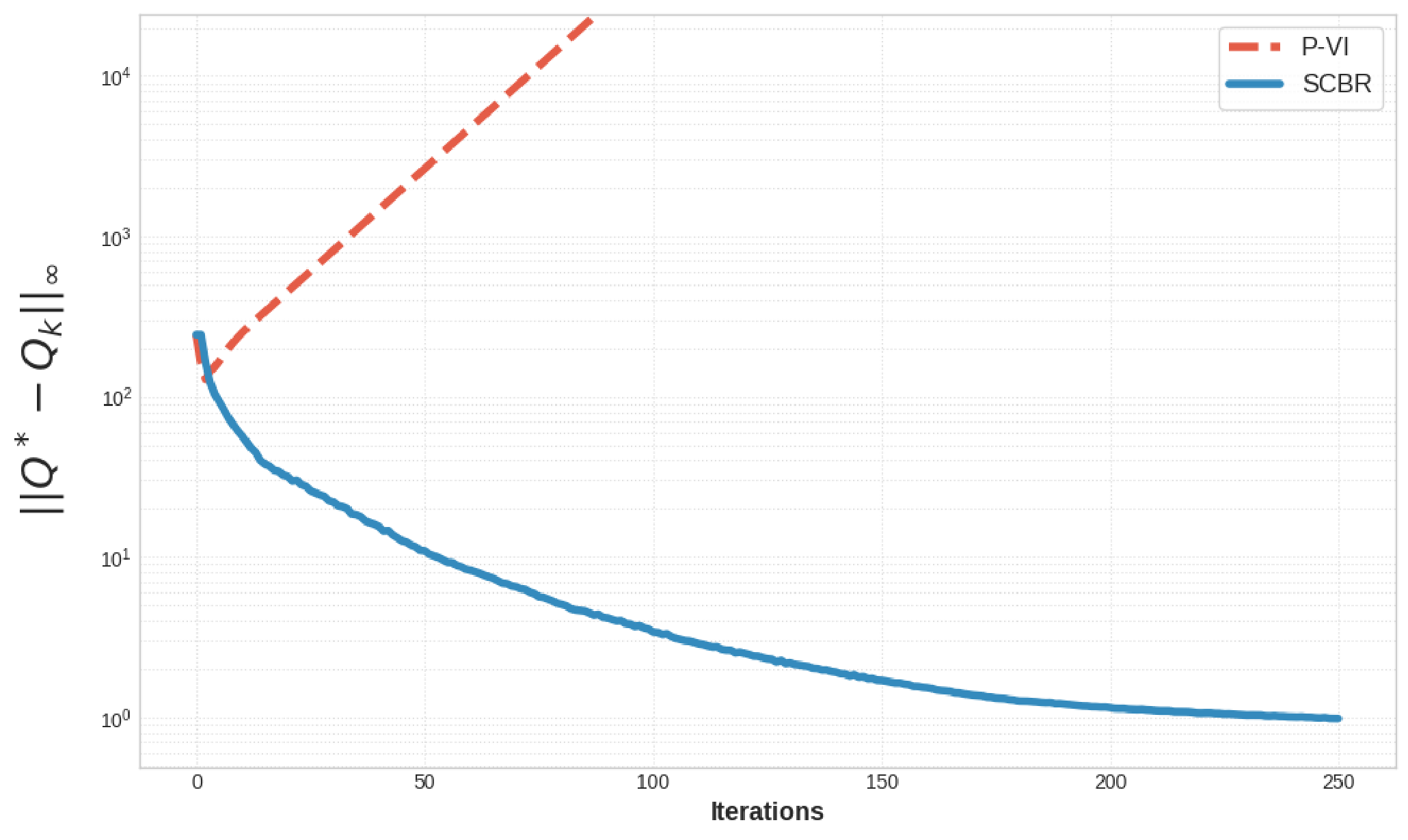}
\caption{Comparison of errors between the gradient descent on the SCBR (blue) and the P-VI (red). While the P-VI diverges, the gradient approach exhibits a steady decrease in error throughout the iterations.}\label{exp-fig1}
\end{figure}
As illustrated in~\cref{exp-fig1}, the P-VI method exhibits divergence due to the noncontractive nature of the combined operators. We conducted experiments using randomly generated feature matrices. Overall, P-VI exhibited unstable behavior, whereas gradient-based methods converged.
Following the convergence analysis, we evaluated the quality of the policies derived from the trained $Q$-functions. We extracted a greedy or Boltzmann policy from each approach and conducted 2,000 simulation episodes, each with a maximum limit of 100 steps. The agent's starting position was randomly sampled from the \texttt{8x8 FrozenLake-v1} grid, excluding hole (H) and goal (G) states. In these simulations, the policy derived from the SCBR method achieved an average success rate of $20.7\%$, whereas the baseline P-VI failed to succeed, reporting a $0.0\%$ success rate.
\end{example}

\section{Deep RL version}\label{sec:app:deep-rl}
In this section, we briefly discuss a deep reinforcement learning (RL) variant of the SCBR method. In this paper, we studied gradient-based SCBR methods. Although we treated both the tabular and linear function approximation settings in a model-based manner, these results can be extended to deep RL. While the theoretical guarantees developed for the tabular and linear approximation cases do not directly carry over, they provide useful insights that guide the extension to deep RL. Because a precise theoretical analysis of deep RL with nonlinear function approximation is generally very challenging, we instead present empirical results to investigate the practical viability of residual-based methods.

We begin by considering a loss function similar to the one used in the deep Q-network (DQN) in~\citet{mnih2015human}, as given below.
\[L(\theta ): = \frac{1}{2}\sum\limits_{(s,a,r,s') \in B} {\delta {{(s,a,r,s';\theta )}^2}} \]
where $\delta (s,a,r,s';\theta ): = r + \gamma {\max _{a' \in {\cal A}}}{Q_\theta }(s',a') - {Q_\theta }(s,a)$ is the TD-error, $Q_\theta$ is a deep neural network function approximator, $B$ is the mini-batch sampled uniformly from the replay buffer, and $\theta$ is its parameter. Unlike in DQN, the target network parameters here are not treated as fixed constants; instead, they can be updated jointly via gradient-based optimization.
If we apply gradient descent directly to this loss function, the resulting gradient becomes biased due to the double-sampling issue~\citep{baird1995residual,bertsekas1996neuro}, which can in turn degrade learning performance. Therefore, we propose the following loss function:
\begin{align*}
L(\theta ): = \frac{1}{2}\sum\limits_{(s,a,r,\bar r,s',\bar s') \in B} {{\rm{no\_grad}}[\delta (s,a,\bar r,\bar s';\theta )]}\times \delta (s,a,r,s';\theta )
\end{align*}
where $\rm{no\_grad}$ denotes a term that is treated as a constant with respect to differentiation, i.e., it does not propagate gradients,
\begin{align*}
\delta (s,a,r,s';\theta ):= r + \gamma \lambda \ln \left( {\sum\limits_{a' \in {\cal A}} {\exp \left\{ {\frac{{{Q_\theta }(s',a')}}{\lambda }} \right\}} } \right) - {Q_\theta }(s,a),
\end{align*}
where $s' \sim P(\cdot|s,a)$ is the next state sample, $\bar s' \sim P(\cdot|s,a)$ is an independent next state sample, $\bar r$ denotes the reward corresponding to $(s,a,\bar s')$.
To implement the above loss function, we require two independent next-state samples (and the corresponding two reward samples) from the current state–action pair. This requirement is known as the double-sampling issue, and it makes online implementation in real-world settings difficult. However, as is common in reinforcement learning, it can be implemented in simulation environments where such sampling can be obtained, or in deterministic environments.

\begin{figure}[H]
    \centering
    \includegraphics[width=1.0\textwidth, keepaspectratio]{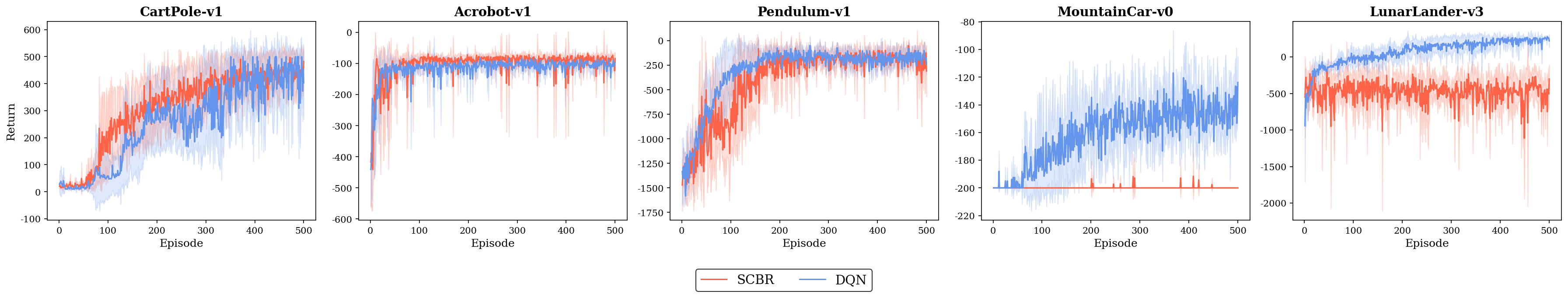}
    \caption{Deep learning experiments across five different tasks from the \texttt{Gym} control environment, comparing SCBR (red) and DQN (blue). While SCBR outperforms DQN in certain environments, such as \texttt{CartPole-v1} and \texttt{Acrobot-v1}, it exhibits suboptimal performance in others.}
    \label{fig:deepRL}
\end{figure}

\cref{fig:deepRL} compares the deep RL variant of the SCBR method with the baseline DQN. We conducted each experiment across five different random seeds. As mentioned above, the SCBR method does not use a target network, whereas DQN does. In~\texttt{CartPole-v1}, the SCBR method clearly outperforms DQN. Moreover, its performance curve rises steadily, while the DQN learning curve exhibits extreme fluctuations. In~\texttt{Acrobot-v1}, both methods perform similarly, while the SCBR method shows superior performance. However, in other environments such as~\texttt{Pendulum-v1},~\texttt{MountainCar-v0}, and~\texttt{LunarLander-v3}, DQN outperforms the SCBR approach. In fact, for~\texttt{MountainCar-v0} and~\texttt{LunarLander-v3}, the SCBR fails to receive appropriate learning signals.
From these experimental results, we can observe that the SCBR method generally exhibits inferior learning performance compared to DQN. Of course, in some environments, it achieves better performance than DQN. A more comprehensive comparative analysis would be needed for a more rigorous comparison; however, this is not the primary goal of the present paper. Therefore, we do not pursue additional experiments in this paper. Further investigation and experiments remain important directions for future work.

\end{document}